\documentclass{article}

\usepackage[preprint]{neurips_2026}

\usepackage[utf8]{inputenc} 
\usepackage[T1]{fontenc}    
\usepackage{amsmath,amssymb,amsthm,mathtools}
\usepackage{bm}
\usepackage[table]{xcolor}
\definecolor{linkcolor}{HTML}{f07167}
\usepackage[pagebackref,breaklinks,colorlinks,allcolors=linkcolor]{hyperref}

\usepackage{url}            
\usepackage{array}          
\usepackage{multirow}       
\usepackage{booktabs}       
\usepackage{arydshln}       
\usepackage{threeparttable} 
\usepackage{amsfonts}       
\usepackage{nicefrac}       
\usepackage{microtype}      
\usepackage{xspace}         
\usepackage{etoc}
\usepackage{graphicx}       
\usepackage{wrapfig2}       
\usepackage[skins]{tcolorbox} 
\usepackage{subcaption}     
\usepackage{afterpage}      
\graphicspath{{content/figures/}}
\usepackage[capitalise]{cleveref}  
\creflabelformat{equation}{#2#1#3}  
\usepackage{enumitem}
\usepackage{thmtools}
\usepackage{thm-restate}
\usepackage{algorithm}
\usepackage{algpseudocode}
\usepackage{macros}

\usepackage{marginnote}

\newcommand{\dataset}[1]{#1}
\newcommand{\model}[1]{\texttt{#1}}
\newcommand{\method}[1]{\textsc{#1}}
\newcommand{\methodcite}[2]{\hyperlink{cite.#2}{\method{#1}}\nocite{#2}}
\definecolor{bestcolor}{HTML}{D9534F}
\definecolor{secondcolor}{HTML}{0500c7}

\newcommand{\best}[1]{\textbf{#1}}
\newcommand{\secondbest}[1]{\underline{#1}}
\definecolor{degradedcolor}{HTML}{A02020}
\newcommand{\degraded}[1]{\textcolor{degradedcolor}{#1}}

\title{Learning What to Forget: Improving LLM Unlearning via Learned Token-Level Importance}

\author{%
    Gizem Yüce\thanks{Equal contribution} \qquad Giorgos Nikolaou\footnotemark[1] \qquad Nicolas Flammarion \\
    Theory of Machine Learning Lab, EPFL \\
    \texttt{ \{gizem.yuce, georgios.nikolaou, nicolas.flammarion\}@epfl.ch }
}

\begin{document}

\etocdepthtag.toc{main}

\maketitle

\begin{abstract}
    Machine unlearning aims to remove targeted knowledge from a trained model while preserving its general capabilities. 
    For autoregressive language models, not all tokens in a forget sample are equally relevant to forgetting.
    Existing approaches either ignore this heterogeneity or rely on auxiliary models, heuristics, or external annotations to estimate each token's relevance for forgetting.
    We instead characterize it through the interaction with the retain objective: a token is forget-specific to the extent that minimizing the forget loss on that token does not conflict with retain optimality. 
    We formalize this perspective as a joint optimization problem over the model parameters and the token weights and show that, under a natural separation condition, the resulting objective recovers the oracle forget-specific token support.
    Motivated by this formulation, we introduce \textbf{Alternating Token-Weighted Unlearning} (\method{ATWU}), a lightweight framework that jointly learns token forget-specificity and model parameters during unlearning using a simple linear scorer over hidden states, without external token-level supervision.
    Across \dataset{TOFU} and \dataset{RWKU}, \method{ATWU} achieves state-of-the-art forget--retain trade-offs, outperforming sample-level methods, probability-based token-weighting heuristics, and auxiliary-model-based approaches. Moreover, the learned scores align substantially better with ground-truth forget-specific spans, indicating that \method{ATWU} identifies semantically meaningful token-level forgetting signals. Overall, our results suggest that retain conflict provides an effective criterion for identifying what language models should forget, enabling unsupervised learning of token-level forget-specificity directly from model representations with minimal computational overhead.
\end{abstract}

\section{Introduction}\label{sec:intro}

Large language models trained on web-scale corpora can memorize personal information~\citep{nasr2025scalable}, copyrighted text~\citep{karamolegkou-etal-2023-copyright, carlini2023quantifying}, and harmful content~\citep{li2024the, Barrett_2023}. However, the sheer scale of these models and their training data makes it infeasible to identify the specific training samples containing such information and retrain the model from scratch without them. 
Machine unlearning~\citep{10.1109/SP.2015.35, 9519428, 10.1145/3749987} has therefore become a critical capability for deploying large language models responsibly, enabling the removal of sensitive, copyrighted, or harmful knowledge without retraining from scratch. Despite its importance, scalable and reliable unlearning remains far from solved, particularly in the context of modern language models. 

Unlearning in large language models requires identifying which parts of a forget sample actually encode the targeted information to be removed. As in standard training, the unlearning objective decomposes over tokens in an autoregressive manner. However, when given a forget sample (e.g., a sentence or document to be unlearned), not all tokens contribute equally to the information that should be removed. Many tokens—such as common function words, punctuation, and syntactic patterns—are \emph{structural} and would be generated regardless of exposure to the forget data. Penalizing the generation of such tokens does not promote forgetting of the targeted content, but instead risks degrading general language capabilities.

This token heterogeneity has motivated token-weighted unlearning objectives~\citep{yang2025exploring, wan-etal-2025-every, wang2025selective, zhou2026not}, in which the forget loss is selectively applied to tokens judged to encode the targeted information. Existing approaches,
however, still obtain these token weights from external annotations, auxiliary models, probability
heuristics, or task-specific rules. In contrast, we characterize token-level forget-specificity through
the interaction between the forget and the retain objectives themselves. While prior work typically treats
retain conflict merely as a motivation for token selective unlearning, we argue that it can also provide a criterion for selection. 
A forget token is deemed \emph{forget-specific} if applying the forget objective to that
token is compatible with remaining close to retain optimality, whereas applying the same objective
to structural tokens induces retain degradation. This turns token weighting from an external heuristic
into a latent variable of the unlearning objective.

\paragraph{Contributions.} 
Our contributions are threefold: 
\begin{itemize}[topsep=-1pt,leftmargin=2em]
\item We give a retain-conflict characterization of token-level forget-specificity and show that a joint token-weighted objective recovers the oracle forget-specific support under a natural separation condition. 
\item We instantiate this principle with \method{ATWU} (\textbf{Alternating Token-Weighted Unlearning}): a lightweight scorer parameterizes token weights as a linear direction in the language model's hidden-state space and is trained by alternating scorer and model updates on the joint objective. 
\item 
Across \dataset{TOFU} and \dataset{RWKU}, we show that \method{ATWU} improves the forget--retain trade-off over sample-level baselines, probability-based token heuristics, and auxiliary-model-based token weighting, while producing token scores that align with ground-truth forget-specific spans.
\end{itemize}

\begin{figure}[t]
    \centering
    \captionsetup[subfigure]{justification=centering}
    \begin{subfigure}[c]{0.48\textwidth}
        \centering
        \footnotesize
        \setlength{\fboxsep}{1pt}
        \chatbubble{human-icon.pdf}{bubble}{What is the full name of the geology author born in Karachi, Pakistan on 06/30/1975?}

        \vspace{0.15em}
        \chatbubble{ai-icon-seul.pdf}{bubble}{
            \tokhl{seulhl}{0}{The} \tokhl{seulhl}{75}{author's} \tokhl{seulhl}{0}{name} \tokhl{seulhl}{0}{is} \tokhl{seulhl}{100}{\gttok{Hina}} \tokhl{seulhl}{100}{\gttok{Ameen}}\tokhl{seulhl}{100}{.}%
        }

        \vspace{0.05em}
        \chatbubble{ai-icon-su.pdf}{bubble}{
            \tokhl{suhl}{0}{The} \tokhl{suhl}{0}{author's} \tokhl{suhl}{100}{name} \tokhl{suhl}{100}{is} \tokhl{suhl}{100}{\gttok{Hina}} \tokhl{suhl}{100}{\gttok{Ameen}}\tokhl{suhl}{100}{.}%
        }

        \vspace{0.05em}
        \chatbubble{ai-icon-satimp.pdf}{bubble}{
            \tokhl{satimphl}{3}{The} \tokhl{satimphl}{71}{author's} \tokhl{satimphl}{87}{name} \tokhl{satimphl}{0}{is} \tokhl{satimphl}{98}{\gttok{Hina}} \tokhl{satimphl}{64}{\gttok{Ameen}}\tokhl{satimphl}{99}{.}%
        }

        \vspace{0.05em}
        \chatbubble{ai-icon-ours.pdf}{bubble}{
            \tokhl{ourshl}{0}{The} \tokhl{ourshl}{0}{author's} \tokhl{ourshl}{0}{name} \tokhl{ourshl}{0}{is} \tokhl{ourshl}{22}{\gttok{Hina}} \tokhl{ourshl}{47}{\gttok{Ameen}}\tokhl{ourshl}{0}{.}%
        }
        \caption{}
        \label{fig:qualitative:teaser-distilled}
    \end{subfigure}\hfill
    \begin{subfigure}[c]{0.51\textwidth}
        \centering
        \includegraphics[width=\linewidth]{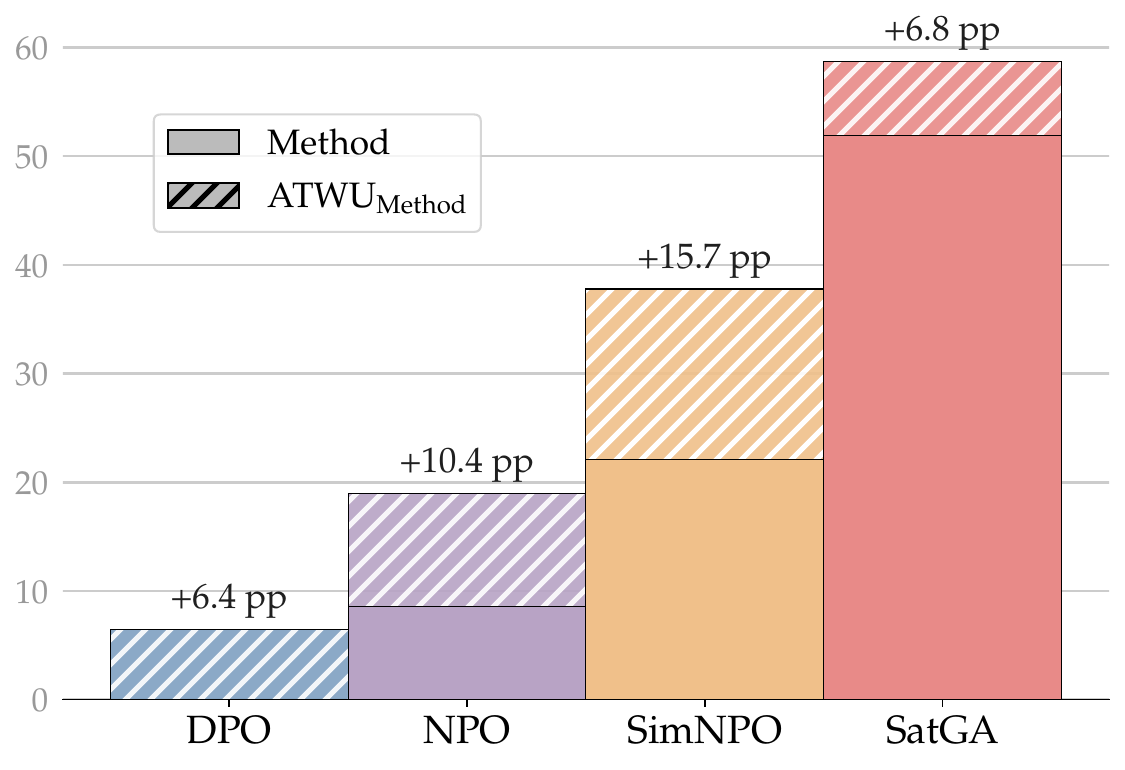}
        \caption{}
        \label{fig:qualitative:rwku-uplift}
    \end{subfigure}

    \caption{
    (\subref{fig:qualitative:teaser-distilled})~Token-level
    scores on a \dataset{TOFU} forget sample.
    \textcolor{seulhl!85!black}{\method{SEUL}},
    \textcolor{suhl}{\method{SU-LLM}}, and
    \textcolor{satimphl}{\method{SatImp}} assign substantial weight to both
    forget-specific and structural tokens, whereas 
    \textcolor{ourshl}{\method{ATWU}} concentrates on the bold
    ground-truth forget-specific span.
    (\subref{fig:qualitative:rwku-uplift})~\method{ATWU} can be combined with diverse forget losses. For each objective,
    \textcolor{dpohl}{\method{DPO}},
    \textcolor{npohl}{\method{NPO}},
    \textcolor{simnpohl}{\method{SimNPO}}, and
    \textcolor{graddiffhl}{\method{SatGA}},
    solid bars denote baseline Unlearning Quality (\UQ), and  added segments denote the gains from using \method{ATWU} token weights with the corresponding forget loss. \method{ATWU} consistently improves the forget--retain trade-off. 
    }
    \label{fig:qualitative:teaser-rwku}
\end{figure}

\section{Related Work}\label{sec:related-work}

We situate our approach relative to four key areas of prior work: we first review general sample-level unlearning objectives and their more recent token-weighted variants; we then explore the representation probing literature that motivates our hidden-state scoring mechanism; and finally, we summarize the robust evaluation protocols used to assess unlearning efficacy.

\paragraph{Sample-level/token-agnostic unlearning methods.}
Sample-level LLM unlearning methods specify the forget target at the level of full examples, documents, or question--answer pairs. Gradient Ascent (\method{GA}) maximizes the cross-entropy loss on forget samples~\citep{jang2023knowledge}, while Gradient Difference (\method{GradDiff}) combines this ascent with retain-set descent or regularization~\citep{lu2022quark}. Preference-based approaches instead cast unlearning as alignment, including \method{DPO}-style refusal objectives~\citep{rafailov2023direct}, Negative Preference Optimization (\method{NPO})~\citep{zhang2024negative}, and reference-free variants such as \method{SimNPO}~\citep{fan2026simplicity}. JensUn~\citep{singh2025unlearninglasts} instead replaces the forget and retain training losses with Jensen-Shannon-divergence objectives, yielding a strong sample-level baseline with stable unlearning dynamics. Other methods act on internal states or softened token distributions: \method{RMU} perturbs representations associated with the forget set~\citep{li2024the}, and \method{UNDIAL} distills toward adjusted distributions that reduce memorized-token probability~\citep{dong2025undial}.  

\paragraph{Token-weighted unlearning methods.}
Token-weighted unlearning approaches aim to identify tokens that encode targeted information and selectively apply the forgetting objective to avoid unnecessary utility degradation. These methods typically adopt one of the sample-level objectives above for the actual forget update; they differ primarily in how token forget-specificity is estimated. \emph{Auxiliary-model approaches} infer it by comparing the predictions of models fine-tuned on retain and forget splits~\citep{eldan2023whos, wan-etal-2025-every}. \emph{Probability-based heuristics} instead use model confidence, entropy or other loss-based criteria as proxies for token forget-specificity~\citep{wang2025rethinking, yang2025exploring, koh2026forgetmattersrestselective}. Other methods rely on external signals such as linguistic parsers~\citep{dong2025undial} or LLM-generated annotations~\citep{zhou2026not} to identify forget-specific tokens. In particular, approaches based on large language model-generated annotations can provide strong supervision, but may be costly to scale and raise additional privacy concerns.
In contrast, \method{ATWU} learns token-level weights jointly with the unlearning update, using the same retain and forget objectives that define the desired deletion behavior, and requires neither external supervision nor auxiliary model training. 

\paragraph{Probes for linguistic properties.} 
A complementary line of work studies what token-level information is encoded in language model hidden representations using probing classifiers, often linear maps or MLPs trained on frozen activations \citep{alain2017understanding}. These studies show that contextual token representations encode a broad range of linguistic structure including part-of-speech, morphology \citep{belinkov-etal-2017-neural}, syntactic dependencies, and lexical-semantic distinctions \citep{clark-etal-2019-bert} and that many such properties are linearly recoverable from intermediate layers of transformer language models \citep{tenney-etal-2019-bert}. Structural probing work further suggests that hidden-state geometry can reflect syntactic relations directly, rather than merely supporting downstream decoding \citep{hewitt-manning-2019-structural}. While this literature is primarily descriptive and does not address unlearning, it motivates our choice to parameterize token forget-specificity as a lightweight linear scorer over hidden states: if token-level linguistic and semantic attributes are already organized in these representations, then forget-specific information may likewise be identifiable from the same representation space. Furthermore, our formulation allows the scorer to be learned jointly during unlearning without requiring any external supervision. 

\paragraph{Unlearning evaluation.}
Unlearning in LLMs is typically evaluated by balancing forget efficacy against retained utility. TOFU~\citep{maini2024tofu, dorna2026openunlearning} focuses on fictitious-author question answering, measuring forget quality and model utility. MUSE~\citep{shi2025muse} expands evaluation to broader corpora and includes metrics such as memorization, privacy leakage, utility preservation, and scalability. RWKU~\citep{jin2024rwku} considers real-world entity forgetting alongside general capabilities including reasoning, factuality, and fluency. Recent work highlights limitations of surface-form evaluation: \citet{singh2025unlearninglasts} show that models may appear to forget while still recovering the same information under paraphrased or perturbed inputs, and advocate evaluation based on paraphrases, semantic judging, and worst-case performance across variants. We follow this perspective in our evaluations for a more robust assessment. 

\section{Token-Weighted Unlearning as a Joint Optimization Problem}
\label{sec:setting}
We now formalize the token-selection problem underlying \method{ATWU}. The goal is to define when a forget token should receive the forget update, and to show that this selector can be identified from the retain and forget objectives themselves.

\subsection{Background and Notation}
Let $\vocab$ be a finite vocabulary, $\paramspace \subseteq \R^d$ the parameter space, and $\llm{\tth} : \vocab^* \to \Delta(\vocab)$ an autoregressive language model. We write $\llmf{\tth}{x}{t}$ for the probability of token $\tok{x}{t}$ in context $\tokc{x}{t}$ assigned by the model and $\size*{\seq{x}}$ for the length of sequence $\seq{x}$. We are given a pretrained model $\llm{\thpre}$, a forget set $\DF$ whose influence should be removed, and a retain set $\DR$ on which model performance should be preserved.


The common framework for LLM unlearning combines a retain loss $\rtokloss{x}{t}{\tth}$, which preserves behavior on retain samples, with a forget loss $\ftokloss{x}{t}{\tth}$, which suppresses the original forget completion. The retain loss is typically cross-entropy \(\ell_{\text{CE}}\) or KL divergence to the pretrained model predictions. Forget losses include negative cross-entropy $\ell_{\text{GA}} =\minus \ell_{\text{CE}} = \GAloss{\tth}{x}{t}$ with various scalings and preference-based losses. 
These objectives are in tension because while the retain loss tries to increase the likelihood of a sample, the forget loss often tries to decrease it.  
For clarity, we absorb fixed loss coefficients and dataset normalizations into $\ell_\retain$ and $\ell_\forget$. 
Finally, we write $R(\tth) \defeq \sum_{\seq{x} \in \DR}\sum_{t=1}^{\size{\seq{x}}} \rtokloss{x}{t}{\tth}$ for the retain objective and $R^\star \defeq \min_{\tth \in \paramspace} R(\tth)$ for its minimum over the parameter space. 

\subsection{Forget-Specific vs.\ Structural Tokens}
The forget set contains both information that should be removed and linguistic structure that should be preserved. Tokens such as names, dates, quantities, rare identifiers, and attribute-bearing phrases can identify the targeted knowledge in \(\DF\); we call these tokens \textbf{forget-specific}. In contrast, function words, punctuation, syntactic patterns, and generic continuations are \textbf{structural}: they are useful under the retain distribution and would often be generated even without exposure to the forget data. Applying the forget loss to forget-specific tokens removes targeted information, while applying it to structural tokens can instead damage general language capabilities.



This distinction has motivated the introduction of token-level weights $\lvar{x}{t} \in [0,1]$ approximating the oracle binary label $\olvar{x}{t} = \mathbf{1}_{\{\tok{x}{t} \text{ is a forget-specific token}\}}$, and the token-weighted loss
\begin{equation}
    \Loss(\tth,\zz)
    \;=\;
    \sum_{\seq{x} \in \DR} \sum_{t=1}^{\size{\seq{x}}} \rtokloss{x}{t}{\tth}
    \;+\;
    \sum_{\seq{x} \in \DF} \sum_{t=1}^{\size{\seq{x}}}
    \lvar{x}{t} \, \ftokloss{x}{t}{\tth}.
    \label{eq:obj-weighted}
\end{equation}

If the oracle labels $\zstar$ were known, re-weighted unlearning could simply compute a minimizer \(\tth^\star \in \arg\min_{\tth} \Loss(\tth, \zstar)\). This oracle setting is useful as an analysis tool---Table~\ref{tab:exp:utilisation} later shows that using ground-truth token labels can indeed substantially improve the forget-retain trade-off---but such labels are unavailable or very costly in practice. Prior token-weighted methods therefore approximate $\zz^\star$ using annotations, auxiliary models, or probability-based proxies such as confidence, saturation, and surprisal. These methods, which adopt the formulation in Equation~\ref{eq:obj-weighted}, are likewise motivated by the retain conflict: minimizing forget-loss on structural tokens harms retain performance, while doing so on the forget-specific tokens does not. The key difference in our method from prior works is that we use retain conflict as a token-level identification criterion.


\subsection{Characterizing Forget-Specificity via Retain Conflict}

Let $\IF \defeq \{(x,t) : x \in \DF,\, t \in [|x|]\}$ denote the set of forget sequence position pairs and $N_F \defeq |\IF|$ be the number of tokens in the forget set. For each $i = (x, t) \in \IF$, write $\ell^i_\forget(\tth)=\ell_\forget(\paramby{\given{\tok{x}{t}}{\tokc{x}{t}}}{\tth})$. 
For any subset $\set{A} \subseteq \IF$, define the retain conflict
\begin{equation}
\conflict(\set{A}) \coloneqq \min_{\tth \in \paramspace}\left[ R(\tth) - R^\star + \sum_{i \in \set{A}}\left(\ell^i_\forget(\tth)-\ell_\forget^{\min}\right)\right],
\end{equation}
where $\ell_\forget^{\min} \coloneqq \min_{\tth,i} \ell^i_\forget(\tth)$ is the global minimum of the forget loss. Each summand $\ell^i_\forget(\tth) - \ell_\forget^{\min} \geq 0$, and therefore we have $\conflict(\set{A}) \geq 0$. Moreover, $\conflict$ is \emph{monotone}: for $\set{A} \subseteq \set{B}$, $\conflict(\set{A}) \leq \conflict(\set{B})$.


We write $\conflict_i \coloneqq \conflict(\{i\})$ for the singleton conflict. This quantity is the best possible residual cost of simultaneously preserving the retain performance and driving the forget loss for token $i \in \IF$ to its floor. A token with small conflict can be suppressed without moving far from a retain-optimal model, so it is a plausible forget-specific token. A token with large conflict cannot be suppressed without sacrificing retain performance, which is precisely the behavior expected of structural tokens. We give a formal characterization of this intuition with the next assumption. 


\begin{assumption}[Retain-conflict separation]
\label{def:zstar}
There exist $\varepsilon \ge 0, \delta > 0$, and a partition $\IF = \Fset^\star \cup \set{S}^\star$ with \(\size{\Fset^\star} = \rho^\star N_F\), where the oracle labels are $\solvar{i} = \mathbf{1}_{\{i \in \Fset^\star\}}$, such that
\begin{equation}
    \conflict(\Fset^\star) \le \varepsilon \qquad\text{and}\qquad \conflict_i \ge \varepsilon + \delta \quad \forall i \in \set{S}^\star.
    \label{eq:partition-margins}
\end{equation}
\end{assumption}
This assumption characterizes the forget-specific tokens as the tokens that can be forgotten while remaining in the retain optimum, and structural tokens as those whose unlearning forces the model outside the retain optimum. While this conceptual separation is elegant, using it directly as an algorithmic selection criterion is intractable. Computing the singleton conflict $\conflict_i$ demands a full optimization over the model parameters $\tth$ for each of the $N_F$ forget tokens. To bypass this computational bottleneck, we show that under the retain-conflict separation, the oracle labels $\zstar$ as well as \(\tth^\star\) are instead recovered as the minimizer of the joint problem
\begin{equation}
    \min_{\tth \in \paramspace,\;\zz \in \mathcal{Z}} \Loss(\tth, \zz),
    \qquad
    \mathcal{Z} \defeq \bigl\{\zz \in \{0,1\}^{N_F} : {\textstyle\sum_{i \in \IF}} \slvar{i} = \rho N_F \bigr\}.
    \label{eq:joint}
\end{equation}
The budget constraint $\mathcal{Z}$ prevents trivial minimizers of the form $\zz = \mathbf{0}$ or $\zz = \mathbf{1}$, which do not identify a forget-specific subset.

\begin{restatable}[Recovery]{theorem}{thmrecovery}
\label{thm:recovery}
Under \Cref{def:zstar}, suppose $0 < \rho \leq \rho^\star$ and $\rho N_F$ is an integer. If $(\hat{\tth},\hat{\zz})$ is any global minimizer of \eqref{eq:joint}, then $\mathrm{supp}(\hat{\zz}) \subseteq \Fset^\star$. If $\rho = \rho^\star$, then $\mathrm{supp}(\hat{\zz}) = \Fset^\star$ and \(\hat{\tth} = \tth^\star\).
\end{restatable}
The proof is deferred to Appendix~\ref{app:proofs}.
\Cref{thm:recovery} shows that the joint fixed-budget objective recovers the oracle token subset under retain-conflict separation. In particular, for $\rho \le \rho^\star$, every global minimizer selects only tokens in $\Fset^\star$, and for $\rho=\rho^\star$, it recovers both the oracle support $\zstar$ and the corresponding oracle minimizer $\tth^\star$. Notably, this recovery argument only requires the token-wise forget loss to be bounded below---a condition satisfied by the saturated negative cross-entropy loss used in our \method{ATWU} instantiation (\Cref{sec:alternating}). The theorem therefore justifies learning token weights jointly with model parameters directly from the unlearning objective, without requiring externally specified token labels. In practice, $\rho$ serves as a scalar hyperparameter controlling the budget of this learned selector.

\section{Alternating Token-Weighted Unlearning}
\label{sec:alternating}

We now turn the joint formulation into a practical learning algorithm: \emph{Alternating Token-Weighted Unlearning} (ATWU). We first relax the discrete optimization problem, then introduce the hidden-state scorer, specify and justify the selection of forget loss used in our experiments, and finally describe the alternating optimization procedure. 

\paragraph{Lagrangian relaxation.} The constrained problem \eqref{eq:joint} is combinatorial in $\zz$. We therefore instead optimize a continuous penalized objective
\begin{equation}
    \tilde{\Loss}(\tth, \zz) = \Loss(\tth, \zz) + \lambda_H \sum_{\lvar{x}{t} \in \zz} H(\lvar{x}{t}) +  \lambda_{\mathrm{\rho}} \Bigl(\tfrac{1}{N_F} \sum_{\lvar{x}{t} \in \zz} \lvar{x}{t} - \rho\Bigr)^{\!2},
    \label{eq:atwu}
\end{equation}
 where $H(z)=-z\log z-(1-z)\log(1-z)$ is the binary entropy.

The objective $\tilde{\Loss}(\tth,\zz)$ relaxes the original combinatorial problem \eqref{eq:joint} on two fronts: the binary constraint $\zz \in \{0,1\}^{N_F}$ is relaxed to $\zz \in [0,1]^{N_F}$, and the budget constraint $\sum \lvar{x}{t} = \rho N_F$ is replaced by the quadratic penalty. The entropy term counters the first relaxation by pushing each $\lvar{x}{t}$ back towards $\{0,1\}$. This relaxation can be \emph{exact} under sufficiently large regularization.  

\begin{lemma}[Exactness of the relaxation, informal]\label{lem:exact-relaxation}
For sufficiently large \(\lambda_H\) and \(\lambda_\rho\), every global minimizer of the relaxed objective~\eqref{eq:atwu} over \([0,1]^{N_F}\) is binary, satisfies the budget constraint, and is a global minimizer of the constrained problem~\eqref{eq:joint}. 
\end{lemma}

The formal statement and proof are given in \Cref{lem:exact-relaxation-formal}. While this free-token-weight formulation provides strong theoretical guarantees, learning an independent scalar for every token is unscalable for large corpora and cannot generalize to unseen sequences. To bridge this gap from theory to practice, \method{ATWU} parameterizes the selector using a shared scoring function $g_\sw$. Under this scalable parameterization, the entropy and budget terms naturally transition from enforcing exact combinatorial constraints to acting as principled regularizers, successfully guiding the network toward sparse, budget-controlled token selection.

To instantiate this shared scorer, we parameterize token forget specificity using a simple linear projection over the language model’s hidden representations. This architectural choice is directly motivated by the structural probing literature: because intermediate hidden states natively encode rich linguistic and semantic properties \citep{tenney-etal-2019-bert}, they provide an ideal representation space for identifying forget-specific patterns without requiring a complex auxiliary network.

Let $\sh_{\tth}(\tok{x}{t}) \in \R^d$ denote the $t$-th hidden representation produced by $\llm{\tth}$ for the sequence $\seq{x}$. \method{ATWU} defines
\begin{equation}
\alvar{x}{t}
=
g_\sw(\sh_{\tth}(\tok{x}{t}))
=
\sigma\!\left(\langle \sw,\sh_{\tth}(\tok{x}{t})\rangle\right)
\in (0,1),
\end{equation}
where $\sw \in \R^d$ is the learned scorer parameter.

This single linear projection keeps the mechanism lightweight while encouraging reusable token-level patterns rather than independent per-token decisions. We initialize $\sw=\mathbf{0}_d$, meaning all tokens initially receive a uniform score of $1/2$. Consequently, the objective begins token-agnostic and smoothly becomes selective during training. As a proof of concept, Appendix~\ref{app:subsec:additional:linear-separability} shows that this exact linear architecture can recover ground-truth forget labels when trained with explicit supervision, i.e., the scorer model has enough capacity to model the true forget specificity; \method{ATWU}, however, learns the scores entirely implicitly from the unlearning objective.

\paragraph{Choice of loss functions.} 
We use standard cross-entropy as the retain loss and a saturated negative cross-entropy introduced by \citet{wang2025rethinking} as the base forget loss. For a hyperparameter $\beta > 0$, the token-wise forget loss is defined as:
\begin{equation} \label{eq:loss-sat}
    \ell_\forget(\paramby{\given{\tok{x}{t}}{\tokc{x}{t}}}{\tth})
    \;\defeq\;
    \llmf{\tth}{x}{t}^{\beta}
    \cdot
    \GAloss{\tth}{x}{t}.
\end{equation}
This loss formulation incorporates the saturation weighting mechanism utilized by \citep{wang2025rethinking} inside the loss, which we will refer to as \method{SatGA}. While saturation alone is a poor proxy for token-level forget-specificity (as evidenced by the poor AUROC of the saturation score in \cref{fig:exp:auroc}), it provides highly desirable gradient scaling properties for the loss function. 

Crucially, unlike raw negative cross-entropy $\ell_{\text{GA}}$, this saturated variant is bounded below. As a function of $p\in(0,1]$, the term $p^\beta\log p$ reaches a minimum of $-1/(e\beta)$ at $p=e^{-1/\beta}$. Consequently, the hyperparameter $\beta$ effectively controls the target probability level for the forget update, with smaller $\beta$ values enforcing stronger suppression.

To instantiate our general joint-learning framework using this bounded objective, we introduce a slight modification: we also inject the dynamically learned token score into the saturation exponent. We refer to this score-modulated variant as \method{SatGA}$^+$. Plugging this into the joint formulation yields our primary \method{ATWU} objective:
\begin{equation}\label{eq:atwu-parametric}
\begin{aligned}
    \Loss_{\method{ATWU}}(\tth, \sw) &= \alpha R(\tth) + \gamma \sum_{\seq{x},t} g_\sw(\tok{x}{t}) \cdot \llmf{\tth}{x}{t}^{\beta g_\sw(\tok{x}{t})} \cdot \GAloss{\tth}{x}{t} \\
    &\quad + \lambda_H \sum_{\seq{x},t} H\!\bigl(g_\sw(\tok{x}{t})\bigr) + \lambda_{\mathrm{\rho}} \Bigl(\tfrac{1}{N_F}{\textstyle\sum_{\seq{x},t}} g_\sw(\tok{x}{t}) - \rho\Bigr)^{\!2},
\end{aligned}
\end{equation}
where $g_\sw(\tok{x}{t}) \coloneqq g_\sw(\sh_{\tth}(\tok{x}{t}))$. 

For strictly binary scores, the \method{SatGA}$^+$ modification reduces to standard \method{SatGA}. However, because our scores $g_\sw$ are continuous during training, the modified exponent forces uncertain tokens to exert a much smoother, attenuated forget update. This stabilizes the early phases of optimization before the scorer has fully converged, an effect we ablate directly in \cref{tab:exp:utilisation}. The coefficients $\alpha$ and $\gamma$ denote the retain and forget loss weights, respectively. While we absorbed these constants into the base loss definitions in earlier sections for notational simplicity, we make them explicit here to reflect our exact empirical objective.

\paragraph{Alternating optimization.}
We optimize $\Loss_{\method{ATWU}}(\tth, \sw)$ by alternating between language-model and scorer updates.
With the scorer $\sw$ fixed, the language-model parameters \(\tth\) are updated using the current token-weighted forget objective together with the retain objective. With $\tth$ fixed, the scorer $\sw$ is updated to improve token selection under the same regularized objective.
During model updates, the scores are detached and treated as fixed
coefficients; during scorer updates, \(\tth\) is frozen and gradients flow only
through \(\mathbf{w}\).
This scheduled alternation is empirically more stable than updating the scorer and model in lockstep: the scorer changes the effective forget objective, while the model changes the hidden-state geometry on which the scorer depends. Updating them on separate timescales reduces this feedback loop. Finally, in Appendix~\ref{app:subsec:additional:ablations} we present an ablation on the update frequency that shows that the alternating variant outperforms joint updates.

\section{Experiments}\label{sec:experiments}
%
\begin{table}[t]
    \centering
    \small
    \renewcommand{\arraystretch}{1.2}
    \resizebox{\textwidth}{!}{%
    \begin{tabular}{@{}c c ccc ccc c ccc ccc@{}}
        \toprule
        \multirow{2}{*}[-0.4ex]{\textbf{Method}}
        & & \multicolumn{3}{c}{\textbf{Unlearning}}
        & \multicolumn{3}{c}{\textbf{Utility}}
        & & \multicolumn{3}{c}{\textbf{Unlearning}}
        & \multicolumn{3}{c}{\textbf{Utility}} \\
        \cmidrule(lr){3-5} \cmidrule(lr){6-8} \cmidrule(lr){10-12} \cmidrule(lr){13-15}
        & & $\FQ\,\uparrow$ & $\RD\,\downarrow$ & $\UQ\,\uparrow$ & MMLU & Rep. & \WR
        & & $\FQ\,\uparrow$ & $\RD\,\downarrow$ & $\UQ\,\uparrow$ & MMLU & Rep. & \WR \\
        \midrule
        \method{Original} &\multirow{10}{*}{\rotatebox[origin=c]{90}{\textbf{TOFU}}}
                & $0.0$              & $0.0$              & $0.0$              & $66.6$           & $546$             & $50.0$
                & \multirow{10}{*}{\rotatebox[origin=c]{90}{\textbf{RWKU}}}
                & $0.0$              & $0.0$              & $0.0$              & $70.1$           & $529$             & $50.0$                \\
        \midrule
        \methodcite{GradDiff}{lu2022quark} && $39.7$             & $17.3$             & $22.4$             & $66.1$           & $545$             & $48.0$
                & & $\secondbest{81.9}$ & $36.7$            & $45.2$             & $69.2$           & $514$             & \degraded{$43.0$} \\
        \methodcite{DPO}{rafailov2023direct} && $59.5$             & $15.0$             & $44.5$             & \degraded{$64.3$} & $533$            & \degraded{$38.5$}
                & & $12.3$             & $\best{17.0}$      & $0.0$              & $69.7$           & $523$             & \degraded{$43.0$} \\
        \methodcite{NPO}{zhang2024negative} && $48.0$             & $5.0$              & $43.0$             & $66.2$           & $551$             & $48.5$
                & & $47.6$             & $39.1$             & $8.5$              & $68.1$           & $543$             & \degraded{$33.5$} \\
        \methodcite{SimNPO}{fan2026simplicity} && $69.3$             & $6.5$              & $62.8$             & $65.8$           & $541$             & $50.5$
                & & $51.8$             & $29.7$             & $22.1$             & $68.8$           & $529$             & $45.5$             \\
        \methodcite{JensUn}{singh2025unlearninglasts} && $\best{98.0}$      & $9.6$              & $\secondbest{88.3}$ & $65.3$          & \degraded{$236$} & \degraded{$11.0$}
                & & $\best{85.4}$      & $36.1$             & $49.3$             & $70.2$           & $523$             & $46.5$             \\
        \methodcite{RMU}{li2024the} && $87.2$             & $\best{1.3}$       & $85.9$             & $65.3$           & $540$             & $50.5$
                & & ---                & ---                & ---                & ---              & ---               & ---               \\
        \methodcite{WGA}{wang2025rethinking} && $66.8$             & $4.8$              & $62.1$             & $66.0$           & $550$             & $60.0$
                & & $78.0$             & $26.1$             & $\secondbest{51.9}$ & $69.4$          & $518$             & \degraded{$41.0$} \\
        \methodcite{SatImp}{yang2025exploring} && $78.6$             & $\secondbest{3.0}$ & $75.6$             & $65.9$           & $549$             & $53.5$
                & & $77.3$             & $29.3$             & $48.0$             & $69.8$           & $518$             & \degraded{$40.0$} \\
        \midrule
        \method{ATWU} && $\secondbest{95.2}$ & $3.5$              & $\best{91.7}$      & $66.5$           & $558$             & $58.0$
                & & $81.4$             & $\secondbest{22.7}$ & $\best{58.7}$      & $70.3$           & $519$             & $46.0$             \\
        \bottomrule
    \end{tabular}%
    }
    {\footnotesize\par\vspace{0.25em}\noindent\FQ: forget quality, the relative reduction in worst-case forget-set judge score. \RD: retain degradation, the relative loss in retain-set judge score. $\UQ=[\FQ-\RD]_{+}$: net forget--retain trade-off. Definitions in Appendix~\ref{app:sec:metrics:derived}.\par}
    \vspace{0.2em}
    \caption{
        \method{ATWU} achieves the best \UQ on both benchmarks while preserving utility close to the original checkpoint, the only method to do so consistently across the six metric panels. Left: \dataset{TOFU}~\texttt{forget10} with \model{Llama-3.1-8B-Instruct}; right: canonical \dataset{RWKU} ten-subject batch with \model{Phi-3-Mini-4k-Instruct}. Higher-\FQ competitors such as \method{JensUn} forget aggressively but collapse generation quality; \method{RMU} matches \method{ATWU} on retain degradation but trails on \UQ. \best{Best} (\textbf{bold}) and \secondbest{second-best} (\underline{underlined}) per column; utility values that are materially degraded (MMLU drop $> 2$pp, Rep.\ drop $> 5\%$, or $\mathrm{WR} < 45$) are shown in \degraded{red}.
    }
    \label{tab:exp:headline}
\end{table}%

Our experiments are designed to test whether the learned token weights are accurate, useful for end-to-end unlearning, and whether the alternating token-weighted unlearning (\method{ATWU}) framework generalizes to loss functions beyond saturated negative cross-entropy. First, we compare \method{ATWU} against state-of-the-art unlearning methods across multiple models and datasets. Second, we evaluate the quality of the forget-specificity scores learned by the \method{ATWU} scorer, comparing them against alternative token-weighting approaches. Third, we demonstrate the versatility of our framework by testing whether \method{ATWU} consistently improves unlearning quality when used with other common forget losses.

\paragraph{Benchmarks, models, and training.}
We evaluate \method{ATWU} on two LLM unlearning benchmarks. \dataset{TOFU}~\citep{maini2024tofu} is a synthetic question--answering benchmark built from fictitious-author biographies. We report main results on \texttt{forget10} with \model{Llama-3.1-8B-Instruct}~\citep{grattafiori2024llama}, using the checkpoint released by \citet{dorna2026openunlearning}. Results on \texttt{forget01} and \texttt{forget05}, as well as all three splits with
\model{Llama-3.2-1B-Instruct}, are qualitatively consistent and deferred to Appendix~\ref{app:subsec:exp:unlearning} for brevity. \dataset{RWKU}~\citep{jin2024rwku} targets real-world public-figure knowledge in a pretrained LLM. Following \citet{singh2025unlearninglasts}, we perform batch unlearning on a fixed canonical batch of ten subjects using \model{Phi-3-Mini-4k-Instruct}~\citep{abdin2024phi3technicalreporthighly}. Hardware, software, and optimization details used across all settings are provided in Appendix~\ref{app:sec:exp-details}.

\paragraph{Hyperparameter tuning.}
For each method, we tune method-specific hyperparameters, including learning rates and loss coefficients, via Bayesian optimization~\citep{akiba2019optuna}, seeded with each method's recommended configuration and guided by a cheap surrogate metric (\ESD for \dataset{TOFU}, \NDelta for \dataset{RWKU}). Crucially, we find that strong baselines are often obscured by limited tuning budgets. With adequate per-method tuning, several prior methods perform markedly better than previously reported. \method{RMU} serves as a prime example: it becomes highly competitive on \dataset{TOFU} and surpasses several recently proposed methods—a finding that contradicts earlier literature. We view this outcome not as a critique of prior work, but as a broader motivation for rigorous, uniform tuning protocols in unlearning evaluations, a point we expand upon in Appendix~\ref{app:tuning-importance}.

\paragraph{Evaluation.}
We use the paraphrase- and judge-based evaluation protocol of \citet{singh2025unlearninglasts}, with \model{GPT-5.4-mini} as the judge. We report three baseline-relative summary metrics: forget quality (\FQ), which measures the reduction in worst-case forget-set judge score, retain degradation (\RD), which measures the loss in retain-set judge score, and unlearning quality $\UQ=[\FQ-\RD]_{+}$, which summarizes the net forget--retain trade-off. Unlike metrics like token-overlap with the original answers or likelihood-based surrogates, this protocol penalizes methods that suppress only the original surface form while leaving semantically equivalent completions recoverable. We complement these summary metrics with three utility-preservation probes: MMLU accuracy, generation repetitiveness (Rep.), and pairwise win rate against the original checkpoint (\WR). Detailed definitions and discussions on the surrogate, judge-based and utility metrics are provided in Appendix~\ref{app:sec:metrics}. Full metric panels are reported in Appendix~\ref{app:sec:exp-details}.

\subsection{Unlearning Quality}
\label{subsec:exp:main}

\paragraph{Main results.}
\Cref{tab:exp:headline} reports the main
$\FQ/\RD/\UQ$ and utility results on the
two largest experimental configurations:
\dataset{TOFU}~\texttt{forget10} with
\model{Llama-3.1-8B-Instruct} and the canonical \dataset{RWKU}
ten-subject batch with \model{Phi-3-Mini-4k-Instruct}.
\method{ATWU} achieves the best \UQ on both benchmarks:
$91.7$ on \dataset{TOFU}, exceeding the runner-up
\method{RMU} by $5.8$ percentage points, and $58.7$ on
\dataset{RWKU}, exceeding \method{WGA} by $6.8$ percentage
points.
Among the strongest unlearning methods, \method{ATWU} is also the only one that avoids a major degradation on the utility axes we measure. It attains the highest MMLU score on both benchmarks, the strongest repetitiveness score on \dataset{TOFU}, and remains competitive in pairwise win rate. Results are consistent across the smaller \texttt{forget01}
and \texttt{forget05} splits and across model scales, as shown
in \cref{tab:tofu-full-1b,tab:tofu-full-8b}.

\begin{wraptable}[14]{r}{0.51\textwidth}
    \vspace{-1em}
    \centering\small
    \setlength{\tabcolsep}{3pt}
    \renewcommand{\arraystretch}{1.1}
    \resizebox{\linewidth}{!}{%
    \begin{tabular}{@{}c cc ccc@{}}
        \toprule
        \multirow{2}{*}[-0.6ex]{\textbf{Method}}
        & \multicolumn{2}{c}{\textbf{Relative}}
        & \multicolumn{3}{c}{\textbf{Utility}} \\
        \cmidrule(lr){2-3} \cmidrule(lr){4-6}
          & $\UQ\,\uparrow$ & $\Delta$$\UQ\,\uparrow$
          & MMLU & Rep. & \WR \\
        \midrule
        \method{Original}            & $0.0$              & ---     & $70.1$ & $529$ & $50.0$ \\
        \midrule
        \method{ATWU}$_\mathrm{DPO}$    & $6.4$               & $+6.4$  & $69.5$ & $512$ & \degraded{$36.0$} \\
        \method{ATWU}$_\mathrm{NPO}$    & $18.9$              & $+10.4$ & \degraded{$67.7$} & $545$ & \degraded{$35.5$} \\
        \method{ATWU}$_\mathrm{SimNPO}$ & $\secondbest{37.8}$ & $+15.7$ & $69.9$ & $531$ & \degraded{$43.5$} \\
        \midrule
        \method{ATWU}                    & $\best{58.7}$       & $+6.8$ & $70.3$ & $519$ & $46.0$ \\
        \bottomrule
    \end{tabular}%
    }
    \vspace{0.3em}
    \caption{\dataset{RWKU} ten-subject batch results of
    \textsc{ATWU} with various forget losses. $\Delta$\UQ reports the gain
    over each method's vanilla counterpart in
    \cref{tab:exp:headline}.}
    \label{tab:exp:rwku-atwu-augmented}
\end{wraptable}
\paragraph{\method{ATWU} with other forget losses.}
So far, we have instantiated \method{ATWU} using a saturated negative cross-entropy loss. However, our joint-learning framework is fundamentally loss-agnostic. Any forget loss that decomposes autoregressively over tokens (\cref{eq:obj-weighted}) can be seamlessly plugged into the formulation. Moreover, this flexibility extends to forget objectives constructed from an autoregressive building block---such as the token-wise NLL of the forget sequence, or its log-ratio against a frozen reference---even if they do not perfectly match the exact formulation in \cref{eq:obj-weighted}. In such cases, we simply replace the autoregressive block with its scorer-reweighted analog, leaving the rest of the loss untouched. This adaptation applies cleanly to \method{DPO}, \method{NPO}, and \method{SimNPO}. Crucially, the training procedure remains identical: the model and the token scorer are still optimized in an alternating manner from scratch during unlearning. \Cref{fig:qualitative:rwku-uplift,tab:exp:rwku-atwu-augmented} demonstrate the unlearning quality of \method{ATWU} instantiated with \method{DPO}, \method{NPO}, and \method{SimNPO}, alongside our primary \method{SatGA}$^+$ formulation for comparison. \method{ATWU} improves \UQ over the unweighted version of every loss by $6.4$ to $15.7$ percentage points---evidence that the performance gains from \method{ATWU} are not restricted to a specific choice of loss function, even though the underlying base loss dictates the overall performance ceiling. Further details are provided in Appendix~\ref{app:subsec:additional:augment}.


\subsection{Comparison with other Token Weighting Methods}
\label{subsec:exp:tokenimp}

\paragraph{Score quality.}
We first evaluate the learned scores directly. \Cref{fig:exp:auroc} ranks tokens within each \dataset{TOFU} forget sample by each method's score and computes AUROC against the ground-truth labels from \citet{zhou2026not}. \method{Saturation} is the only method that performs below the random baseline of $50$ ($33 \pm 17$). The auxiliary-model approaches \method{LLM Diff}, \method{N-Gram Diff}, and \method{LogProb Mask} do better but stall in the $54$–$63$ range. The strongest baselines---\method{Importance}, \method{Entropy}, and \method{Noun Mask}---cluster around $67$–$68$. 
Our \method{ATWU} learned scores stand clearly above all of them at $75 \pm 9$, combining the highest mean with the lowest variance among the compared methods. This indicates that \method{ATWU}'s gains come not only from how the token weights are used, but also from a substantially stronger token-level forget-specificity signal than that of competing token-weighting mechanisms.

\begin{wrapfigure}[25]{r}{0.51\textwidth}
    \centering
    \vspace{-1.6em}
    \includegraphics[width=\linewidth]%
        {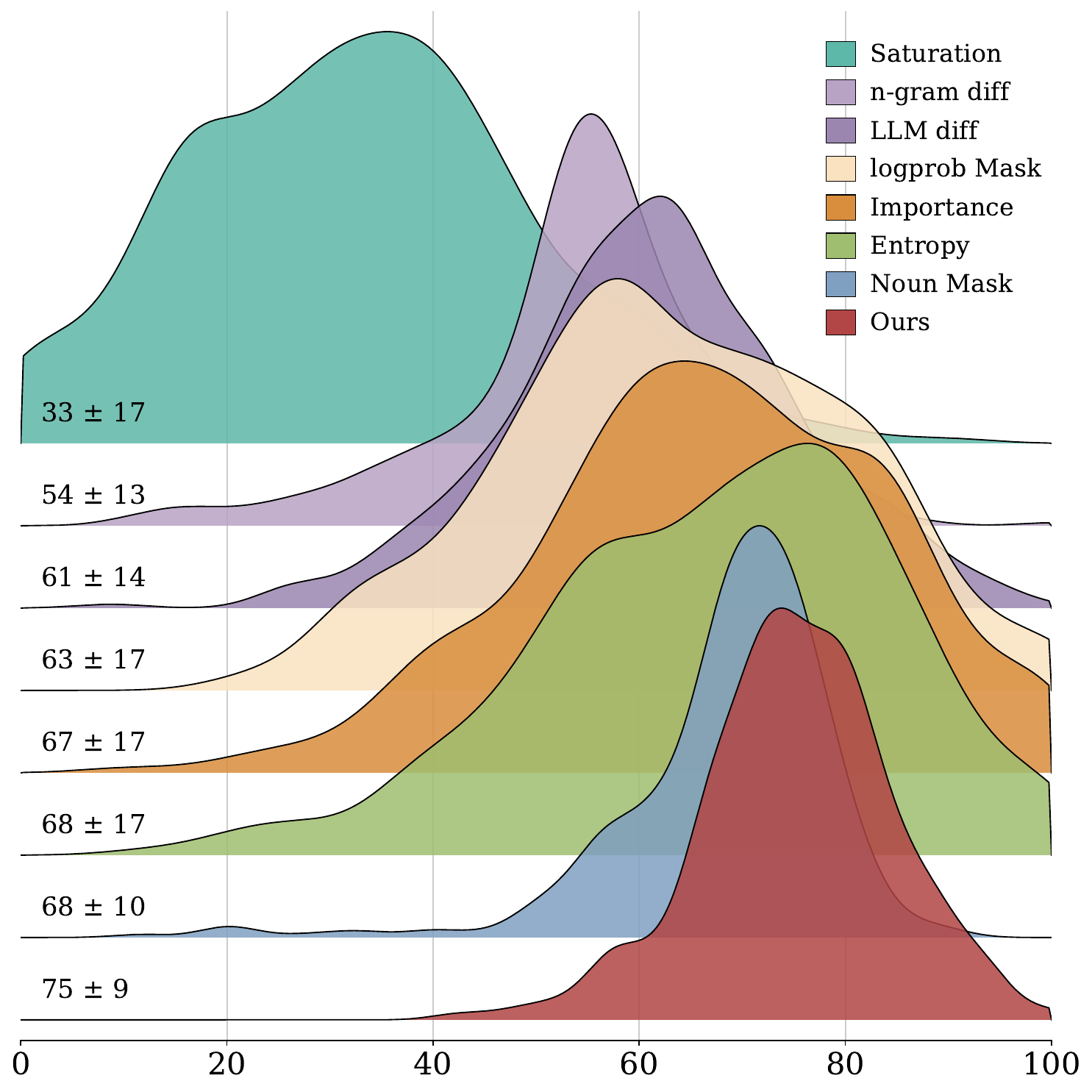}
    \caption{Per-sample AUROC distributions for token-level forget-specificity on \dataset{TOFU} forget10, scored against the ground-truth token labels of \citet{zhou2026not}. Labels report mean $\pm$ std across forget samples. Method-to-scoring-criterion mapping in \Cref{tab:exp:utilisation}; pooled per-token ROC curves in \Cref{fig:method-roc-tofu-forget10}.} 
    \label{fig:exp:auroc}
\end{wrapfigure}
\paragraph{Unlearning quality.}
The superior quality of the \method{ATWU} scores directly translates to unlearning performance: as shown in \Cref{tab:exp:tokenimp-comparison}, \method{ATWU} achieves the strongest forget--retain trade-off among all token-weighted approaches. It secures the best \FQ and \UQ, remains competitive on \RD, and preserves general utility close to the original checkpoint. \Cref{tab:exp:tokenimp-comparison} also details how the seven representative token-weighted baselines following the formulation in \Cref{eq:obj-weighted} approximate token-level forget-specificity. The choice of forget loss varies between these methods: \method{FUNDIAL} uses an engineered distillation loss, \method{SatImp} and \method{ATWU} use \method{SatGA} and \method{SatGA}$^+$ forget losses, respectively, and all other methods use the token-wise negative \method{CE} loss. Notably, \method{WGA} can be viewed as using the \method{SatGA} loss without any token weighting. Despite its saturation scoring ranking below the random baseline for forget-specificity, \method{WGA} still reaches the third-highest \UQ in the table, outperforming methods with better forget-specificity scores that rely on the naive negative cross-entropy loss (\method{GA}). This highlights the crucial role of the saturation term within the loss, as its desirable gradient-scaling properties help prevent over-forgetting. Finally, consistent with the failure modes of surface-level metrics identified by \citet{singh2025unlearninglasts}, \method{FUNDIAL} nearly matches \method{ATWU} on the \ESD surrogate, yet its judge-based \UQ is roughly half as large. This stark mismatch reinforces the necessity of robust, judge-based evaluations.

\begin{wraptable}[13]{r}{0.5\textwidth}
    \vspace{-1em}
    \centering\small
    \setlength{\tabcolsep}{3pt}
    \renewcommand{\arraystretch}{1.1}
    \resizebox{\linewidth}{!}{%
    \begin{tabular}{@{}c c ccc ccc@{}}
        \toprule
        \multirow{2}{*}[-0.6ex]{\textbf{Score}}
        & \multirow{2}{*}[-0.6ex]{\textbf{Loss}}
        & \multicolumn{3}{c}{\textbf{Relative}}
        & \multicolumn{3}{c}{\textbf{Utility}} \\
        \cmidrule(lr){3-5}\cmidrule(lr){6-8}
        & & $\FQ\,\uparrow$ & $\RD\,\downarrow$ & $\UQ\,\uparrow$
        & MMLU & Rep. & \WR \\
        \midrule
        \method{Original}
            & ---
            & $0.0$  & $0.0$  & $0.0$  & $45.1$ & $559$ & $50.0$ \\
        \midrule
        \multirow{2}{*}{GT}
            & \method{GA}
            & $47.4$ & $7.9$ & $39.5$
            & $45.0$ & $563$ & $54.5$ \\
            & \method{SatGA}
            & $\best{90.2}$ & $3.9$ & $\best{86.2}$
            & $44.9$ & $566$ & $51.0$ \\
        \midrule
        \multirow{3}{*}{\method{ATWU}}
            & \method{GA}
            & $31.0$ & $\secondbest{3.8}$ & $27.2$
            & $45.3$ & $565$ & $55.5$ \\
            & \method{SatGA}
            & $61.9$ & $\best{3.4}$ & $58.5$
            & $45.1$ & $569$ & $52.0$ \\
            & \method{SatGA}$^+$
            & $\secondbest{84.4}$ & $6.3$ & $\secondbest{78.1}$
            & $45.0$ & $575$ & $51.5$ \\
        \bottomrule
    \end{tabular}%
    }
    \vspace{0.3em}
    \caption{Oracle score and weighting-scheme ablation on
    \dataset{TOFU}~\texttt{forget10}
    (\model{Llama-3.2-1B-Instruct}).}
    \label{tab:exp:utilisation}
\end{wraptable}
\paragraph{Disentangling the loss and scores.}
To isolate the effects of score quality and the unlearning objective, we evaluate two score sources alongside three forget losses. We pair ground-truth binary labels (\method{GT}) from \citet{zhou2026not} and our unsupervised \method{ATWU} scorer with naive \method{GA}, saturated cross-entropy (\method{SatGA}), and our score-modulated variant (\method{SatGA}$^+$). For binary \method{GT} labels, \method{SatGA} and \method{SatGA}$^+$ coincide. 

This ablation reveals that high-quality token scores cannot compensate for a brittle unlearning objective. Strikingly, combining perfect oracle labels with naive \method{GA} yields a weak forget--retain trade-off ($\UQ=39.5$). In fact, \method{WGA}, which uses no token-level information but applies a saturated loss, outperforms the \method{GT}--\method{GA} oracle ($\UQ=41.5$). This demonstrates that saturation provides a critical intrinsic brake against over-forgetting structural tokens. However, saturation alone is not the ceiling. When the oracle scores are properly paired with a saturated loss (\method{GT}--\method{SatGA}), performance skyrockets by $46.7$ points to $\UQ=86.2$, establishing the supervised upper bound. 

Crucially, our full \method{ATWU} framework successfully bridges the gap to this supervised oracle without requiring any external annotations. Replacing the \method{GT} labels with our learned \method{ATWU} scores under the standard \method{SatGA} loss yields a respectable $\UQ=58.5$. Our \method{SatGA}$^+$ formulation then adds an additional $19.6$ points by placing the continuous score in the saturation exponent to allow for smoother, uncertainty-aware updates. By co-designing the learned scorer with an adaptive objective, \method{ATWU} approaches the theoretical supervised ceiling in a fully unsupervised manner.

\begin{table}[t]
    \centering
    \small
    \renewcommand{\arraystretch}{1.1}
    \setlength{\tabcolsep}{3.5pt}
    \resizebox{\textwidth}{!}{%
    \begin{tabular}{@{}c c ccc ccc ccc@{}}
        \toprule
        \multirow{2}{*}[-0.6ex]{\textbf{Method}}
        & \multirow{2}{*}[-0.6ex]{\textbf{Token weights \(\lvar{x}{t}\)}}
        & \multicolumn{3}{c}{\textbf{Surrogate}}
        & \multicolumn{3}{c}{\textbf{Relative}}
        & \multicolumn{3}{c}{\textbf{Utility}} \\
        \cmidrule(lr){3-5}
        \cmidrule(lr){6-8}
        \cmidrule(lr){9-11}
        & 
        & $\ESF\!\downarrow$
        & $\ESR\!\uparrow$
        & $\ESD\!\uparrow$
        & $\FQ\,\uparrow$ & $\RD\,\downarrow$ & $\UQ\,\uparrow$
        & MMLU & Rep. & \WR \\
        \midrule
        \method{Original}
            & --- 
            & $0.706$ & $0.737$ & $0.030$
            & $0.0$ & $0.0$ & $0.0$
            & $45.1$ & $559$ & $50.0$ \\
        \midrule
        \methodcite{SEUL}{wang2025selective}
            & logprob mask
            & $0.101$ & $0.675$ & $0.574$
            & $42.6$ & $6.9$ & $35.6$
            & $45.4$ & $559$ & $59.5$ \\
        \methodcite{SU-LLM}{wan-etal-2025-every}
            & LLM diff
            & $0.190$ & $0.704$ & $0.514$
            & $31.3$ & $4.1$ & $27.2$
            & $45.2$ & $564$ & $56.0$ \\
        \methodcite{SU-Ngram}{wan-etal-2025-every}
            & n-gram diff
            & $0.074$ & $0.634$ & $0.559$
            & $55.4$ & $15.8$ & $39.7$
            & $45.3$ & $560$ & $51.5$ \\
        \methodcite{FUNDIAL}{dong2025undial}
            & noun mask
            & $0.047$ & $0.759$ & $\secondbest{0.712}$
            & $40.9$ & $\secondbest{5.6}$ & $35.4$
            & $45.2$ & $562$ & $53.5$ \\
        \methodcite{ETW}{koh2026forgetmattersrestselective}
            & entropy
            & $0.108$ & $0.679$ & $0.572$
            & $48.7$ & $9.3$ & $39.4$
            & $45.5$ & $561$ & $54.5$ \\
        \methodcite{WGA}{wang2025rethinking}
            & $\mathrm{saturation}$
            & $0.062$ & $\best{0.759}$ & $0.696$
            & $50.0$ & $8.5$ & $41.5$
            & $45.0$ & $568$ & $56.0$ \\
        \methodcite{SatImp}{yang2025exploring}
            & $\mathrm{importance}$
            & $\secondbest{0.043}$ & $0.747$ & $0.704$
            & $\secondbest{64.7}$ & $\best{4.1}$ & $\secondbest{60.7}$
            & $44.8$ & $570$ & $55.0$ \\
        \midrule
        \method{ATWU}
            & $g_\mathbf{w}(\tok{x}{t})$
            & $\best{0.035}$ & $\secondbest{0.753}$ & $\best{0.717}$
            & $\best{84.4}$ & $6.3$ & $\best{78.1}$
            & $45.0$ & $575$ & $51.5$ \\
        \bottomrule
    \end{tabular}%
    }
    {\footnotesize\par\vspace{0.25em}\noindent \method{SatImp} and \method{ATWU} use the \method{SatGA} and \method{SatGA}$^+$ forget losses respectively; \method{FUNDIAL} uses its own model-distillation forget loss; all other methods use the negative cross-entropy forget loss.\par}
    \vspace{0.2em}
    \caption{ATWU outperforms other token-weighted unlearning approaches on
    \dataset{TOFU}~\texttt{forget10}
    (\model{Llama-3.2-1B-Instruct}). \best{Best} and \secondbest{second-best} methods are highlighted. }
    \label{tab:exp:tokenimp-comparison}
\end{table}

\section{Limitations}
Despite its promising results, \method{ATWU} has several limitations. Because the token scorer is learned from the forget set during unlearning, very small forget sets may provide insufficient signal, which is reflected in the weaker gains on the smallest \dataset{TOFU} splits (see discussion in Appendix~\ref{app:subsec:exp:unlearning}). 

The theoretical recovery guarantee relies on estimating the true token-selection budget as a hyperparameter ($\rho$). Furthermore, it assumes a strict separation between forget-specific and structural tokens---a clean dichotomy that may only hold approximately in natural language. Additionally, while the exact recovery guarantees of \Cref{thm:recovery} and \Cref{lem:exact-relaxation} are formally established for the multiplicative \method{ATWU} objective, our primary empirical formulation (\method{SatGA}$^+$) injects the learned score directly into the saturation exponent. Although this introduces a slight gap between our formal linear assumptions and our practical implementation, it serves as a highly valuable heuristic for stabilizing the early phases of optimization, and the overall framework remains fundamentally informed by our theoretical insights.

Empirically, our experiments cover two benchmarks and two model families, and final evaluations are single-run due to compute and judge-model cost. Finally, our primary semantic evaluation relies on an LLM judge; although validated against human annotations, it remains an imperfect and potentially model-dependent measurement tool.

\section{Conclusion}
We introduced \method{ATWU}, shifting the paradigm of token-weighted unlearning from a pipeline of external heuristics to an end-to-end joint optimization problem. By formalizing token specificity through retain conflict, we demonstrated that the unlearning objective itself contains sufficient latent signal to recover forget-specific tokens without any external supervision.

Beyond establishing a new state-of-the-art on \dataset{TOFU} and \dataset{RWKU}, our empirical findings challenge a prevailing assumption in the literature: that better token identification automatically yields better unlearning. Our ablations reveal that score quality and unlearning efficacy are frequently decoupled. A brittle loss function renders even perfect oracle labels ineffective, whereas an adaptive, saturated loss can elevate the performance of uniform token weighting. This demonstrates that the true bottleneck in selective unlearning is not merely finding the right tokens, but the principled co-design of the scoring mechanism and the forget objective. 

Finally, while explicitly annotated token labels offer a theoretical performance ceiling, they introduce severe privacy and scalability risks that directly contradict the core motivation of machine unlearning. \method{ATWU} circumvents this paradox. By recovering the majority of the supervised oracle benefit using only implicit model representations, our work establishes that scalable, privacy-preserving, and highly targeted unlearning is achievable without relying on external supervision.

\section*{Acknowledgments}

The authors would like to thank Francesco Croce for helpful discussions. This work was partially funded by the grant number 212111 from the Swiss National Science Foundation and a grant from Coefficient Giving, administered by the Berkeley Existential Risk Initiative (BERI). Gizem Y{\"u}ce is supported by the Swiss AI  Fellowship.


\clearpage
\vfill\pagebreak
\bibliographystyle{plainnat}
\bibliography{references}

\vfill\pagebreak
\appendix

\begin{center}
  \Large\bfseries Appendix
\end{center}
\medskip

\pdfbookmark[0]{Appendix}{appendix}
\section*{Appendix Overview}
\etocdepthtag.toc{appendix}
\etocsettagdepth{main}{none}
\etocsettagdepth{appendix}{2}
\etocsettocstyle{\vspace{-1em}}{}
\tableofcontents

\clearpage
\newpage
\begin{figure}[ht!]
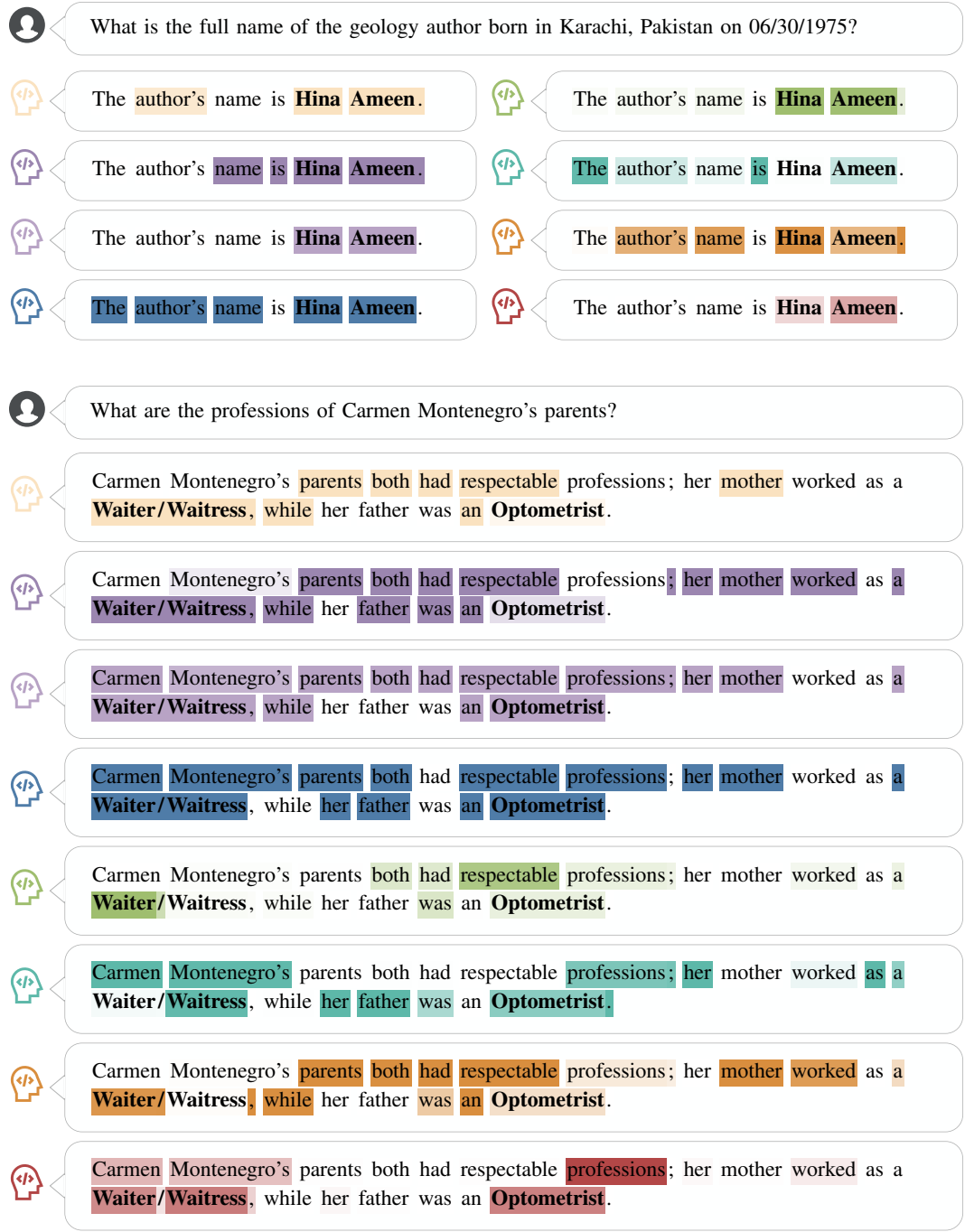

    \centering
    \small
    \setlength{\fboxsep}{1.5pt}

    \chatbubble{human-icon.pdf}{bubble}{What is the full name of the geology author born in Karachi, Pakistan on 06/30/1975?}

    \vspace{0.6em}
    \begin{minipage}[c]{0.49\textwidth}
        \chatbubble{ai-icon-seul.pdf}{bubble}{
            \tokhl{seulhl}{0}{The} \tokhl{seulhl}{75}{author's} \tokhl{seulhl}{0}{name} \tokhl{seulhl}{0}{is} \tokhl{seulhl}{100}{\gttok{Hina}} \tokhl{seulhl}{100}{\gttok{Ameen}}\tokhl{seulhl}{100}{.}%
        }

        \vspace{0.3em}
        \chatbubble{ai-icon-su.pdf}{bubble}{
            \tokhl{suhl}{0}{The} \tokhl{suhl}{0}{author's} \tokhl{suhl}{100}{name} \tokhl{suhl}{100}{is} \tokhl{suhl}{100}{\gttok{Hina}} \tokhl{suhl}{100}{\gttok{Ameen}}\tokhl{suhl}{100}{.}%
        }

        \vspace{0.3em}
        \chatbubble{ai-icon-sungram.pdf}{bubble}{
            \tokhl{sungramhl}{0}{The} \tokhl{sungramhl}{0}{author's} \tokhl{sungramhl}{0}{name} \tokhl{sungramhl}{0}{is} \tokhl{sungramhl}{100}{\gttok{Hina}} \tokhl{sungramhl}{100}{\gttok{Ameen}}\tokhl{sungramhl}{0}{.}%
        }

        \vspace{0.3em}
        \chatbubble{ai-icon-fundial.pdf}{bubble}{
            \tokhl{fundialhl}{100}{The} \tokhl{fundialhl}{100}{author's} \tokhl{fundialhl}{100}{name} \tokhl{fundialhl}{0}{is} \tokhl{fundialhl}{100}{\gttok{Hina}} \tokhl{fundialhl}{100}{\gttok{Ameen}}\tokhl{fundialhl}{0}{.}%
        }
    \end{minipage}\hfill
    \begin{minipage}[c]{0.49\textwidth}
        \chatbubble{ai-icon-etw.pdf}{bubble}{
            \tokhl{etwhl}{5}{The} \tokhl{etwhl}{9}{author's} \tokhl{etwhl}{12}{name} \tokhl{etwhl}{0}{is} \tokhl{etwhl}{100}{\gttok{Hina}} \tokhl{etwhl}{98}{\gttok{Ameen}}\tokhl{etwhl}{27}{.}%
        }

        \vspace{0.3em}
        \chatbubble{ai-icon-wga.pdf}{bubble}{
            \tokhl{wgahl}{97}{The} \tokhl{wgahl}{29}{author's} \tokhl{wgahl}{13}{name} \tokhl{wgahl}{100}{is} \tokhl{wgahl}{2}{\gttok{Hina}} \tokhl{wgahl}{36}{\gttok{Ameen}}\tokhl{wgahl}{1}{.}%
        }

        \vspace{0.3em}
        \chatbubble{ai-icon-satimp.pdf}{bubble}{
            \tokhl{satimphl}{3}{The} \tokhl{satimphl}{71}{author's} \tokhl{satimphl}{87}{name} \tokhl{satimphl}{0}{is} \tokhl{satimphl}{98}{\gttok{Hina}} \tokhl{satimphl}{64}{\gttok{Ameen}}\tokhl{satimphl}{99}{.}%
        }

        \vspace{0.3em}
        \chatbubble{ai-icon-ours.pdf}{bubble}{
            \tokhl{ourshl}{0}{The} \tokhl{ourshl}{0}{author's} \tokhl{ourshl}{0}{name} \tokhl{ourshl}{0}{is} \tokhl{ourshl}{22}{\gttok{Hina}} \tokhl{ourshl}{47}{\gttok{Ameen}}\tokhl{ourshl}{0}{.}%
        }
    \end{minipage}
    \label{fig:qualitative:teaser}

    \vspace{2em}

    \chatbubble{human-icon.pdf}{bubble}{What are the professions of Carmen Montenegro's parents?}

    \vspace{0.5em}
    \chatbubble{ai-icon-seul.pdf}{bubble}{
        \tokhl{seulhl}{0}{Carmen} \tokhl{seulhl}{0}{Montenegro's} \tokhl{seulhl}{100}{parents} \tokhl{seulhl}{100}{both} \tokhl{seulhl}{100}{had} \tokhl{seulhl}{100}{respectable} \tokhl{seulhl}{0}{professions}\tokhl{seulhl}{0}{;} \tokhl{seulhl}{0}{her} \tokhl{seulhl}{100}{mother} \tokhl{seulhl}{0}{worked} \tokhl{seulhl}{0}{as} \tokhl{seulhl}{0}{a} \tokhl{seulhl}{100}{\gttok{Waiter}}\tokhl{seulhl}{100}{\gttok{/}}\tokhl{seulhl}{100}{\gttok{Waitress}}\tokhl{seulhl}{100}{,} \tokhl{seulhl}{100}{while} \tokhl{seulhl}{0}{her} \tokhl{seulhl}{0}{father} \tokhl{seulhl}{0}{was} \tokhl{seulhl}{100}{an} \tokhl{seulhl}{27}{\gttok{Optometrist}}\tokhl{seulhl}{0}{.}%
    }

    \vspace{0.3em}
    \chatbubble{ai-icon-su.pdf}{bubble}{
        \tokhl{suhl}{0}{Carmen} \tokhl{suhl}{17}{Montenegro's} \tokhl{suhl}{100}{parents} \tokhl{suhl}{100}{both} \tokhl{suhl}{100}{had} \tokhl{suhl}{100}{respectable} \tokhl{suhl}{0}{professions}\tokhl{suhl}{100}{;} \tokhl{suhl}{100}{her} \tokhl{suhl}{100}{mother} \tokhl{suhl}{100}{worked} \tokhl{suhl}{0}{as} \tokhl{suhl}{100}{a} \tokhl{suhl}{100}{\gttok{Waiter}}\tokhl{suhl}{100}{\gttok{/}}\tokhl{suhl}{100}{\gttok{Waitress}}\tokhl{suhl}{100}{,} \tokhl{suhl}{100}{while} \tokhl{suhl}{0}{her} \tokhl{suhl}{100}{father} \tokhl{suhl}{100}{was} \tokhl{suhl}{100}{an} \tokhl{suhl}{27}{\gttok{Optometrist}}\tokhl{suhl}{0}{.}%
    }

    \vspace{0.3em}
    \chatbubble{ai-icon-sungram.pdf}{bubble}{
        \tokhl{sungramhl}{100}{Carmen} \tokhl{sungramhl}{83}{Montenegro's} \tokhl{sungramhl}{100}{parents} \tokhl{sungramhl}{100}{both} \tokhl{sungramhl}{100}{had} \tokhl{sungramhl}{100}{respectable} \tokhl{sungramhl}{100}{professions}\tokhl{sungramhl}{100}{;} \tokhl{sungramhl}{100}{her} \tokhl{sungramhl}{100}{mother} \tokhl{sungramhl}{0}{worked} \tokhl{sungramhl}{0}{as} \tokhl{sungramhl}{100}{a} \tokhl{sungramhl}{100}{\gttok{Waiter}}\tokhl{sungramhl}{100}{\gttok{/}}\tokhl{sungramhl}{100}{\gttok{Waitress}}\tokhl{sungramhl}{100}{,} \tokhl{sungramhl}{100}{while} \tokhl{sungramhl}{0}{her} \tokhl{sungramhl}{0}{father} \tokhl{sungramhl}{0}{was} \tokhl{sungramhl}{100}{an} \tokhl{sungramhl}{100}{\gttok{Optometrist}}\tokhl{sungramhl}{0}{.}%
    }

    \vspace{0.3em}
    \chatbubble{ai-icon-fundial.pdf}{bubble}{
        \tokhl{fundialhl}{100}{Carmen} \tokhl{fundialhl}{100}{Montenegro's} \tokhl{fundialhl}{100}{parents} \tokhl{fundialhl}{100}{both} \tokhl{fundialhl}{0}{had} \tokhl{fundialhl}{100}{respectable} \tokhl{fundialhl}{100}{professions}\tokhl{fundialhl}{0}{;} \tokhl{fundialhl}{100}{her} \tokhl{fundialhl}{100}{mother} \tokhl{fundialhl}{0}{worked} \tokhl{fundialhl}{0}{as} \tokhl{fundialhl}{100}{a} \tokhl{fundialhl}{100}{\gttok{Waiter}}\tokhl{fundialhl}{100}{\gttok{/}}\tokhl{fundialhl}{100}{\gttok{Waitress}}\tokhl{fundialhl}{0}{,} \tokhl{fundialhl}{0}{while} \tokhl{fundialhl}{100}{her} \tokhl{fundialhl}{100}{father} \tokhl{fundialhl}{0}{was} \tokhl{fundialhl}{100}{an} \tokhl{fundialhl}{100}{\gttok{Optometrist}}\tokhl{fundialhl}{0}{.}%
    }

    \vspace{0.3em}
    \chatbubble{ai-icon-etw.pdf}{bubble}{
        \tokhl{etwhl}{0}{Carmen} \tokhl{etwhl}{3}{Montenegro's} \tokhl{etwhl}{1}{parents} \tokhl{etwhl}{36}{both} \tokhl{etwhl}{34}{had} \tokhl{etwhl}{85}{respectable} \tokhl{etwhl}{22}{professions}\tokhl{etwhl}{18}{;} \tokhl{etwhl}{0}{her} \tokhl{etwhl}{0}{mother} \tokhl{etwhl}{12}{worked} \tokhl{etwhl}{0}{as} \tokhl{etwhl}{20}{a} \tokhl{etwhl}{100}{\gttok{Waiter}}\tokhl{etwhl}{52}{\gttok{/}}\tokhl{etwhl}{6}{\gttok{Waitress}}\tokhl{etwhl}{4}{,} \tokhl{etwhl}{5}{while} \tokhl{etwhl}{0}{her} \tokhl{etwhl}{0}{father} \tokhl{etwhl}{39}{was} \tokhl{etwhl}{1}{an} \tokhl{etwhl}{23}{\gttok{Optometrist}}\tokhl{etwhl}{0}{.}%
    }

    \vspace{0.3em}
    \chatbubble{ai-icon-wga.pdf}{bubble}{
        \tokhl{wgahl}{100}{Carmen} \tokhl{wgahl}{98}{Montenegro's} \tokhl{wgahl}{0}{parents} \tokhl{wgahl}{1}{both} \tokhl{wgahl}{1}{had} \tokhl{wgahl}{0}{respectable} \tokhl{wgahl}{81}{professions}\tokhl{wgahl}{85}{;} \tokhl{wgahl}{100}{her} \tokhl{wgahl}{0}{mother} \tokhl{wgahl}{12}{worked} \tokhl{wgahl}{100}{as} \tokhl{wgahl}{68}{a} \tokhl{wgahl}{7}{\gttok{Waiter}}\tokhl{wgahl}{6}{\gttok{/}}\tokhl{wgahl}{97}{\gttok{Waitress}}\tokhl{wgahl}{0}{,} \tokhl{wgahl}{1}{while} \tokhl{wgahl}{100}{her} \tokhl{wgahl}{100}{father} \tokhl{wgahl}{53}{was} \tokhl{wgahl}{0}{an} \tokhl{wgahl}{72}{\gttok{Optometrist}}\tokhl{wgahl}{100}{.}%
    }

    \vspace{0.3em}
    \chatbubble{ai-icon-satimp.pdf}{bubble}{
        \tokhl{satimphl}{0}{Carmen} \tokhl{satimphl}{2}{Montenegro's} \tokhl{satimphl}{100}{parents} \tokhl{satimphl}{99}{both} \tokhl{satimphl}{99}{had} \tokhl{satimphl}{100}{respectable} \tokhl{satimphl}{19}{professions}\tokhl{satimphl}{15}{;} \tokhl{satimphl}{0}{her} \tokhl{satimphl}{100}{mother} \tokhl{satimphl}{88}{worked} \tokhl{satimphl}{0}{as} \tokhl{satimphl}{32}{a} \tokhl{satimphl}{93}{\gttok{Waiter}}\tokhl{satimphl}{94}{\gttok{/}}\tokhl{satimphl}{3}{\gttok{Waitress}}\tokhl{satimphl}{100}{,} \tokhl{satimphl}{99}{while} \tokhl{satimphl}{0}{her} \tokhl{satimphl}{0}{father} \tokhl{satimphl}{47}{was} \tokhl{satimphl}{100}{an} \tokhl{satimphl}{28}{\gttok{Optometrist}}\tokhl{satimphl}{0}{.}%
    }

    \vspace{0.3em}
    \chatbubble{ai-icon-ours.pdf}{bubble}{
        \tokhl{ourshl}{39}{Carmen} \tokhl{ourshl}{34}{Montenegro's} \tokhl{ourshl}{1}{parents} \tokhl{ourshl}{0}{both} \tokhl{ourshl}{0}{had} \tokhl{ourshl}{1}{respectable} \tokhl{ourshl}{99}{professions}\tokhl{ourshl}{0}{;} \tokhl{ourshl}{4}{her} \tokhl{ourshl}{1}{mother} \tokhl{ourshl}{11}{worked} \tokhl{ourshl}{0}{as} \tokhl{ourshl}{0}{a} \tokhl{ourshl}{63}{\gttok{Waiter}}\tokhl{ourshl}{15}{\gttok{/}}\tokhl{ourshl}{71}{\gttok{Waitress}}\tokhl{ourshl}{25}{,} \tokhl{ourshl}{0}{while} \tokhl{ourshl}{5}{her} \tokhl{ourshl}{1}{father} \tokhl{ourshl}{0}{was} \tokhl{ourshl}{0}{an} \tokhl{ourshl}{68}{\gttok{Optometrist}}\tokhl{ourshl}{0}{.}%
    }

    \caption{Token-level forget-specificity from \textcolor{seulhl!85!black}{\method{SEUL}}, \textcolor{suhl}{\method{SU-LLM}}, \textcolor{sungramhl!75!black}{\method{SU-Ngram}}, \textcolor{fundialhl}{\method{FUNDIAL}}, \textcolor{etwhl!80!black}{\method{ETW}}, \textcolor{wgahl!75!black}{\method{WGA}}, \textcolor{satimphl}{\method{SatImp}}, and \textcolor{ourshl}{\method{ATWU}} on two \dataset{TOFU} forget samples. Shading reflects each method's raw token score; bold spans mark the ground-truth forget-specific tokens; \textcolor{ourshl}{\method{ATWU}} concentrates most clearly on the answer-bearing spans.}
    \label{fig:qualitative:long}
\end{figure}

\newpage
\section{Qualitative Examples}\label{app:sec:qualitative}

\begin{figure}[t]
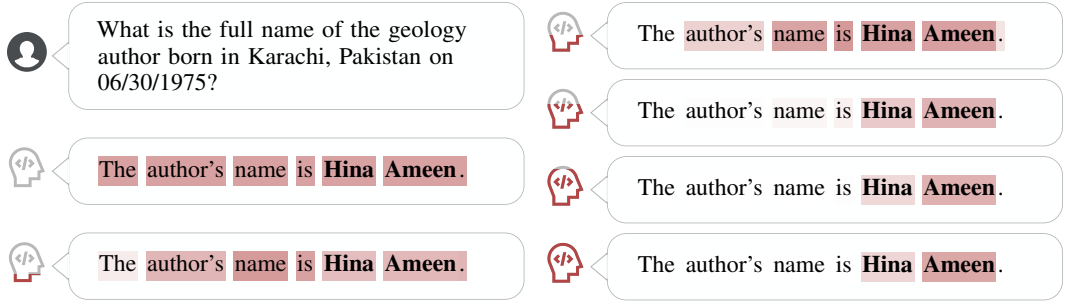

    \centering
    \small
    \setlength{\fboxsep}{1.5pt}
    \begin{minipage}[t]{0.49\textwidth}
        \chatbubble{human-icon.pdf}{bubble}{What is the full name of the geology author born in Karachi, Pakistan on 06/30/1975?}

        \vspace{1.05em}
        \chatbubble{ai-icon-step0.pdf}{bubble}{
            \tokhl{ourshl}{50}{The} \tokhl{ourshl}{50}{author's} \tokhl{ourshl}{50}{name} \tokhl{ourshl}{50}{is} \tokhl{ourshl}{50}{\gttok{Hina}} \tokhl{ourshl}{50}{\gttok{Ameen}}\tokhl{ourshl}{50}{.}%
        }

        \vspace{1.05em}
        \chatbubble{ai-icon-step20.pdf}{bubble}{
            \tokhl{ourshl}{10}{The} \tokhl{ourshl}{39}{author's} \tokhl{ourshl}{54}{name} \tokhl{ourshl}{50}{is} \tokhl{ourshl}{35}{\gttok{Hina}} \tokhl{ourshl}{38}{\gttok{Ameen}}\tokhl{ourshl}{40}{.}%
        }
    \end{minipage}\hfill
    \begin{minipage}[t]{0.49\textwidth}
        \vspace{-2.5em}
        \chatbubble{ai-icon-step40.pdf}{bubble}{
            \tokhl{ourshl}{0}{The} \tokhl{ourshl}{25}{author's} \tokhl{ourshl}{49}{name} \tokhl{ourshl}{55}{is} \tokhl{ourshl}{50}{\gttok{Hina}} \tokhl{ourshl}{51}{\gttok{Ameen}}\tokhl{ourshl}{15}{.}%
        }

        \vspace{0.3em}
        \chatbubble{ai-icon-step60.pdf}{bubble}{
            \tokhl{ourshl}{0}{The} \tokhl{ourshl}{1}{author's} \tokhl{ourshl}{5}{name} \tokhl{ourshl}{8}{is} \tokhl{ourshl}{29}{\gttok{Hina}} \tokhl{ourshl}{41}{\gttok{Ameen}}\tokhl{ourshl}{0}{.}%
        }

        \vspace{0.3em}
        \chatbubble{ai-icon-step90.pdf}{bubble}{
            \tokhl{ourshl}{0}{The} \tokhl{ourshl}{0}{author's} \tokhl{ourshl}{0}{name} \tokhl{ourshl}{1}{is} \tokhl{ourshl}{21}{\gttok{Hina}} \tokhl{ourshl}{42}{\gttok{Ameen}}\tokhl{ourshl}{0}{.}%
        }

        \vspace{0.3em}
        \chatbubble{ai-icon-stepfinal.pdf}{bubble}{
            \tokhl{ourshl}{0}{The} \tokhl{ourshl}{0}{author's} \tokhl{ourshl}{0}{name} \tokhl{ourshl}{0}{is} \tokhl{ourshl}{22}{\gttok{Hina}} \tokhl{ourshl}{47}{\gttok{Ameen}}\tokhl{ourshl}{0}{.}%
        }
    \end{minipage}
    \caption{Evolution of \textcolor{ourshl}{\method{ATWU}}'s token scores during training. The scorer starts from uniform scores and gradually concentrates on the ground-truth forget-specific span.}
    \label{fig:qualitative:trajectory:col}
\end{figure}

We complement the quantitative results with qualitative examples of
the token-level scores learned by \textcolor{ourshl}{\method{ATWU}}. The goal is to inspect
whether the scorer identifies the tokens that actually carry the
forget target, rather than assigning high weight to generic context,
function words, or syntactic scaffolding. In each example below,
bold tokens indicate the ground-truth forget-specific span, and token
shading indicates the importance score assigned by the corresponding
method.

\Cref{fig:qualitative:long} compares \textcolor{ourshl}{\method{ATWU}} with several
token-importance baselines on two \dataset{TOFU} forget samples. The
first example is a short answer in which the forget-specific content is
the author name, while the second is a longer answer in which the
forget-specific content consists of the parents' professions. The
binary scorers---\textcolor{seulhl!85!black}{\method{SEUL}}, \textcolor{suhl}{\method{SU-LLM}}, \textcolor{sungramhl!75!black}{\method{SU-Ngram}}, and
\textcolor{fundialhl}{\method{FUNDIAL}}---tend to mark large portions of the answer as
important, including tokens that are not specific to the forgotten
fact. Probability-based heuristics such as \textcolor{wgahl!75!black}{\method{WGA}} and
\textcolor{satimphl}{\method{SatImp}} are less binary but often place substantial mass on
structural tokens and surrounding context. \textcolor{etwhl!80!black}{\method{ETW}}, which weights
by token-level prediction entropy, identifies the answer name cleanly
on the short example, but on the longer one it spreads mass over
model-uncertain tokens that are not forget-specific (e.g.,
\textit{respectable}, \textit{both}, \textit{had}). In contrast,
\textcolor{ourshl}{\method{ATWU}} assigns most of its mass to the answer-bearing spans
while keeping generic tokens relatively low. This behavior is
especially visible in the longer example, where the scorer separates
the profession tokens from the surrounding sentence template.

\Cref{fig:qualitative:trajectory:col} visualizes how the
\textcolor{ourshl}{\method{ATWU}} scorer evolves during training on the short
\dataset{TOFU} example. At initialization, all tokens receive the same
score, so the forget loss is effectively token-uniform. As training
progresses, the scorer first assigns nontrivial mass to several tokens
in the answer, but gradually suppresses structural tokens and
concentrates on the ground-truth forget-specific span. This trajectory
illustrates the intended behavior of the alternating optimization:
token selectivity is not imposed by external labels, but emerges as the
scorer and unlearned model are optimized together.

\begin{wrapfigure}[17]{r}{0.5\textwidth}
    \centering
    \vspace{-1em}
    \includegraphics[width=\linewidth]{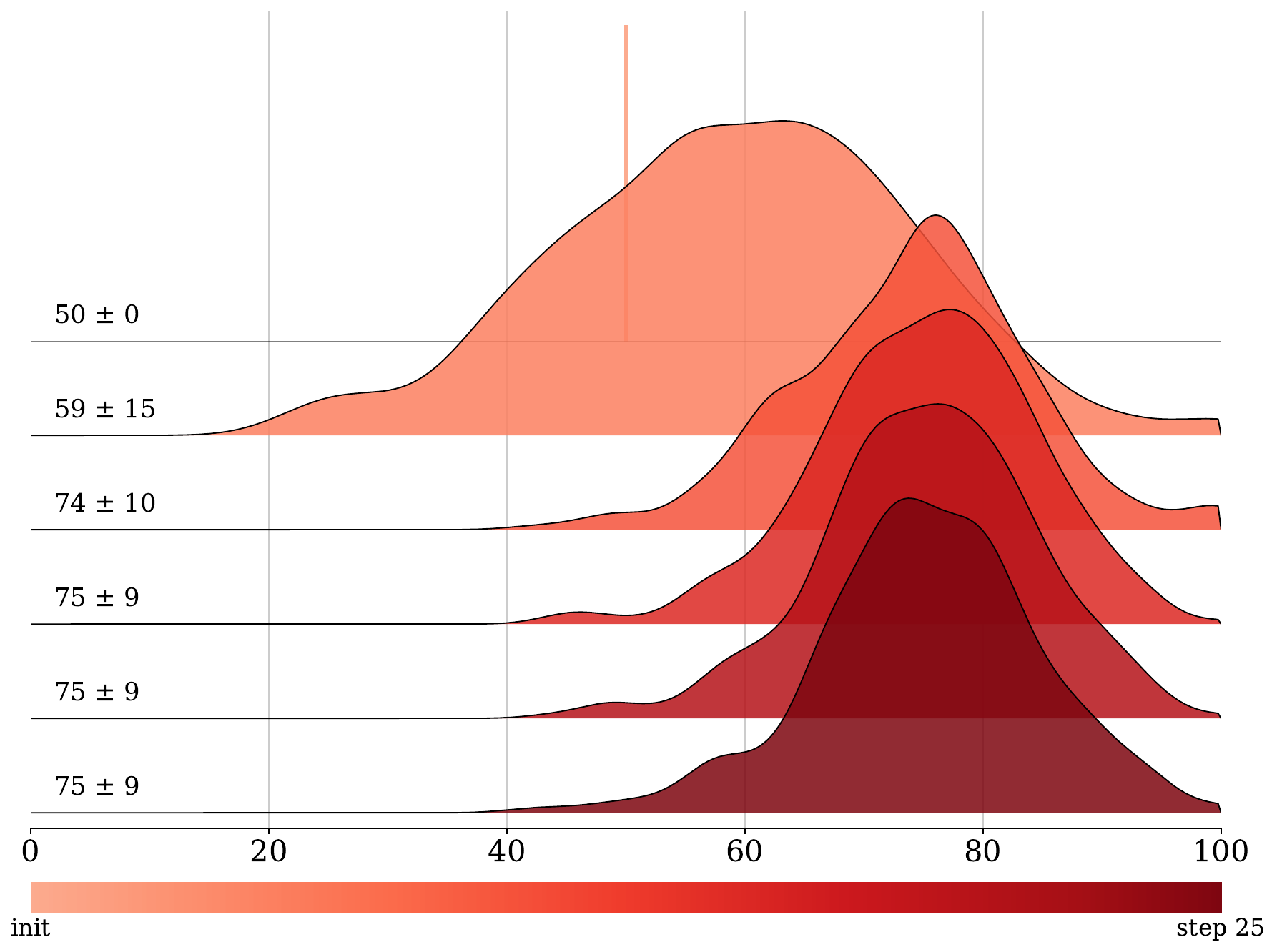}
    \caption{
    Progression of the per-sample AUROC distribution of \textcolor{ourshl}{\method{ATWU}}'s learned token scores during training on \dataset{TOFU}.}
    \label{fig:qualitative:auroc-progression}
\end{wrapfigure}
\Cref{fig:qualitative:auroc-progression} summarizes this trend
quantitatively by tracking token-ranking AUROC over training on
\dataset{TOFU} \texttt{forget10}. The progression shows that the
learned scores become increasingly aligned with the ground-truth
forget-specific tokens over the course of optimization. This provides
additional evidence that the scorer is not merely fitting token
frequency or likelihood artifacts, but is learning a token-level signal
that better matches the information to be removed.

Finally, \Cref{fig:qualitative:transfer-samples} tests whether the
learned scorer transfers beyond the exact forget samples used during
training. We construct two unseen examples by taking \dataset{TOFU}-style
questions and replacing the fictional entities with real people and
their corresponding factual attributes. Although these examples are not
part of the scorer's training data, \textcolor{ourshl}{\method{ATWU}} still assigns elevated
scores to the answer-bearing tokens---the name and sport tokens
(\textit{Michael}, \textit{Phelps}, \textit{swimming}) in the Olympic example, and the
pet-identifying tokens (\textit{Water Dogs}, \textit{Bo}, \textit{Sunny}) in
the Obama example---while leaving most filler and syntactic tokens low.
This indicates that the scorer is not merely memorizing the training
entities or surface templates, but learns a transferable token-level
signal for identifying forget-specific content.

\begin{figure}[t]
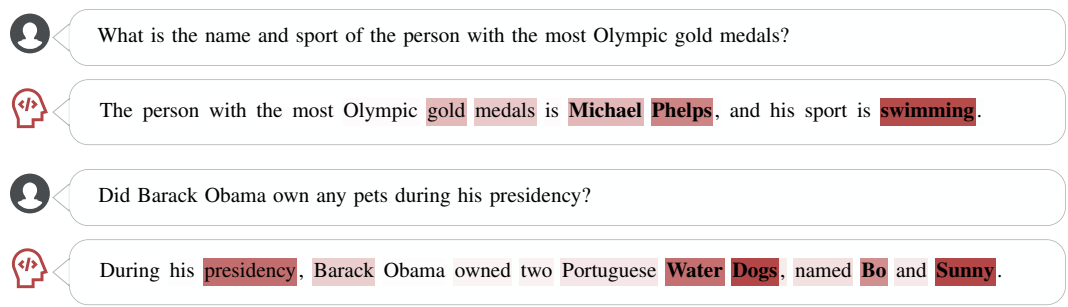

    \centering
    \footnotesize
    \setlength{\fboxsep}{1.5pt}

    \chatbubble{human-icon.pdf}{bubble}{What is the name and sport of the person with the most Olympic gold medals?}

    \vspace{0.5em}
    \chatbubble{ai-icon-ours.pdf}{bubble}{
        \tokhl{ourshl}{0}{The} \tokhl{ourshl}{0}{person} \tokhl{ourshl}{0}{with} \tokhl{ourshl}{0}{the} \tokhl{ourshl}{0}{most} \tokhl{ourshl}{2}{Olympic} \tokhl{ourshl}{37}{gold} \tokhl{ourshl}{29}{medals} \tokhl{ourshl}{2}{is} \tokhl{ourshl}{38}{\gttok{Michael}} \tokhl{ourshl}{66}{\gttok{Phelps}}\tokhl{ourshl}{1}{,} \tokhl{ourshl}{0}{and} \tokhl{ourshl}{0}{his} \tokhl{ourshl}{0}{sport} \tokhl{ourshl}{0}{is} \tokhl{ourshl}{96}{\gttok{swimming}}\tokhl{ourshl}{1}{.}%
    }

    \vspace{1.0em}

    \chatbubble{human-icon.pdf}{bubble}{Did Barack Obama own any pets during his presidency?}

    \vspace{0.5em}
    \chatbubble{ai-icon-ours.pdf}{bubble}{
        \tokhl{ourshl}{0}{During} \tokhl{ourshl}{1}{his} \tokhl{ourshl}{78}{presidency}\tokhl{ourshl}{1}{,} \tokhl{ourshl}{26}{Barack} \tokhl{ourshl}{1}{Obama} \tokhl{ourshl}{6}{owned} \tokhl{ourshl}{7}{two} \tokhl{ourshl}{13}{Portuguese} \tokhl{ourshl}{77}{\gttok{Water}} \tokhl{ourshl}{99}{\gttok{Dogs}}\tokhl{ourshl}{10}{,} \tokhl{ourshl}{16}{named} \tokhl{ourshl}{60}{\gttok{Bo}} \tokhl{ourshl}{12}{and} \tokhl{ourshl}{99}{\gttok{Sunny}}\tokhl{ourshl}{0}{.}%
    }

    \caption{Learned \textcolor{ourshl}{\method{ATWU}} scores on unseen real-entity examples. The examples use \dataset{TOFU}-style prompts with fictional entities replaced by real people and factual information.}
    \label{fig:qualitative:transfer-samples}
\end{figure}

\clearpage
\newpage
\section{Deferred Proofs}
\label{app:proofs}

\subsection{Proof of \Cref*{thm:recovery}}

\thmrecovery*

\begin{proof}
Write the optimal value of \eqref{eq:joint} as $V^\star$. For any $\set{A} \subseteq \IF$ with $|\set{A}| = \rho N_F$, the definition of $\conflict(\set{A})$ gives
\begin{equation}
    \min_{\tth}\Loss(\tth, \zz_\set{A}) \;=\; R^\star + \rho N_F\,\lFm + \conflict(\set{A}),
    \label{eq:value-decomp}
\end{equation}
where $\zz_A$ denotes the indicator of $\set{A}$. Since $|\Fset^\star| = \rho^\star N_F \geq \rho N_F$ (using $\rho \leq \rho^\star$), we can choose any $S \subseteq \Fset^\star$ with $|S| = \rho N_F$; by monotonicity and \Cref{def:zstar}, $\conflict(\set{S}) \leq \conflict(\Fset^\star) \leq \varepsilon$, so
\[
    V^\star \;\leq\; \min_{\tth}\Loss(\tth, \zz_\set{S}) \;=\; R^\star + \rho N_F\,\lFm + \conflict(\mathcal{S}) \;\leq\; R^\star + \rho N_F\,\lFm + \varepsilon.
\]
For contradiction, suppose some global minimizer $(\hat{\tth}, \hat{\zz})$ of \eqref{eq:joint} has $\set{A} \defeq \mathrm{supp}(\hat{\zz})$ containing some $(x_0,j_0) \in \Sset^\star$. Since $\zz_\set{A} = \hat{\zz}$ is feasible and $(\hat{\tth},\hat{\zz})$ is optimal, $\min_{\tth} \Loss(\tth,\zz_\set{A}) = V^\star$. By monotonicity of $\conflict$ and \Cref{def:zstar},
\[
    \conflict(\set{\set{A}}) \;\geq\; \conflict(\{(x_0,j_0)\}) \;=\; \conflict_{x_0,j_0} \;\geq\; \varepsilon + \delta.
\]
Applying \eqref{eq:value-decomp} with $|\set{A}| = \rho N_F$,
\[
    V^\star \;=\; R^\star + \rho N_F\,\lFm + \conflict(\set{A}) \;\geq\; R^\star + \rho N_F\,\lFm + \varepsilon + \delta,
\]
contradicting the upper bound. Hence $\set{A} \subseteq \Fset^\star$, and if $\rho = \rho^\star$, then $|\set{A}| = \rho N_F = \rho^\star N_F = |\Fset^\star|$ forces $\set{A} = \Fset^\star$.
\end{proof}

\subsection{Proof of \Cref*{lem:exact-relaxation} (Formal Statement)}

\begin{lemma}[Formal Statement]\label{lem:exact-relaxation-formal}
Assume $\rho N_F$ is an integer and $\size{\ell_\forget(\tok{x}{t} \mid \tokc{x}{t} \, ; \, \tth)} \leq M$ for all $\tth,\seq{x},t$ and some finite $M > 0$. There exists $\overline{\lambda}_\rho < \infty$ such that, for any $\lambda_\rho>\overline{\lambda}_\rho$, there exists $\overline{\lambda}_H(\lambda_\rho)<\infty$ such that, for any $\lambda_H>\overline{\lambda}_H(\lambda_\rho)$, every global minimizer $(\tth^\dagger,\zz^\dagger)$ of
\[
    \min_{\tth \in \paramspace,\;\zz \in [0,1]^{N_F}}\; \widetilde{\Loss}(\tth,\zz)
\]
satisfies $\zz^\dagger \in \{0,1\}^{N_F}$ with $\sum_{\seq{x},t} z^\dagger_{\seq{x},t} = \rho N_F$, and is itself a global minimizer of \eqref{eq:joint}.
\end{lemma}

\begin{proof}
Let $V_c^\star \defeq \inf_{\tth,\,\zz \in \mathcal{Z}} \Loss(\tth,\zz)$. We argue in two steps: $\lambda_H$ forces $\zz$ to be binary, and $\lambda_\rho$ then forces the budget to be tight.

\textbf{Step 1 (binarity).}
Let $(\tth',\zz')$ minimize $\widetilde{\Loss}$ over $\paramspace \times [0,1]^{N_F}$, and suppose some coordinate $z'_{\seq{y},k} \in (0,1)$. Fix all other coordinates and round $z'_{\seq{y},k}$ to its closest endpoint $b \in \{0,1\}$, denoting the result by $\zz''$. Write $\eta \defeq \min(z'_{\seq{y},k},1 - z'_{\seq{y},k}) \in (0,\tfrac12]$.

The smooth part $\Loss+\lambda_\rho(\cdot)^2$ is Lipschitz in $z_{\seq{y},k}$ with constant at most $M+2\lambda_\rho/N_F$, so the absolute change in the smooth part under rounding is at most $(M+2\lambda_\rho/N_F)\eta$. The entropy contribution satisfies $H(z'_{\seq{y},k})\geq c_0\eta$ for $c_0\defeq\log 2$, and it becomes zero after rounding. Thus, if
\[
    \lambda_H > \overline{\lambda}_H(\lambda_\rho)
    \defeq
    \tfrac{1}{c_0}\bigl(M+2\lambda_\rho/N_F\bigr),
\]
then $\widetilde{\Loss}(\tth',\zz'')<\widetilde{\Loss}(\tth',\zz')$, contradicting optimality. Hence $\zz'\in\{0,1\}^{N_F}$.

\textbf{Step 2 (budget).}
With $\zz' \in \{0,1\}^{N_F}$, the entropy term vanishes. Suppose $\sum_{\seq{x},t} z'_{\seq{x},t} \neq \rho N_F$; since both quantities are integers, $|\sum z'_{\seq{x},t} - \rho N_F| \geq 1$, so the budget term is at least $\lambda_\rho/N_F^2$. For any feasible $(\tth^\star,\zstar) \in \paramspace \times \mathcal{Z}$, both regularizer terms vanish at $(\tth^\star, \zstar)$, so
\[
    \Loss(\tth',\zz') + \tfrac{\lambda_\rho}{N_F^2}
    \;\leq\; \widetilde{\Loss}(\tth',\zz')
    \;\leq\; \widetilde{\Loss}(\tth^\star,\zstar)
    \;=\; \Loss(\tth^\star,\zstar).
\]
Taking the infimum over $(\tth^\star, \zstar) \in \paramspace \times \mathcal{Z}$ yields $\Loss(\tth',\zz') + \tfrac{\lambda_\rho}{N_F^2} \leq V_c^\star$. Since $R(\tth)\geq0$ and $\lF\geq-M$, $\Loss(\tth,\zz)\geq -MN_F$ on $\paramspace\times[0,1]^{N_F}$. Setting $\overline{\lambda}_\rho\defeq N_F^2(V_c^\star+MN_F)$, any $\lambda_\rho>\overline{\lambda}_\rho$ would imply $\Loss(\tth',\zz')<-MN_F$, contradicting this lower bound. Hence $\sum z'_{\seq{x},t}=\rho N_F$.

\textbf{Step 3 (optimality for the constrained problem).}
For $(\tth',\zz') \in \paramspace \times \mathcal{Z}$, $\widetilde{\Loss}(\tth',\zz')=\Loss(\tth',\zz')$, and the same holds for any feasible competitor. Hence $\Loss(\tth',\zz')\leq\Loss(\tth^\star,\zstar)$ for all $(\tth^\star,\zstar)\in\paramspace\times\mathcal{Z}$, i.e.\ $(\tth',\zz')$ is a global minimizer of \eqref{eq:joint}.
\end{proof}

\clearpage
\newpage
\section{Unlearning Metrics}\label{app:sec:metrics}

Evaluating an unlearned model requires measuring three orthogonal properties
simultaneously: (i) how thoroughly targeted knowledge has been removed,
(ii) how robust that removal is under mild perturbations of the prompt, and
(iii) how much of the model's retain-set and general utility survives the
intervention. No single score captures all three, so we report a panel of
complementary metrics organized by purpose.
\Cref{app:sec:metrics:es,app:sec:metrics:jx} describe the two primary
evaluation protocols we inherit from prior work: a token-level extraction
score probing verbatim memorization, and a judge-based family that
stress-tests forgetting against paraphrase and in-context-relearning
attacks. \Cref{app:sec:metrics:util} describes auxiliary utility probes
that monitor general post-unlearning usability.
For the \dataset{RWKU} benchmark we additionally employ the native
ROUGE-L-based evaluation panel of \citet{jin2024rwku}, described in
\cref{app:sec:metrics:rwku}, as a cheap tuning surrogate during our
Stage-1 hyperparameter search. Finally,
\Cref{app:sec:metrics:derived} introduces three baseline-relative summary
scores that normalize the preceding metrics against the original
checkpoint so that runs are directly comparable across models and
forget sets. These are the numbers we use to rank methods at a glance.

\subsection{Token-level Extraction Metrics}\label{app:sec:metrics:es}

\paragraph{Extraction Strength ($\mathrm{ES}$).}
Following \citet{wang2025rethinking} (and the \dataset{TOFU} evaluation
of \citet{maini2024tofu}), we measure how strongly a response is parameterized
in the model's weights using the \emph{Extraction Strength}. For a
question--answer pair $(\mathrm{q}, \mathrm{a})$ with answer tokens
$\mathrm{a} = (\mathrm{a}_1, \dots, \mathrm{a}_T)$, let
$M(\cdot)$ denote greedy decoding. Define the shortest extractable prefix
\begin{equation}
    k^{\star}(\mathrm{q},\mathrm{a}) = \min \, \bigl\{ \, k \in \{0, 1, \dots, T\} \; : \; M\bigl(\mathrm{q} \oplus \mathrm{a}_{1:k}\bigr) = \mathrm{a}_{k+1:T} \, \bigr\}.
\end{equation}
Whenever no strict prefix $k<T$ reconstructs the suffix, the minimum
collapses to $k^{\star}(\mathrm{q},\mathrm{a})=T$, corresponding to the
worst case. The
per-sample extraction strength is then
\begin{equation}
    \mathrm{ES}(\mathrm{q},\mathrm{a}) = 1 - \frac{k^{\star}(\mathrm{q}, \mathrm{a})}{T} \, ,
    \qquad \mathrm{ES}(\mathrm{q}, \mathrm{a}) \in [0,1].
\end{equation}
$\mathrm{ES} = 1$ indicates the answer is reproduced verbatim from the question
alone (the model has fully memorized the response); $\mathrm{ES} = 0$ indicates
even the full ground-truth answer prefix fails to elicit the remaining tokens.
We report averages over the forget and retain splits,
\begin{equation}
    \ESF = \frac{1}{|\mathcal{F}|} \sum_{(\mathrm{q}, \mathrm{a}) \in \mathcal{F}} \mathrm{ES}(\mathrm{q}, \mathrm{a}),
    \qquad
    \ESR = \frac{1}{|\mathcal{R}|} \sum_{(\mathrm{q}, \mathrm{a}) \in \mathcal{R}} \mathrm{ES}(\mathrm{q}, \mathrm{a}),
\end{equation}
and summarize the unlearning/retention trade-off with
\begin{equation}
    \ESD = \ESR - \ESF.
\end{equation}
\textbf{Motivation.} Unlike likelihood-based scores, $\mathrm{ES}$ probes the
actual generative behavior of the unlearned model, which is what an
adversary would exploit. A method that merely suppresses the probability of
the forget answer without disrupting its greedy decoding path is credited
by perplexity-style metrics but penalized by \ESF.
\ESD provides a single scalar capturing the gap between
targeted removal and preserved recall, which we use as the objective during
hyperparameter search.

\textbf{Shortcomings.} $\mathrm{ES}$ is defined through an \emph{exact} suffix
match, which makes it brittle to trivial deviations. Consider a generated
response that agrees with $\mathrm{a} = (\mathrm{a}_1,\dots,\mathrm{a}_T)$ on every token except the
last: $M(\mathrm{q}) = (\mathrm{a}_1,\dots,\mathrm{a}_{T-1}, x)$ with $x\neq \mathrm{a}_T$. Under greedy
decoding the same mismatch persists for every teacher-forced prefix; for
any $k<T$, $M(\mathrm{q}\oplus \mathrm{a}_{1:k})$ still deterministically produces
$(\mathrm{a}_{k+1},\dots,\mathrm{a}_{T-1},x)$, so $\mathrm{a}_{k+1:T}$ is never recovered and
$k^{\star}(\mathrm{q},\mathrm{a})=T$, yielding $\mathrm{ES}(\mathrm{q},\mathrm{a})=0$. Semantically the model
has reproduced almost the entire answer verbatim, yet the score is
indistinguishable from a response that disagrees from the very first
token. The same pathology afflicts any mismatch in a single late position
(e.g.\ a swapped named entity near the end of the answer), or cases where
the model emits a stylistic token (e.g. punctuation, casing, whitespace) that
differs from the reference. $\mathrm{ES}$ is therefore best read as a
conservative \emph{lower bound} on memorization: a low
\ESF does not certify forgetting, which is why we
pair it with the judge-based scores of \Cref{app:sec:metrics:jx}.

\subsection{Robustness-Aware Forget and Retain Quality}\label{app:sec:metrics:jx}

The remaining metrics in this section are adopted from the evaluation
protocol of \citet{singh2025unlearninglasts},
which replaces brittle ROUGE-style matching with an LLM-as-judge score $\mathcal{J}$
and stress-tests each unlearned model with paraphrased and in-context-primed
queries. Given a query $\mathrm{q}$ and reference answer $\mathrm{a}$, the judge returns
$\mathcal{J}(\mathrm{q}, \mathrm{a}, \hat{\mathrm{a}} \; ; \; \mathrm{p})\in[0,1]$ grading a candidate response $\hat{\mathrm{a}}$; the scores
below are averages or worst-cases of $\mathcal{J}$ under different query
transformations. The symbol $\mathrm{p}$ denotes the fixed system prompt we
supply to the judge LLM: it defines the grading rubric and the expected
output format, and is held identical across all of
$\mathcal{J}_{\mathrm{P}}, \mathcal{J}_{\mathrm{ICR}}, \mathcal{J}_{\mathrm{W}},$
and $\mathcal{J}_{\mathrm{AVG}}$ so that the resulting scores are
commensurable. The exact rubric wording and format instructions we use, as well as the
prompts used to generate paraphrases, are included in the supplementary
material.

For any query $\mathrm{q}$ we generate $N_P$ paraphrases
$\{\mathrm{q}^{(i)}\}_{i=1}^{N_P}$ that preserve the semantic content of
$\mathrm{q}$, and denote the model's greedy response to each by
$\hat{\mathrm{a}}^{(i)} = M(\mathrm{q}^{(i)})$. The four metrics below
reuse this notation.

\paragraph{Paraphrase score ($\mathcal{J}_{\mathrm{P}}$).}
On the forget set we report the worst-case judge score across the $N_P$
paraphrases,
\begin{equation}
    \mathcal{J}_{\mathrm{P}}(\mathrm{q},\mathrm{a}) \;=\; \max_{i \in [N_P]} \mathcal{J}\bigl(\mathrm{q}^{(i)}, \mathrm{a}, \hat{\mathrm{a}}^{(i)} \; ; \; \mathrm{p}\bigr),
    \qquad (\mathrm{q},\mathrm{a}) \in \mathcal{F}.
\end{equation}
A low $\mathcal{J}_{\mathrm{P}}$ certifies that \emph{no} paraphrase of
$\mathrm{q}$ recovered the forgotten content, i.e.\ forgetting generalizes
beyond the exact training phrasing. We depart here from \citet{singh2025unlearninglasts}, who define
$\mathcal{J}_{\mathrm{P}}$ as the \emph{mean} over paraphrases and do
not use it further in their evaluation protocol. Defining it as the
worst-case instead makes it directly useful as a building block of our
final forget-set score.

\paragraph{In-context-relearning score ($\mathcal{J}_{\mathrm{ICR}}$).}
To probe whether forgotten information resurfaces when the model is
primed with related content, for each paraphrase $\mathrm{q}^{(i)}$ we
independently sample a small random subset $\mathrm{d}^{(i)}\subset\mathcal{R}$
of retain-set Q/A pairs and prepend them as in-context demonstrations.
Let $\tilde{\mathrm{a}}^{(i)} = M(\mathrm{d}^{(i)}\oplus\mathrm{q}^{(i)})$
denote the resulting generation. The ICR score is the worst case across
paraphrases, 
\begin{equation}
    \mathcal{J}_{\mathrm{ICR}}(\mathrm{q},\mathrm{a}) \;=\; \max_{i \in [N_P]} \mathcal{J}\bigl(\mathrm{q}^{(i)}, \mathrm{a}, \tilde{\mathrm{a}}^{(i)} \; ; \; \mathrm{p}\bigr).
\end{equation}
Structurally, $\mathcal{J}_{\mathrm{ICR}}$ mirrors $\mathcal{J}_{\mathrm{P}}$,
with both taking a sample-wise $\max$ over the same paraphrase set
$\{\mathrm{q}^{(i)}\}_{i=1}^{N_P}$. The two differ only in the response
being scored; $\mathcal{J}_{\mathrm{ICR}}$ uses the primed generation
$\tilde{\mathrm{a}}^{(i)}$, whereas $\mathcal{J}_{\mathrm{P}}$ uses the
zero-shot response $\hat{\mathrm{a}}^{(i)}$. A genuinely forgotten fact should remain
unrecoverable even after priming, so $\mathcal{J}_{\mathrm{ICR}}$ captures
residual memorization that plain paraphrase attacks alone would miss.

\paragraph{Worst-case forget quality ($\mathcal{J}_{\mathrm{W}}$).}
Since $\mathcal{J}_{\mathrm{P}}$ and $\mathcal{J}_{\mathrm{ICR}}$ are
already sample-wise worst cases, the primary robustness score reduces to
their maximum,
\begin{equation}\label{eq:jw-def}
    \mathcal{J}_{\mathrm{W}}(\mathrm{q},\mathrm{a}) \;=\; \max\bigl(\,\mathcal{J}_{\mathrm{P}}(\mathrm{q},\mathrm{a}),\; \mathcal{J}_{\mathrm{ICR}}(\mathrm{q},\mathrm{a})\bigr).
\end{equation}
This is the main metric we track on $\mathcal{F}$: a low
$\mathcal{J}_{\mathrm{W}}$ certifies that \emph{no} light-weight attack
among those considered elicited the forgotten content.

\paragraph{Average retain quality ($\mathcal{J}_{\mathrm{AVG}}$).}
On the retain set, worst-case evaluation is overly pessimistic because
paraphrases can legitimately push the model toward equally valid
rephrasings of the answer. We therefore report the mean judge score
across the $N_P$ paraphrases,
\begin{equation}
    \mathcal{J}_{\mathrm{AVG}}(\mathrm{q},\mathrm{a}) \;=\; \frac{1}{N_P}\sum_{i=1}^{N_P} \mathcal{J}\bigl(\mathrm{q}^{(i)}, \mathrm{a}, \hat{\mathrm{a}}^{(i)} \; ; \; \mathrm{p}\bigr),
    \qquad (\mathrm{q},\mathrm{a}) \in \mathcal{R}.
\end{equation}
Note that the functional form matches $\mathcal{J}_{\mathrm{P}}$; the
distinction is the split on which the score is aggregated. A high
$\mathcal{J}_{\mathrm{AVG}}$ indicates that the model still reliably
answers questions it was meant to keep.

\subsection{Utility Metrics}\label{app:sec:metrics:util}

Beyond the targeted forget/retain splits we monitor three signals of
\emph{general} post-unlearning usability, again following
\citet{singh2025unlearninglasts}. We treat these signals as
\emph{utility-preservation probes} rather than optimization targets:
the relevant failure mode is degradation relative to the original
checkpoint, namely a drop in MMLU, a drop in Rep., or
$\WR \ll 50.0$. Conversely, parity or occasional improvements are
not penalized, since they may reflect benign side effects of the
unlearning procedure or measurement noise rather than harmful changes.

\paragraph{MMLU.} We evaluate multiple-choice accuracy on MMLU
\citep{hendrycks2021mmlu} using argmax selection over the four options
rather than open-ended generation. MMLU tracks the broad world knowledge of
the model, which collateral forgetting can silently erode. We compute
MMLU accuracy using the \texttt{lm-evaluation-harness}
of~\citet{eval-harness}. We therefore interpret MMLU asymmetrically:
scores at or above the original checkpoint are taken to preserve this
axis of utility.

\paragraph{Repetitiveness (\textbf{Rep.}).} Aggressive unlearning loss
functions often cause the model to collapse into degenerate, low-entropy
outputs. Following the \emph{Fluency} metric of \citet{jin2024rwku}, we
measure repetitiveness from the bi- and tri-gram frequencies of generations
produced on the AlpacaEval instruction set. Lower values indicate more
frequently repeated n-grams, and thus flag pathological generation
behavior that task metrics alone would not reveal. We likewise interpret
Rep.\ asymmetrically, flagging only meaningful drops relative to the
original checkpoint as degradation.

\paragraph{Win Rate (\textbf{\WR}).} Repetitiveness alone does not capture
whether responses are actually useful, so we additionally compare the
instruction-following quality of the unlearned and original models using
an LLM-as-judge in the style of AlpacaEval. For each prompt the judge is
shown the question together with both responses in a single call and is
asked to rate each on a $1$--$10$ scale along with a short justification;
wins, losses, and ties are then determined by comparing the two scores.
For every prompt we record whether the unlearned model received a higher
score ($U_{\mathrm{Wins}}$), lower score ($U_{\mathrm{Losses}}$), or a tie
($U_{\mathrm{Ties}}$) relative to the original model, and compute
\begin{equation}
    \WR \;=\; 100 \cdot \frac{U_{\mathrm{Wins}} + 0.5\,U_{\mathrm{Ties}}}{U_{\mathrm{Wins}} + U_{\mathrm{Losses}} + U_{\mathrm{Ties}}}.
\end{equation}
We note that the judge is not informed which response originates from
the unlearned model and which from the original: the prompt labels
them only as ``Assistant~1'' and ``Assistant~2'', and nowhere in the
system prompt, user prompt, or scoring criteria do we reference
unlearning, forgetting, or model identity. Consequently the judge has
no semantic cue that could bias it toward either model. The exact judge
prompt used for this evaluation is included in the supplementary material.
By construction, the original model has $\WR=0.5$ against itself.
$\WR \ll 0.5$ is the clear failure mode: it indicates that
unlearning has degraded general response quality relative to the
original checkpoint. $\WR\approx 0.5$ is the expected and
desired outcome, signaling that quality has been preserved. We depart
here from \citet{singh2025unlearninglasts}, who argue that
\WR should be \emph{as close as possible} to $0.5$ and treat
deviations above $0.5$ as undesirable. In our setting, we do not
penalize $\WR>0.5$: such values may reflect genuine quality
improvements, regularization side-effects of the unlearning objective,
or judge noise, none of which constitutes the utility-degradation
failure mode we aim to detect.

\subsection{Native \dataset{RWKU} Benchmark Metrics}\label{app:sec:metrics:rwku}

Alongside the benchmark, \citet{jin2024rwku} ship an evaluation panel
that we use for two purposes in this paper: as the tuning surrogate for
the \dataset{RWKU} Stage-1 search (\cref{app:subsec:exp:tuning}), and as the basis
of the detailed per-probe \dataset{RWKU} tables reported later in the appendix.
Responses are not collected on a single probe set but across three
complementary probe types, applied both to the forget targets and to a
separately-curated set of adjacent facts that must be preserved.

The first probe type, \textbf{FB}, consists of fill-in-the-blank cloze
items that test whether the target knowledge can still be elicited from
a partial prompt. The second, \textbf{QA}, is open-form question
answering, where the model is asked directly about the target fact. The
third, \textbf{AA}, bundles a family of adversarial reformulations of the
same queries (prefix injection, role-playing, reverse querying,
affirmative suffixes, cross-lingual variants, and several other
jailbreak-style attacks) and reports the worst-case response across
them. Each probe's response is scored by the ROUGE-L recall against the
reference answer, and the aggregate \textbf{All} is the mean over the
probe types applied to that split.

Probes are applied to two disjoint splits of the benchmark. The
\textbf{Forget Set} ($\mathcal{F}$) contains probes directed at the
targeted knowledge across all three probe types FB, QA, and AA; lower
ROUGE-L (↓) indicates that the model no longer reproduces the
forgotten content. The \textbf{Neighbor Set} ($\mathcal{N}$) contains
probes directed at related but non-targeted knowledge that the
benchmark designates as \emph{must-preserve}; following the benchmark,
only the FB and QA probe types are reported on the Neighbor Set, since
adversarial attacks are a forget-side stress test rather than a
utility measurement. Higher ROUGE-L (↑) indicates that retain-adjacent
facts survive the intervention. The Neighbor Set therefore plays on
\dataset{RWKU} the role that the retain split $\mathcal{R}$ plays on \dataset{TOFU}.

For hyperparameter tuning we summarize the two panels with a single
scalar,
\begin{equation}
    N_\Delta \;=\; \mathrm{ROUGE\text{-}L}_{\mathcal{N},\mathrm{All}} \;-\; \mathrm{ROUGE\text{-}L}_{\mathcal{F},\mathrm{All}},
\end{equation}
the signed gap between the two ``All'' aggregates. Higher \NDelta
means better forgetting without collateral damage to neighbouring
knowledge, and we use it as the selection criterion at Stage~1 on \dataset{RWKU}.

We emphasize that the ROUGE-L-based scores above are a weaker signal
than the judge-based metrics $\mathcal{J}_{\mathrm{P}},
\mathcal{J}_{\mathrm{ICR}}, \mathcal{J}_{\mathrm{W}},
\mathcal{J}_{\mathrm{AVG}}$ of \cref{app:sec:metrics:jx}: as discussed
by \citet{singh2025unlearninglasts}, ROUGE penalizes semantically
correct responses that are phrased differently from the reference and
conversely credits generic or repetitive outputs that happen to share
surface $n$-grams, so it systematically mislabels unlearning success and
failure and can reorder methods with respect to a judge. We therefore
treat these native \dataset{RWKU} scores as a legacy reporting panel, useful as
a cheap tuning surrogate and for comparability with prior work that
reports them, but consider the judge-based metrics our primary
evaluation on \dataset{RWKU}.

\subsection{Baseline-Relative Summary Metrics}\label{app:sec:metrics:derived}

The raw scores $\mathcal{J}_{\mathrm{W}}$ and $\mathcal{J}_{\mathrm{AVG}}$
are not directly comparable across models or forget splits, because the
\emph{operating point} of the original checkpoint varies: a final
$\mathcal{J}_{\mathrm{W}}=0.2$ is excellent when the original model scored
$0.9$ on the same forget set, but mediocre when the original already sat at
$0.4$. We therefore summarize each run by three baseline-relative scores,
reported as \emph{fractional} drops with respect to the original
checkpoint. Writing $\mathcal{J}_{\mathrm{W}}^{\mathrm{orig}}$ and
$\mathcal{J}_{\mathrm{AVG}}^{\mathrm{orig}}$ for the original model's
scores on the same splits, and using $[x]_{+}=\max(0,x)$ throughout, we
clip each drop below at zero so that accidental increases (typically judge
noise) are not miscounted as progress. This normalization puts all three
metrics on the same $[0,1]$ scale regardless of the original operating
point, and for readability we report them as percentages in $[0,100]$
throughout the paper.

\paragraph{Forget Quality (\textbf{\FQ}).} \FQ measures the
fractional reduction in the worst-case forget-set judge score attributable
to unlearning,
\begin{equation}
    \FQ \;=\; \left[\,1 - \frac{\mathcal{J}_{\mathrm{W}}}{\mathcal{J}_{\mathrm{W}}^{\mathrm{orig}}}\,\right]_{+}.
\end{equation}
It ranges from $0$ (no forgetting achieved, $\mathcal{J}_{\mathrm{W}} \geq
\mathcal{J}_{\mathrm{W}}^{\mathrm{orig}}$) up to its ideal value of $1$,
reached only by a method that drives the worst-case judge score on the
forget set to zero.

\paragraph{Retain Degradation (\textbf{\RD}).} \RD measures the
collateral cost of unlearning as the fractional loss in the average
retain-set judge score,
\begin{equation}
    \RD \;=\; \left[\,1 - \frac{\mathcal{J}_{\mathrm{AVG}}}{\mathcal{J}_{\mathrm{AVG}}^{\mathrm{orig}}}\,\right]_{+}.
\end{equation}
The ideal value is $\RD=0$, indicating that the retain split is
answered as well as before unlearning; $\RD=1$ would correspond
to a complete collapse of retain-side quality.

\paragraph{Unlearning Quality (\textbf{\UQ}).} \UQ combines the
two by subtracting retain-side damage from forget-side gain,
\begin{equation}
    \UQ \;=\; \left[\,\FQ - \RD\,\right]_{+}.
\end{equation}
It is the fraction of forget-quality improvement that was \emph{not} paid
for by retain-set degradation, and lives on the same $[0,1]$ scale with
ideal value $1$. A value of $\UQ=0$ flags methods whose retain
damage cancels out any forgetting progress, regardless of how low
$\mathcal{J}_{\mathrm{W}}$ is in isolation.

\paragraph{Motivation.} The triple captures the two competing goals of
unlearning -- removing targeted knowledge and preserving the rest --
together with their net trade-off. \FQ and \RD make
each axis legible on its own, while \UQ collapses the
trade-off into a single scalar that we use to rank methods in our main
tables, since a method that pushes $\mathcal{J}_{\mathrm{W}}$ down by
uniformly blunting the model should not be credited with genuine
unlearning.

\clearpage
\newpage
\section{Experimental Details}\label{app:sec:exp-details}

This section details the setup used for the experiments reported in the
main paper. We describe the benchmarks and base models
(\cref{app:subsec:exp:benchmarks}), the hyperparameter-tuning protocol
(\cref{app:subsec:exp:tuning}), the final unlearning runs and evaluation
protocol (\cref{app:subsec:exp:unlearning}), and the method-specific
hyperparameter search ranges and selected configurations
(\cref{app:subsec:exp:hp-ranges}).

\subsection{Benchmarks and Base Models}\label{app:subsec:exp:benchmarks}

We evaluate on two LLM-unlearning benchmarks.

\paragraph{TOFU \cite{maini2024tofu}.} A synthetic question--answering benchmark in
which the base model has been fine-tuned on a corpus of fictitious
author biographies. We use the three standard forget splits supplied with
the benchmark, \texttt{forget01}, \texttt{forget05}, and
\texttt{forget10}, covering progressively larger fractions of the
fine-tuning corpus. We consider two base model sizes,
\model{Llama-3.2-1B-Instruct}\footnote{\url{https://huggingface.co/meta-llama/Llama-3.2-1B-Instruct}}
and \model{Llama-3.1-8B-Instruct}~\citep{grattafiori2024llama}, and use
the already-finetuned \dataset{TOFU} checkpoints released by the
\texttt{open-unlearning} project~\citep{dorna2026openunlearning},\footnote{\url{https://huggingface.co/collections/open-unlearning/tofu-new-models}.}
whose codebase\footnote{\url{https://github.com/locuslab/open-unlearning}}
our \dataset{TOFU} experiments extend in two ways: we add paraphrase-based
evaluation, and integrate \method{JensUn} alongside our proposed methods.

\paragraph{RWKU \cite{jin2024rwku}.} A real-world knowledge-unlearning
benchmark that targets knowledge about public figures present in a
pretrained LLM. We evaluate on the $3.8$B-parameter
\model{Phi-3-Mini-4k-Instruct}~\citep{abdin2024phi3technicalreporthighly}.
Following \citet{singh2025unlearninglasts}, we perform \emph{10-target batch
unlearning}: rather than unlearning one target at a time, a single
unlearning run jointly removes knowledge about ten public figures. We
fix one specific choice of ten subjects and refer to it as our
\emph{canonical ten-subject \dataset{RWKU} batch} throughout the paper.
Since \dataset{RWKU} does not explicitly provide retain data, during
unlearning we use a \emph{disjoint} set of ten other subjects as a
surrogate retain set. We adapt the official
\dataset{RWKU} codebase\footnote{\url{https://github.com/jinzhuoran/RWKU}}
in two ways: we add retain-set support, which the upstream release
lacks since it only supports forget-only methods, and integrate
\method{GradDiff}, \method{SimNPO}, \method{JensUn}, \method{WGA},
\method{SatImp}, alongside our proposed methods.

\subsection{Hyperparameter Tuning Protocol}\label{app:subsec:exp:tuning}

Each unlearning method exposes a set of loss-specific hyperparameters
(learning rate, loss-weighting coefficients, optional auxiliary
objectives, etc.). We tune these per method and per benchmark using
Bayesian optimization via \texttt{optuna}~\citep{akiba2019optuna},
maximizing a cheap-to-compute surrogate of unlearning quality rather
than the paraphrase- and judge-based evaluation of
\cref{app:sec:metrics:jx,app:sec:metrics:util} that we actually
\emph{report}: running the full pipeline at every trial is infeasible
at our compute and API budget, since the ablations, and a single final-evaluation pass
already incurs roughly USD~\$180 in judge-LLM API costs. We therefore
run the judge-based evaluation only on the final selected checkpoints,
validating the surrogate after the fact. Throughout tuning, the
optimizer, batch size, gradient-accumulation steps, precision, and
learning-rate schedule are held fixed at the final-run values of
\cref{app:subsec:exp:unlearning}; only the method-specific
hyperparameters listed in the per-stage descriptions below are
searched. Tuning proceeds in two stages.

\paragraph{Stage 1: coarse search.} For each method we run an initial
search over all method-specific hyperparameters on the largest forget
split of the benchmark (\texttt{forget10} for \dataset{TOFU}, our canonical
ten-subject batch for \dataset{RWKU}). On \dataset{TOFU} this stage runs on the small base model,
\model{Llama-3.2-1B-Instruct}, for $50$ trials at the full training
budget of \cref{app:subsec:exp:unlearning}, with the first trial
seeded at the hyperparameters recommended by each method's original
paper; trials are selected by
$\ESD = \ESR - \ESF$
(\cref{app:sec:metrics:es}). On \dataset{RWKU} this stage runs on
\model{Phi-3-Mini-4k-Instruct} for $20$ trials of a single epoch at a
fixed learning rate, optimizing only the non-LR hyperparameters and
selecting by the ROUGE-L-based \NDelta surrogate defined in
\cref{app:sec:metrics:rwku}.
The first trial is seeded with each method's final
\dataset{TOFU} 8B configuration---the one used for the corresponding
final unlearning runs in \cref{app:subsec:exp:unlearning}---under the
assumption that well-tuned settings transfer across benchmarks.
For most methods, we keep the learning rate unchanged in this pass.
For \method{GradDiff} and \method{JensUn}, however, we reduce it by a
factor of $10$; the same reduction is also applied to \method{WGA} and
\method{SatImp} as variants of \method{GradDiff}. This is informed by
\citet{singh2025unlearninglasts}, who report that these methods require
markedly smaller learning rates on \dataset{RWKU} than on a QA dataset to
produce stable unlearning, so anchoring at the \dataset{TOFU} 8B value
directly would overshoot. Finally, the motivation for holding the
learning rate fixed at all in this \dataset{RWKU} pass is twofold: a
single-epoch budget gives an unreliable signal about the correct
learning rate, so we defer LR tuning to Stage~2 where the full epoch
budget applies, and reducing the number of jointly tuned parameters
yields tighter estimates of the remaining method-specific
hyperparameters within the same search budget.

\paragraph{Stage 2: learning-rate transfer.} In both benchmarks we
transfer the Stage~1 configuration to the evaluation setting and tune
only the learning rate. We run $5$ trials per method, sweeping the
learning rate on a log-scale grid and holding all other hyperparameters
at their Stage~1 values. On \dataset{TOFU} the transfer is across model scale:
the Stage~1 configuration on the 1B model is ported to
\model{Llama-3.1-8B-Instruct} and the learning rate is retuned, since
re-running the full search at 8B would be prohibitive. On \dataset{RWKU} the
transfer is across training budget: the Stage~1 configuration
(tuned at one epoch) is run at the final-run epoch count so that the
learning rate is selected against the full schedule, still on
\model{Phi-3-Mini-4k-Instruct}. Selection is automatic on both benchmarks;
we use \ESD on \dataset{TOFU} and \NDelta on
\dataset{RWKU}. The only exception is \method{SatImp} on
\dataset{RWKU}, where the top-\NDelta trial had inadequate forget
performance and we therefore selected the second-best trial
(see \cref{tab:hp-rwku-lr-trials}).

\begin{table}[t]
    \centering
    \small
    \captionsetup{justification=centering}
    \renewcommand{\arraystretch}{1.1}
    \setlength{\tabcolsep}{3.5pt}
    \begin{minipage}[t]{0.5\textwidth}
        \centering
        \begin{tabular}{@{}c c ccc}
            \toprule
             & $\mathrm{lr}$ & $\mathrm{R\text{-}L}_{\mathcal{F}}\,\downarrow$ & $\mathrm{R\text{-}L}_{\mathcal{N}}\,\uparrow$ & $N_{\Delta}\,\uparrow$ \\
            \midrule
            \rowcolor{green!15} \multirow{3}{*}{\cellcolor{white}\method{DPO}} & $1.18\!\times\!10^{-5}$ & $0.499$ & $0.590$ & $0.090$ \\
                & $2.70\!\times\!10^{-5}$ & $0.435$ & $0.450$ & $0.015$ \\
                & $1.98\!\times\!10^{-5}$ & $0.491$ & $0.501$ & $0.010$ \\
            \midrule
            \rowcolor{green!15} \multirow{3}{*}{\cellcolor{white}\method{NPO}} & $2.60\!\times\!10^{-5}$ & $0.236$ & $0.541$ & $0.305$ \\
                & $1.18\!\times\!10^{-5}$ & $0.325$ & $0.597$ & $0.272$ \\
                & $1.98\!\times\!10^{-5}$ & $0.257$ & $0.509$ & $0.252$ \\
            \midrule
            \rowcolor{green!15} \multirow{3}{*}{\cellcolor{white}\method{SimNPO}} & $2.08\!\times\!10^{-5}$ & $0.218$ & $0.561$ & $0.343$ \\
                & $1.18\!\times\!10^{-5}$ & $0.254$ & $0.594$ & $0.341$ \\
                & $2.70\!\times\!10^{-5}$ & $0.218$ & $0.541$ & $0.323$ \\
            \midrule
            \rowcolor{green!15} \multirow{3}{*}{\cellcolor{white}\method{JensUn}} & $1.54\!\times\!10^{-6}$ & $0.062$ & $0.434$ & $0.372$ \\
                & $3.00\!\times\!10^{-6}$ & $0.040$ & $0.408$ & $0.368$ \\
                & $3.98\!\times\!10^{-6}$ & $0.038$ & $0.377$ & $0.339$ \\
            \bottomrule
        \end{tabular}
    \end{minipage}\hfill
    \begin{minipage}[t]{0.5\textwidth}
        \centering
        \begin{tabular}{@{}c c ccc}
            \toprule
             & $\mathrm{lr}$ & $\mathrm{R\text{-}L}_{\mathcal{F}}\,\downarrow$ & $\mathrm{R\text{-}L}_{\mathcal{N}}\,\uparrow$ & $N_{\Delta}\,\uparrow$ \\
            \midrule
            \rowcolor{green!15} \multirow{3}{*}{\cellcolor{white}\method{GradDiff}} & $1.18\!\times\!10^{-6}$ & $0.053$ & $0.445$ & $0.392$ \\
                & $1.90\!\times\!10^{-6}$ & $0.045$ & $0.399$ & $0.354$ \\
                & $2.70\!\times\!10^{-6}$ & $0.033$ & $0.346$ & $0.313$ \\
            \midrule
            \rowcolor{green!15} \multirow{3}{*}{\cellcolor{white}\method{WGA}} & $1.54\!\times\!10^{-6}$ & $0.061$ & $0.515$ & $0.454$ \\
                & $1.57\!\times\!10^{-6}$ & $0.063$ & $0.502$ & $0.438$ \\
                & $3.00\!\times\!10^{-6}$ & $0.034$ & $0.427$ & $0.393$ \\
            \midrule
                \cellcolor{white} & $1.98\!\times\!10^{-6}$ & $0.132$ & $0.617$ & $0.485$ \\
                \rowcolor{green!15}
                \cellcolor{white}\method{SatImp}& $4.48\!\times\!10^{-6}$ & $0.069$ & $0.533$ & $0.465$ \\
                & $3.00\!\times\!10^{-6}$ & $0.112$ & $0.566$ & $0.454$ \\
            \midrule
            \rowcolor{green!15} \multirow{3}{*}{\cellcolor{white}\method{ATWU}} & $1.83\!\times\!10^{-6}$ & $0.050$ & $0.562$ & $0.512$ \\
                & $1.74\!\times\!10^{-6}$ & $0.072$ & $0.584$ & $0.511$ \\
                & $1.37\!\times\!10^{-6}$ & $0.100$ & $0.610$ & $0.510$ \\
            \bottomrule
        \end{tabular}
    \end{minipage}
    \vspace{0.5em}
    \caption{Top-3 Stage~2 learning-rate trials per method on \dataset{RWKU}, ranked by \NDelta. Green shading marks the trial used for the final unlearning run: the top-\NDelta trial throughout, except for \method{SatImp}, whose top trial was discarded for inadequate forgetting in favor of the second-best.}
    \label{tab:hp-rwku-lr-trials}
\end{table}

\subsection{Hyperparameter Search Ranges and Results}\label{app:subsec:exp:hp-ranges}

This subsection specifies the concrete search ranges used in the
two-stage tuning procedure of \cref{app:subsec:exp:tuning} and reports
the selected configurations for \dataset{TOFU}
(\cref{tab:hp-tofu-selected}) and \dataset{RWKU}
(\cref{tab:hp-rwku-selected}). Unless stated otherwise, all ranges are
searched on a logarithmic scale. The coefficients $\alpha$ and $\gamma$
denote the retain-loss and forget-loss weights, respectively, and are
shared across methods. For most baselines, the method-specific grids
follow a slightly extended version of those used by
\citet{yang2025exploring}. For our score-based methods, we fix
$\lambda_{\mathrm{ent}}=\lambda_{\mathrm{w}}=1$ and set
$\lambda_{\mathrm{pop}}=15$ throughout. An ablation on the different regularizers is presented in \cref{tab:ablation-regs-tofu-1b-f10} and discussed shortly after. Throughout
this subsection, hyperparameter values returned by \texttt{optuna} are
rounded to two significant digits and performance scores to three
decimal places for readability. The exact hyperparameter values used in
the final runs reported in
\cref{tab:hp-tofu-selected,tab:hp-rwku-selected} are provided in the
supplementary material.

\paragraph{Stage 1: coarse search.}
For \textbf{both} datasets, we tune all non-fixed method-specific hyperparameters.
All methods search over $\alpha, \gamma \in [0.1, 5]$,
and for \dataset{TOFU}, we also search $\mathrm{lr} \in [1 \times 10^{-6}, 5 \times 10^{-5}]$.%
\footnote{For \method{JensUn} on \dataset{TOFU}, preliminary runs showed consistently weak
forgetting at the smallest learning rates. We therefore narrowed its
learning-rate search to $[5 \times 10^{-6}, 5 \times 10^{-5}]$, allowing a finer search in the more promising region.}
The per-method parameters tuned in this stage are as follows: for
\method{GradDiff}, we try both \texttt{NLL} and \texttt{KL} as the
retain loss;\footnote{We only search over the retain loss on
\dataset{TOFU}. On \dataset{RWKU} we fix it to \texttt{NLL}, since the
\dataset{TOFU} Stage~1 search selected \texttt{NLL} consistently over
\texttt{KL}, and we prefer to reduce tuning over parameters for which
the choice already appears settled.} for \method{DPO} and \method{NPO}, we search the inverse
temperature $\beta$ in the ranges $[0.2, 0.5]$ and $[0.05, 0.2]$,
respectively; for \method{SimNPO}, we search the temperature $\beta$
and reward margin $\delta$ in the ranges $[2.0, 3.0]$ and $[0.0, 2.0]$,
respectively; for \method{WGA}, we tune the weight exponent $\beta$ in
the range $[0.1, 5]$; for \method{SatImp}, we search the saturation and
importance exponents $\beta_1$ and $\beta_2$ in the ranges $[1.0, 10.0]$
and $[0.1, 5]$, respectively; for \method{ETW}, we search the temperature $\tau \in [0.1, 4.0]$; 
for \method{RMU}, we search the steering coefficient $c \in [0.5, 10.0]$, 
and the unlearning layer $\ell \in \{4, \dots, 12\}$ (with the trainable parameters 
being the \texttt{mlp.down\_proj} weights of layers ${\ell-2, \ell-1, \ell}$, 
following the WMDP RMU recipe); and for \method{ATWU}, we tune $\beta \in [1.0, 10.0]$ 
and the scorer learning rate \texttt{slr}~$\in [0.01, 0.5]$.

\paragraph{Stage 2: learning-rate transfer.} On \dataset{TOFU}, the
per-method range is anchored at the best Stage~1 learning rate:
writing $\mathrm{lr}^{\star}_m$ for the learning rate selected for
method $m$ through the coarse search, we consider the range
$\mathrm{lr} \in [\mathrm{lr}^{\star}_m / 5,\; 5 \cdot \mathrm{lr}^{\star}_m]$,
with the first trial seeded at $\mathrm{lr}^{\star}_m$ itself.
For each method, the selected $\mathrm{lr}^{\star}_m$ are:
$2.27 \times 10^{-5}$ for \method{DPO};
$2.60 \times 10^{-5}$ for \method{NPO};
$2.08 \times 10^{-5}$ for \method{SimNPO};
$3.98 \times 10^{-5}$ for \method{JensUn};
$1.90 \times 10^{-5}$ for \method{GradDiff};
$1.57 \times 10^{-5}$ for \method{WGA};
$1.98 \times 10^{-5}$ for \method{SatImp};
$2.58 \times 10^{-5}$ for \method{ETW}; 
$4.88 \times 10^{-5}$ for \method{RMU}; 
and $2.00 \times 10^{-5}$ for \method{ATWU}.
On \dataset{RWKU} no analogous Stage~1 optimum is available, since the
learning rate was held fixed throughout Stage~1; we instead choose
method-specific log-scale ranges based on preliminary runs at the full
Stage~2 training budget. The resulting ranges per method are:
for \method{DPO}, \method{NPO}, and \method{SimNPO},
$[1 \times 10^{-5},\, 5 \times 10^{-5}]$;
for \method{GradDiff}, \method{JensUn},
\method{WGA}, \method{SatImp}, and \method{ATWU},
$[1 \times 10^{-6},\, 1 \times 10^{-5}]$.

\begin{table}[t]
    \centering
    \small
    \captionsetup{justification=centerlast}
    \renewcommand{\arraystretch}{1.2}
    \begin{tabular}{@{}c cc cc cc@{}}
        \toprule
        \multirow{2}{*}[-0.6ex]{\textbf{Method}} & \multicolumn{2}{c}{\textbf{Shared}} & \multicolumn{2}{c}{\textbf{Method-Specific}} & \multicolumn{2}{c}{\textbf{Learning Rate}} \\
        \cmidrule(lr){2-3} \cmidrule(lr){4-5} \cmidrule(lr){6-7}
          & $\alpha$ & $\gamma$ & $\beta$ / $\beta_1$ / $\tau$ / $c$ & $\delta$ / $\beta_2$ / $\mathrm{slr}$ / $\ell$ & \texttt{1B} & \texttt{8B} \\
        \midrule
        \method{GradDiff} & $0.80$ & $0.12$  & ---    & ---    & $1.90\!\times\!10^{-5}$ & $1.90\!\times\!10^{-5}$ \\
        \method{DPO}      & $0.15$ & $3.80$  & $0.21$ & ---    & $2.27\!\times\!10^{-5}$ & $3.12\!\times\!10^{-5}$ \\
        \method{NPO}      & $4.10$ & $0.12$  & $0.10$ & ---    & $2.60\!\times\!10^{-5}$ & $2.60\!\times\!10^{-5}$ \\
        \method{SimNPO}   & $1.28$ & $1.49$  & $2.82$ & $0.03$ & $2.08\!\times\!10^{-5}$ & $2.86\!\times\!10^{-5}$ \\
        \method{JensUn}   & $0.79$ & $0.82$  & ---    & ---    & $3.98\!\times\!10^{-5}$ & $2.66\!\times\!10^{-5}$ \\
        \method{RMU}      & $1.27$ & $0.66$  & $4.39$ & $4$    & ---                     & $2.21\!\times\!10^{-4}$ \\
        \method{ETW}      & $1.31$ & $0.033$ & $0.16$ & ---    & $2.58\!\times\!10^{-5}$ & ---                     \\
        \method{WGA}      & $0.79$ & $1.16$  & $2.14$ & ---    & $1.57\!\times\!10^{-5}$ & $1.57\!\times\!10^{-5}$ \\
        \method{SatImp}   & $0.49$ & $0.87$  & $1.43$ & $0.17$ & $1.98\!\times\!10^{-5}$ & $1.98\!\times\!10^{-5}$ \\
        \midrule
        \method{ATWU}     & $0.50$ & $3.00$  & $7.00$ & $0.050$ & $2.00\!\times\!10^{-5}$ & $1.50\!\times\!10^{-5}$ \\
        \bottomrule
    \end{tabular}
    \vspace{0.5em}
    \caption{Selected hyperparameters on \dataset{TOFU}. \texttt{1B} and \texttt{8B} are the learning rates used on \model{Llama-3.2-1B-Instruct} and \model{Llama-3.1-8B-Instruct}, respectively. The method-specific columns are overloaded across methods: $\beta$ for \method{DPO}/\method{NPO}/\method{SimNPO}/\method{WGA}, $\beta_1$ for \method{SatImp}, $\tau$ (temperature) for \method{ETW}, $c$ (steering coefficient) for \method{RMU}; $\delta$ for \method{SimNPO}, $\beta_2$ for \method{SatImp}, $\mathrm{slr}$ for \method{ATWU}, $\ell$ (target layer) for \method{RMU}. \method{ETW} was evaluated only on \texttt{1B} and \method{RMU} only on \texttt{8B}; the non-evaluated column is marked ---.}
    \label{tab:hp-tofu-selected}
\end{table}

\begin{table}[t]
    \centering
    \small
    \captionsetup{justification=centering}
    \renewcommand{\arraystretch}{1.2}
    \begin{tabular}{@{}c ccc cc@{}}
        \toprule
        \multirow{2}{*}[-0.6ex]{\textbf{Method}} & \multicolumn{3}{c}{\textbf{Shared}} & \multicolumn{2}{c}{\textbf{Method-Specific}} \\
        \cmidrule(lr){2-4} \cmidrule(lr){5-6}
          & $\mathrm{lr}$ & $\alpha$ & $\gamma$ & $\beta$ / $\beta_1$ & $\delta$ / $\beta_2$ / $\mathrm{slr}$ \\
        \midrule
        \method{GradDiff} & $1.18\!\times\!10^{-6}$ & $0.80$ & $0.12$ & ---    & ---    \\
        \method{DPO}      & $1.18\!\times\!10^{-5}$ & $2.96$ & $0.13$ & $0.35$ & ---    \\
        \method{NPO}      & $2.60\!\times\!10^{-5}$ & $0.63$ & $4.76$ & $0.07$ & ---    \\
        \method{SimNPO}   & $2.08\!\times\!10^{-5}$ & $0.46$ & $2.43$ & $2.59$ & $1.50$ \\
        \method{JensUn}   & $1.54\!\times\!10^{-6}$ & $1.52$ & $0.39$ & ---    & ---    \\
        \method{WGA}      & $1.54\!\times\!10^{-6}$ & $0.22$ & $2.79$ & $0.40$ & ---    \\
        \method{SatImp}   & $4.48\!\times\!10^{-6}$ & $0.23$ & $2.60$ & $2.64$ & $0.20$ \\
        \midrule
        \method{ATWU}     & $1.83\!\times\!10^{-6}$ & $0.54$ & $2.68$ & $2.25$ & $0.020$ \\
        \bottomrule
    \end{tabular}
    \vspace{0.5em}
    \caption{Selected hyperparameters on the canonical ten-subject \dataset{RWKU} batch. The learning rate is the one used on \model{Phi-3-Mini-4k-Instruct}.}
    \label{tab:hp-rwku-selected}
\end{table}

\subsection{Final Unlearning Runs and Evaluation}\label{app:subsec:exp:unlearning}

Once a hyperparameter configuration is fixed for each
(benchmark, base model, method) triple, we run the final unlearning
procedure and evaluate the resulting checkpoints. On \dataset{TOFU} we
run the procedure on each of the three forget splits
\texttt{forget01}, \texttt{forget05}, and \texttt{forget10}; on
\dataset{RWKU} we run it on our canonical ten-subject batch.

On \dataset{TOFU}, each run uses a single
\texttt{NVIDIA A100-SXM4-80GB} GPU, trains for $10$ epochs at batch
size $8$ with $4$ gradient-accumulation steps (effective batch size
$32$), and follows the \texttt{open-unlearning} repository
defaults~\citep{dorna2026openunlearning}: the \texttt{paged\_adamw\_32bit}
optimizer with $(\beta_1, \beta_2) = (0.9, 0.999)$, weight decay
$0.01$, a one-epoch linear warmup followed by the default
\texttt{linear} decay, gradient clipping at $1.0$, and
\texttt{bfloat16} mixed precision. On \dataset{RWKU}, each run uses
two \texttt{NVIDIA A100-SXM4-80GB} GPUs with \emph{model parallelism}
only---the base model is sharded across the two GPUs while the batch is
not replicated, so the per-step batch size is $8$ and, with $4$
gradient-accumulation steps, the effective batch size is $32$---trains
for $5$ epochs, and follows the official \dataset{RWKU} repository
defaults: AdamW with $(\beta_1, \beta_2) = (0.9, 0.999)$ and weight
decay $0.01$, a \texttt{cosine} learning-rate schedule with $20$ warmup
steps, gradient clipping at $1.0$, and \texttt{fp16} precision.

For the final reported numbers, we evaluate each unlearned checkpoint on
the full metric panel of \cref{app:sec:metrics}. In the main paper, we
assess unlearning performance and retain-set preservation primarily via
the baseline-relative summary metrics \FQ, \RD, and
\UQ introduced in \cref{app:sec:metrics:derived}. These are our
preferred headline metrics because they normalize forget-side gains and
retain-side damage relative to the original checkpoint, making runs more
directly comparable across models and forget sets. In particular,
\UQ serves as our primary one-number summary, since it rewards
forgetting only to the extent that it is \emph{not} purchased by
collateral degradation on the retain set.

\begin{wrapfigure}[19]{r}{0.40\textwidth}
    \centering
    \vspace{-1.0em}
    \captionsetup{justification=centerlast}
    \includegraphics[width=\linewidth]{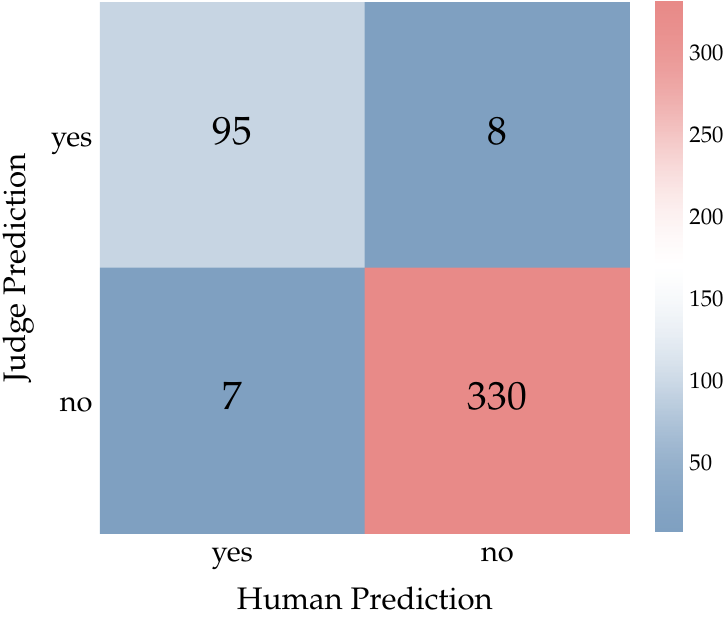}
    \caption{Agreement between the \model{GPT-5.4 mini} judge and human
    annotators on a 440-row sample of \dataset{TOFU} \texttt{forget01}.
    The judge matches the human label on $\sim$96\% of calls; errors split as 8 false positives and 7 false
    negatives.}
    \label{fig:judge-human-agreement}
    \vspace{-0.5em}
\end{wrapfigure}
To monitor broader post-unlearning behavior, we additionally report the
utility metrics of \cref{app:sec:metrics:util}: MMLU, repetitiveness,
and \WR. We interpret these as \emph{utility-preservation}
probes rather than optimization targets. The goal of unlearning is not
to improve general utility, but to preserve as much of the original
model's utility as possible while removing the targeted knowledge.
Accordingly, the relevant failure mode is utility \emph{degradation}:
drops in MMLU, drops in Rep.\ indicating more degenerate generation, and
$\WR<0.5$ against the original checkpoint. Occasional
improvements relative to the original model are not harmful and are
therefore not penalized; we interpret them as benign side effects or
measurement noise rather than as objectives of the unlearning procedure.

For completeness, the appendix also reports the underlying forget/retain
metrics from \cref{app:sec:metrics:es,app:sec:metrics:jx}, namely the
extraction-strength scores
\ESF, \ESR,
\ESD and the judge-based scores
$\mathcal{J}_{\mathrm{P}}$, $\mathcal{J}_{\mathrm{ICR}}$,
$\mathcal{J}_{\mathrm{W}}$, and $\mathcal{J}_{\mathrm{AVG}}$.
Judge-based evaluation uses the fixed prompts provided in the
supplementary material and is performed with OpenAI's
\model{GPT-5.4 mini} judge model, instantiated as
\texttt{gpt-5.4-mini}~\citep{openai2026gpt54mini}. We validate this
judge against human annotations in \cref{fig:judge-human-agreement}.
MMLU is computed with the
\texttt{lm-evaluation-harness}~\citep{eval-harness}. Main-paper tables
report \FQ, \RD, \UQ, MMLU, Rep., and
\WR, while the full metric panel is deferred to this section.

\begin{table}[!ht]
    \centering
    \small
    \captionsetup{justification=centering}
    \renewcommand{\arraystretch}{0.85}
    \setlength{\tabcolsep}{4pt}
    \resizebox{\textwidth}{!}{%
    \begin{tabular}{@{}c ccc cc ccc ccc@{}}
        \toprule
        \multicolumn{12}{c}{\texttt{forget01} - \texttt{Llama-3.2-1B-Instruct}} \\
        \midrule
        \multirow{2}{*}[-0.6ex]{\textbf{Method}} & \multicolumn{3}{c}{\textbf{Surrogate}} & \multicolumn{2}{c}{\textbf{Paraphrase}} & \multicolumn{3}{c}{\textbf{Relative}} & \multicolumn{3}{c}{\textbf{Utility}} \\
        \cmidrule(lr){2-4} \cmidrule(lr){5-6} \cmidrule(lr){7-9} \cmidrule(lr){10-12}
          & $\ESF\!\downarrow$ & $\ESR\!\uparrow$ & $\ESD\!\uparrow$ & $\mathcal{J}_\mathrm{W}\!\downarrow$ & $\mathcal{J}_\mathrm{AVG}\!\uparrow$ & $\FQ\,\uparrow$ & $\RD\,\downarrow$ & $\UQ\,\uparrow$ & MMLU & Rep. & \WR \\
        \midrule
        \method{Original} & $0.743$ & $0.737$ & $-0.007$ & $90.0$ & $41.5$ & $0.0$ & $0.0$ & $0.0$ & $45.1$ & $559$ & $50.0$ \\
        \midrule
        \method{GradDiff} & $0.220$ & $\secondbest{0.625}$ & $0.405$ & $77.5$ & $35.9$ & $13.9$ & $13.6$ & $0.3$ & $45.0$ & $563$ & $51.5$ \\
        \method{DPO} & $0.134$ & $0.566$ & $0.432$ & $\best{17.5}$ & $38.1$ & $\best{80.6}$ & $8.3$ & $\best{72.2}$ & $45.1$ & $552$ & $56.5$ \\
        \method{NPO} & $0.085$ & $0.539$ & $0.454$ & $\secondbest{45.0}$ & $37.7$ & $\secondbest{50.0}$ & $9.3$ & $\secondbest{40.7}$ & $45.3$ & $562$ & $59.0$ \\
        \method{SimNPO} & $0.087$ & $0.574$ & $0.487$ & $52.5$ & $\secondbest{38.3}$ & $41.7$ & $\secondbest{7.9}$ & $33.7$ & $45.4$ & $565$ & $54.0$ \\
        \method{JensUn} & $0.295$ & $0.600$ & $0.304$ & $60.0$ & $29.2$ & $33.3$ & $29.8$ & $3.5$ & $44.9$ & $555$ & $55.0$ \\
        \method{WGA} & $0.093$ & $\best{0.675}$ & $\best{0.582}$ & $80.0$ & $34.6$ & $11.1$ & $16.7$ & $0.0$ & $45.2$ & $560$ & $58.5$ \\
        \method{SatImp} & $\secondbest{0.084}$ & $0.619$ & $0.536$ & $55.0$ & $\best{38.3}$ & $38.9$ & $\best{7.8}$ & $31.1$ & $45.1$ & $560$ & $56.0$ \\
        \midrule
        \method{ATWU} & $\best{0.046}$ & $0.604$ & $\secondbest{0.558}$ & $52.5$ & $36.5$ & $41.7$ & $12.1$ & $29.5$ & $45.1$ & $557$ & $52.0$ \\
        \midrule
        \multicolumn{12}{c}{\texttt{forget05} - \texttt{Llama-3.2-1B-Instruct}} \\
        \midrule
        \multirow{2}{*}[-0.6ex]{\textbf{Method}} & \multicolumn{3}{c}{\textbf{Surrogate}} & \multicolumn{2}{c}{\textbf{Paraphrase}} & \multicolumn{3}{c}{\textbf{Relative}} & \multicolumn{3}{c}{\textbf{Utility}} \\
        \cmidrule(lr){2-4} \cmidrule(lr){5-6} \cmidrule(lr){7-9} \cmidrule(lr){10-12}
          & $\ESF\!\downarrow$ & $\ESR\!\uparrow$ & $\ESD\!\uparrow$ & $\mathcal{J}_\mathrm{W}\!\downarrow$ & $\mathcal{J}_\mathrm{AVG}\!\uparrow$ & $\FQ\,\uparrow$ & $\RD\,\downarrow$ & $\UQ\,\uparrow$ & MMLU & Rep. & \WR \\
        \midrule
        \method{Original} & $0.727$ & $0.737$ & $0.009$ & $90.5$ & $42.9$ & $0.0$ & $0.0$ & $0.0$ & $45.1$ & $559$ & $50.0$ \\
        \midrule
        \method{GradDiff} & $0.178$ & $0.620$ & $0.442$ & $68.6$ & $36.1$ & $24.2$ & $15.8$ & $8.4$ & $44.8$ & $563$ & $50.5$ \\
        \method{DPO} & $0.205$ & $0.624$ & $0.419$ & $50.5$ & $38.6$ & $44.2$ & $10.0$ & $34.2$ & $44.8$ & $551$ & $45.0$ \\
        \method{NPO} & $0.100$ & $0.644$ & $0.543$ & $47.5$ & $38.4$ & $47.5$ & $10.6$ & $\secondbest{36.9}$ & $45.0$ & $565$ & $53.5$ \\
        \method{SimNPO} & $0.121$ & $0.689$ & $0.568$ & $54.6$ & $\secondbest{38.8}$ & $39.7$ & $\secondbest{9.6}$ & $30.1$ & $45.1$ & $568$ & $54.5$ \\
        \method{JensUn} & $\secondbest{0.047}$ & $0.208$ & $0.161$ & $\best{0.0}$ & $1.1$ & $\best{100.0}$ & $97.3$ & $2.7$ & $45.4$ & \degraded{$369$} & \degraded{$25.0$} \\
        \method{WGA} & $0.076$ & $\best{0.746}$ & $\best{0.670}$ & $53.0$ & $\best{40.2}$ & $41.5$ & $\best{6.4}$ & $35.1$ & $45.0$ & $568$ & $49.0$ \\
        \method{SatImp} & $0.061$ & $0.690$ & $0.629$ & $47.9$ & $38.3$ & $47.1$ & $10.9$ & $36.3$ & $44.9$ & $569$ & $57.0$ \\
        \midrule
        \method{ATWU} & $\best{0.041}$ & $\secondbest{0.696}$ & $\secondbest{0.655}$ & $\secondbest{37.5}$ & $37.1$ & $\secondbest{58.6}$ & $13.6$ & $\best{45.0}$ & $45.1$ & $574$ & $60.5$ \\
        \midrule
        \multicolumn{12}{c}{\texttt{forget10} - \texttt{Llama-3.2-1B-Instruct}} \\
        \midrule
        \multirow{2}{*}[-0.6ex]{\textbf{Method}} & \multicolumn{3}{c}{\textbf{Surrogate}} & \multicolumn{2}{c}{\textbf{Paraphrase}} & \multicolumn{3}{c}{\textbf{Relative}} & \multicolumn{3}{c}{\textbf{Utility}} \\
        \cmidrule(lr){2-4} \cmidrule(lr){5-6} \cmidrule(lr){7-9} \cmidrule(lr){10-12}
          & $\ESF\!\downarrow$ & $\ESR\!\uparrow$ & $\ESD\!\uparrow$ & $\mathcal{J}_\mathrm{W}\!\downarrow$ & $\mathcal{J}_\mathrm{AVG}\!\uparrow$ & $\FQ\,\uparrow$ & $\RD\,\downarrow$ & $\UQ\,\uparrow$ & MMLU & Rep. & \WR \\
        \midrule
        \method{Original} & $0.706$ & $0.737$ & $0.030$ & $93.1$ & $43.0$ & $0.0$ & $0.0$ & $0.0$ & $45.1$ & $559$ & $50.0$ \\
        \midrule
        \method{GradDiff} & $0.107$ & $0.658$ & $0.550$ & $57.9$ & $37.4$ & $37.8$ & $13.0$ & $24.8$ & $45.2$ & $563$ & $55.5$ \\
        \method{DPO} & $0.194$ & $0.716$ & $0.523$ & $53.3$ & $38.2$ & $42.8$ & $11.2$ & $31.6$ & $44.9$ & $560$ & $59.0$ \\
        \method{NPO} & $0.103$ & $0.730$ & $0.627$ & $41.0$ & $\best{42.4}$ & $56.0$ & $\best{1.4}$ & $54.6$ & $44.9$ & $563$ & $55.0$ \\
        \method{SimNPO} & $0.106$ & $0.738$ & $0.632$ & $57.4$ & $41.1$ & $38.3$ & $4.4$ & $33.9$ & $45.1$ & $566$ & $57.0$ \\
        \method{JensUn} & $\best{0.033}$ & $0.570$ & $0.537$ & $\best{1.8}$ & $36.4$ & $\best{98.0}$ & $15.3$ & $\best{82.8}$ & $45.1$ & \degraded{$343$} & \degraded{$23.5$} \\
        \method{ETW} & $0.108$ & $0.679$ & $0.572$ & $47.8$ & $39.0$ & $48.7$ & $9.3$ & $39.4$ & $45.5$ & $561$ & $54.5$ \\
        \method{WGA} & $0.062$ & $\best{0.759}$ & $0.696$ & $46.6$ & $39.3$ & $50.0$ & $8.5$ & $41.5$ & $45.0$ & $568$ & $56.0$ \\
        \method{SatImp} & $0.043$ & $0.747$ & $\secondbest{0.704}$ & $32.9$ & $\secondbest{41.2}$ & $64.7$ & $\secondbest{4.1}$ & $60.7$ & $44.8$ & $570$ & $55.0$ \\
        \midrule
        \method{ATWU} & $\secondbest{0.035}$ & $\secondbest{0.753}$ & $\best{0.717}$ & $\secondbest{14.5}$ & $40.3$ & $\secondbest{84.4}$ & $6.3$ & $\secondbest{78.1}$ & $45.0$ & $575$ & $51.5$ \\
        \bottomrule
    \end{tabular}%
    }
    \vspace{0.5em}
    \caption{\dataset{TOFU} full-panel results for \model{Llama-3.2-1B-Instruct} across the three forget splits. In each performance column, best is in bold and second-best underlined. \method{ETW} was evaluated only on \texttt{forget10}.}
    \label{tab:tofu-full-1b}
\end{table}

\paragraph{Judge reproducibility.}
A practical caveat is that OpenAI's hosted inference is not bit-exact reproducible: calling \texttt{gpt-5.4-mini} on the same input does not in general return the same completion or the same per-token log-probabilities, even at \texttt{temperature=0} and with a fixed \texttt{seed}.\footnote{See e.g.\ the OpenAI developer-community threads at:\\
\hspace*{1.8em}\url{https://community.openai.com/t/reproducible-outputs-in-assistants-api/1167782} and \\
\hspace*{1.8em}\url{https://community.openai.com/t/logprobs-inconsistent-between-runs-for-4o/935082/2}} This is a known property of the served models rather than of the API surface, and the standard mitigation in the community is to cache judged responses against a hash of the input rather than to expect run-to-run determinism. Because rejudging every checkpoint on every paraphrase carries non-trivial API cost (cf.\ the \$180 figure quoted earlier in this section), we run the judge once per (checkpoint, query, paraphrase) tuple and report the resulting scores; the 96\% judge--human agreement reported in \cref{fig:judge-human-agreement} is computed on the same single-pass scores. In spot rejudging we observed only minor drift in $\mathcal{J}$-based scores across calls---small relative to the inter-method gaps in our headline tables---which we take as evidence that the rankings we report are robust to judge non-determinism.

\begin{table}[!t]
    \centering
    \small
    \captionsetup{justification=centering}
    \renewcommand{\arraystretch}{0.85}
    \setlength{\tabcolsep}{4pt}
    \resizebox{\textwidth}{!}{%
    \begin{tabular}{@{}c ccc cc ccc ccc@{}}
        \toprule
        \multicolumn{12}{c}{\texttt{forget01} - \texttt{Llama-3.1-8B-Instruct}} \\
        \midrule
        \multirow{2}{*}[-0.6ex]{\textbf{Method}} & \multicolumn{3}{c}{\textbf{Surrogate}} & \multicolumn{2}{c}{\textbf{Paraphrase}} & \multicolumn{3}{c}{\textbf{Relative}} & \multicolumn{3}{c}{\textbf{Utility}} \\
        \cmidrule(lr){2-4} \cmidrule(lr){5-6} \cmidrule(lr){7-9} \cmidrule(lr){10-12}
          & $\ESF\!\downarrow$ & $\ESR\!\uparrow$ & $\ESD\!\uparrow$ & $\mathcal{J}_\mathrm{W}\!\downarrow$ & $\mathcal{J}_\mathrm{AVG}\!\uparrow$ & $\FQ\,\uparrow$ & $\RD\,\downarrow$ & $\UQ\,\uparrow$ & MMLU & Rep. & \WR \\
        \midrule
        \method{Original} & $0.977$ & $0.992$ & $0.015$ & $100.0$ & $68.2$ & $0.0$ & $0.0$ & $0.0$ & $66.6$ & $519$ & $50.0$ \\
        \midrule
        \method{GradDiff} & $0.203$ & $0.724$ & $0.521$ & $72.5$ & $50.9$ & $27.5$ & $25.3$ & $2.2$ & $66.6$ & $529$ & $45.0$ \\
        \method{DPO} & $0.219$ & $0.805$ & $0.585$ & $\best{12.5}$ & $56.6$ & $\best{87.5}$ & $16.9$ & $\best{70.6}$ & $66.3$ & \degraded{$482$} & \degraded{$36.0$} \\
        \method{NPO} & $0.093$ & $0.700$ & $0.607$ & $32.5$ & $53.3$ & $67.5$ & $21.8$ & $45.7$ & $66.3$ & $544$ & $48.5$ \\
        \method{SimNPO} & $\secondbest{0.080}$ & $0.666$ & $0.586$ & $\secondbest{27.5}$ & $51.4$ & $\secondbest{72.5}$ & $24.6$ & $47.9$ & $66.4$ & $541$ & $50.0$ \\
        \method{JensUn} & $0.887$ & $\best{0.979}$ & $0.092$ & $95.0$ & $57.6$ & $5.0$ & $15.6$ & $0.0$ & $66.5$ & \degraded{$278$} & \degraded{$19.5$} \\
        \method{WGA} & $0.080$ & $\secondbest{0.907}$ & $\best{0.827}$ & $42.5$ & $\secondbest{63.0}$ & $57.5$ & $\secondbest{7.5}$ & $50.0$ & $66.5$ & $534$ & $48.5$ \\
        \method{SatImp} & $0.083$ & $0.809$ & $0.726$ & $42.5$ & $56.5$ & $57.5$ & $17.1$ & $40.4$ & $66.6$ & $525$ & \degraded{$41.5$} \\
        \midrule
        \method{ATWU} & $\best{0.065}$ & $0.885$ & $\secondbest{0.820}$ & $45.0$ & $\best{64.8}$ & $55.0$ & $\best{5.0}$ & $\secondbest{50.0}$ & $66.7$ & $558$ & $57.0$ \\
        \midrule
        \multicolumn{12}{c}{\texttt{forget05} - \texttt{Llama-3.1-8B-Instruct}} \\
        \midrule
        \multirow{2}{*}[-0.6ex]{\textbf{Method}} & \multicolumn{3}{c}{\textbf{Surrogate}} & \multicolumn{2}{c}{\textbf{Paraphrase}} & \multicolumn{3}{c}{\textbf{Relative}} & \multicolumn{3}{c}{\textbf{Utility}} \\
        \cmidrule(lr){2-4} \cmidrule(lr){5-6} \cmidrule(lr){7-9} \cmidrule(lr){10-12}
          & $\ESF\!\downarrow$ & $\ESR\!\uparrow$ & $\ESD\!\uparrow$ & $\mathcal{J}_\mathrm{W}\!\downarrow$ & $\mathcal{J}_\mathrm{AVG}\!\uparrow$ & $\FQ\,\uparrow$ & $\RD\,\downarrow$ & $\UQ\,\uparrow$ & MMLU & Rep. & \WR \\
        \midrule
        \method{Original} & $0.972$ & $0.992$ & $0.020$ & $99.0$ & $68.6$ & $0.0$ & $0.0$ & $0.0$ & $66.6$ & $546$ & $50.0$ \\
        \midrule
        \method{GradDiff} & $0.256$ & $0.782$ & $0.525$ & $75.0$ & $51.6$ & $24.2$ & $24.7$ & $0.0$ & $66.7$ & $544$ & $45.5$ \\
        \method{DPO} & $0.268$ & $0.815$ & $0.548$ & $41.5$ & $58.4$ & $58.1$ & $14.8$ & $43.3$ & $65.1$ & $530$ & \degraded{$34.0$} \\
        \method{NPO} & $0.123$ & $0.875$ & $0.752$ & $50.0$ & $60.6$ & $49.5$ & $11.7$ & $37.8$ & $66.3$ & $549$ & $53.0$ \\
        \method{SimNPO} & $0.102$ & $0.852$ & $0.750$ & $42.5$ & $55.7$ & $57.1$ & $18.7$ & $38.4$ & $66.0$ & $545$ & $56.5$ \\
        \method{JensUn} & $\best{0.033}$ & $\secondbest{0.924}$ & $\best{0.891}$ & $\best{1.0}$ & $61.2$ & $\best{99.0}$ & $10.7$ & $\best{88.3}$ & $65.2$ & \degraded{$474$} & \degraded{$29.0$} \\
        \method{WGA} & $0.058$ & $\best{0.931}$ & $0.874$ & $24.5$ & $\best{63.6}$ & $75.3$ & $\best{7.2}$ & $68.0$ & $66.6$ & $541$ & $45.5$ \\
        \method{SatImp} & $0.035$ & $0.891$ & $0.855$ & $15.5$ & $61.0$ & $84.3$ & $11.0$ & $73.4$ & $66.6$ & $546$ & $56.0$ \\
        \midrule
        \method{ATWU} & $\secondbest{0.033}$ & $0.914$ & $\secondbest{0.882}$ & $\secondbest{11.7}$ & $\secondbest{61.7}$ & $\secondbest{88.2}$ & $\secondbest{10.0}$ & $\secondbest{78.2}$ & $66.8$ & $541$ & $48.5$ \\
        \midrule
        \multicolumn{12}{c}{\texttt{forget10} - \texttt{Llama-3.1-8B-Instruct}} \\
        \midrule
        \multirow{2}{*}[-0.6ex]{\textbf{Method}} & \multicolumn{3}{c}{\textbf{Surrogate}} & \multicolumn{2}{c}{\textbf{Paraphrase}} & \multicolumn{3}{c}{\textbf{Relative}} & \multicolumn{3}{c}{\textbf{Utility}} \\
        \cmidrule(lr){2-4} \cmidrule(lr){5-6} \cmidrule(lr){7-9} \cmidrule(lr){10-12}
          & $\ESF\!\downarrow$ & $\ESR\!\uparrow$ & $\ESD\!\uparrow$ & $\mathcal{J}_\mathrm{W}\!\downarrow$ & $\mathcal{J}_\mathrm{AVG}\!\uparrow$ & $\FQ\,\uparrow$ & $\RD\,\downarrow$ & $\UQ\,\uparrow$ & MMLU & Rep. & \WR \\
        \midrule
        \method{Original} & $0.979$ & $0.992$ & $0.013$ & $99.5$ & $67.8$ & $0.0$ & $0.0$ & $0.0$ & $66.6$ & $546$ & $50.0$ \\
        \midrule
        \method{GradDiff} & $0.096$ & $0.838$ & $0.742$ & $60.0$ & $56.1$ & $39.7$ & $17.3$ & $22.4$ & $66.1$ & $545$ & $48.0$ \\
        \method{DPO} & $0.314$ & $0.860$ & $0.547$ & $40.2$ & $57.6$ & $59.5$ & $15.0$ & $44.5$ & \degraded{$64.3$} & $533$ & \degraded{$38.5$} \\
        \method{NPO} & $0.129$ & $0.927$ & $0.798$ & $51.7$ & $64.4$ & $48.0$ & $5.0$ & $43.0$ & $66.2$ & $551$ & $48.5$ \\
        \method{SimNPO} & $0.077$ & $0.905$ & $0.828$ & $30.5$ & $63.4$ & $69.3$ & $6.5$ & $62.8$ & $65.8$ & $541$ & $50.5$ \\
        \method{JensUn} & $\best{0.033}$ & $0.942$ & $0.910$ & $\best{2.0}$ & $61.3$ & $\best{98.0}$ & $9.6$ & $\secondbest{88.3}$ & $65.3$ & \degraded{$236$} & \degraded{$11.0$} \\
        \method{WGA} & $0.046$ & $0.970$ & $0.924$ & $33.0$ & $64.6$ & $66.8$ & $4.8$ & $62.1$ & $66.0$ & $550$ & $60.0$ \\
        \method{SatImp} & $0.035$ & $0.952$ & $0.917$ & $21.2$ & $\secondbest{65.8}$ & $78.6$ & $\secondbest{3.0}$ & $75.6$ & $65.9$ & $549$ & $53.5$ \\
        \method{RMU} & $0.033$ & $\best{0.989}$ & $\best{0.956}$ & $12.8$ & $\best{66.9}$ & $87.2$ & $\best{1.3}$ & $85.9$ & $65.3$ & $540$ & $50.5$ \\
        \midrule
        \method{ATWU} & $0.033$ & $\secondbest{0.973}$ & $\secondbest{0.940}$ & $\secondbest{4.8}$ & $65.4$ & $\secondbest{95.2}$ & $3.5$ & $\best{91.7}$ & $66.5$ & $558$ & $58.0$ \\
        \bottomrule
    \end{tabular}%
    }
    \vspace{0.5em}
    \caption{\dataset{TOFU} full-panel results for \model{Llama-3.1-8B-Instruct} across the three forget splits. In each performance column, best is in bold and second-best underlined. \method{RMU} was evaluated only on \texttt{forget10}.}
    \label{tab:tofu-full-8b}
\end{table}

\paragraph{The importance of proper tuning.}
\label{app:tuning-importance}
Reported baseline performance in the LLM-unlearning literature is unusually sensitive to per-method hyperparameter choices, and numbers across papers are not always directly comparable. \method{RMU}~\citep{li2024the} is the clearest example: \citet{wang2025rethinking} report it underperforming several preference-based baselines on \dataset{TOFU}, yet under our uniform two-stage tuning protocol it reaches $\UQ=85.9$ on \texttt{forget10} with \model{Llama-3.1-8B-Instruct}, second only to \method{ATWU} ($\UQ=91.7$) and well ahead of every preference-based method (\cref{tab:exp:headline}). The driver is \method{RMU}'s idiosyncratic configuration space, where a mis-specified unlearning layer $\ell$ alone can significantly swing \UQ; our Bayesian search lands on $\ell=4$, $c=4.39$ (\cref{tab:hp-tofu-selected}), which differs substantially from the configurations typically reported for WMDP-style settings. The relevant tuning gap is not the loss coefficients $\alpha,\gamma$ alone: \citet{wang2025rethinking} and \citet{yang2025exploring} do tune those in narrow paper-seeded ranges, but the learning rate (which is not equivalent to a coefficient rescaling under \texttt{AdamW} with gradient clipping and a non-constant LR schedule) and, in \method{RMU}'s case, the steering coefficient and target layer, which are not loss coefficients at all. By contrast, \citet{singh2025unlearninglasts} and \citet{dorna2026openunlearning} grid-search the learning rate per method and report \method{RMU} as competitive, consistent with our finding. We view this not as a critique of any specific prior comparison but as empirical evidence that fair ranking across this literature requires re-tuning each method for the target setting, particularly when methods exposing many interacting knobs (\method{RMU}, \method{SimNPO}, \method{SatImp}) are compared against methods with very few (\method{GradDiff}, \method{JensUn}). We therefore tune every baseline uniformly via Bayesian optimization, seeded with each paper's recommended defaults; the full protocol and selected configurations are reported in \cref{app:subsec:exp:tuning,app:subsec:exp:hp-ranges}.

\begin{table}[!t]
    \centering
    \small
    \captionsetup{justification=centering}
    \renewcommand{\arraystretch}{0.85}
    \setlength{\tabcolsep}{4pt}
    \resizebox{\textwidth}{!}{%
    \begin{tabular}{@{}c ccc cc ccc ccc@{}}
        \toprule
        \multirow{2}{*}[-0.6ex]{\textbf{Method}} & \multicolumn{3}{c}{\textbf{Surrogate}} & \multicolumn{2}{c}{\textbf{Paraphrase}} & \multicolumn{3}{c}{\textbf{Relative}} & \multicolumn{3}{c}{\textbf{Utility}} \\
        \cmidrule(lr){2-4} \cmidrule(lr){5-6} \cmidrule(lr){7-9} \cmidrule(lr){10-12}
          & $\mathrm{R\text{-}L}_{\mathcal{F}}\!\downarrow$ & $\mathrm{R\text{-}L}_{\mathcal{N}}\!\uparrow$ & $N_{\Delta}\!\uparrow$ & $\mathcal{J}_\mathrm{W}\!\downarrow$ & $\mathcal{J}_\mathrm{AVG}\!\uparrow$ & $\FQ\,\uparrow$ & $\RD\,\downarrow$ & $\UQ\,\uparrow$ & MMLU & Rep. & \WR \\
        \midrule
        \method{Original} & $0.683$ & $0.759$ & $0.077$ & $84.8$ & $65.8$ & $0.0$ & $0.0$ & $0.0$ & $70.1$ & $529$ & $50.0$ \\
        \midrule
        \method{GradDiff} & $\secondbest{0.053}$ & $0.445$ & $0.392$ & $\secondbest{15.3}$ & $41.6$ & $\secondbest{81.9}$ & $36.7$ & $45.2$ & $69.2$ & $514$ & \degraded{$43.0$} \\
        \method{DPO} & $0.499$ & $\best{0.590}$ & $0.090$ & $74.4$ & $\best{54.6}$ & $12.3$ & $\best{17.0}$ & $0.0$ & $69.7$ & $523$ & \degraded{$43.0$} \\
        \method{NPO} & $0.236$ & $0.541$ & $0.305$ & $44.4$ & $40.1$ & $47.6$ & $39.1$ & $8.5$ & $68.1$ & $543$ & \degraded{$33.5$} \\
        \method{SimNPO} & $0.218$ & $0.561$ & $0.343$ & $40.9$ & $46.2$ & $51.8$ & $29.7$ & $22.1$ & $68.8$ & $529$ & $45.5$ \\
        \method{JensUn} & $0.062$ & $0.434$ & $0.372$ & $\best{12.4}$ & $42.0$ & $\best{85.4}$ & $36.1$ & $49.3$ & $70.2$ & $523$ & $46.5$ \\
        \method{WGA} & $0.061$ & $0.515$ & $0.454$ & $18.7$ & $48.6$ & $78.0$ & $26.1$ & $\secondbest{51.9}$ & $69.4$ & $518$ & \degraded{$41.0$} \\
        \method{SatImp} & $0.069$ & $0.533$ & $\secondbest{0.464}$ & $19.3$ & $46.5$ & $77.3$ & $29.3$ & $48.0$ & $69.8$ & $518$ & \degraded{$40.0$} \\
        \midrule
        \method{ATWU} & $\best{0.050}$ & $\secondbest{0.562}$ & $\best{0.512}$ & $15.8$ & $\secondbest{50.9}$ & $81.4$ & $\secondbest{22.7}$ & $\best{58.7}$ & $70.3$ & $519$ & $46.0$ \\
        \bottomrule
    \end{tabular}%
    }
    \vspace{0.5em}
    \caption{\dataset{RWKU} canonical ten-subject batch, \model{Phi-3-Mini-4k-Instruct}. In each performance column, best is in bold and second-best underlined.}
    \label{tab:rwku-full}
\end{table}

\begin{table}[!b]
    \centering
    \small
    \captionsetup{justification=centering}
    \renewcommand{\arraystretch}{0.85}
    \setlength{\tabcolsep}{4pt}
    \resizebox{\textwidth}{!}{%
    \begin{tabular}{@{}c ccc cc cccc ccc@{}}
        \toprule
        \multirow{2}{*}[-0.6ex]{\textbf{Method}} & \multicolumn{3}{c}{\textbf{Surrogate}} & \multicolumn{2}{c}{\textbf{Paraphrase}} & \multicolumn{4}{c}{\textbf{Relative}} & \multicolumn{3}{c}{\textbf{Utility}} \\
        \cmidrule(lr){2-4} \cmidrule(lr){5-6} \cmidrule(lr){7-10} \cmidrule(lr){11-13}
          & $\mathrm{R\text{-}L}_{\mathcal{F}}\!\downarrow$ & $\mathrm{R\text{-}L}_{\mathcal{N}}\!\uparrow$ & $N_{\Delta}\!\uparrow$ & $\mathcal{J}_\mathrm{W}\!\downarrow$ & $\mathcal{J}_\mathrm{AVG}\!\uparrow$ & $\FQ\,\uparrow$ & $\RD\,\downarrow$ & $\UQ\,\uparrow$ & $\Delta$$\UQ\,\uparrow$ & MMLU & Rep. & \WR \\
        \midrule
        \method{Original} & $0.683$ & $0.759$ & $0.077$ & $84.8$ & $65.8$ & $0.0$ & $0.0$ & $0.0$ & --- & $70.1$ & $529$ & $50.0$ \\
        \midrule
        \method{ATWU}$_\mathrm{DPO}$    & $0.501$              & $\best{0.657}$       & $0.156$              & $68.2$              & $\best{57.1}$        & $19.6$              & $\best{13.2}$       & $6.4$               & $+6.4$  & $69.5$              & $512$ & \degraded{$36.0$} \\
        \method{ATWU}$_\mathrm{NPO}$    & $0.205$              & $0.490$              & $0.285$              & $39.1$              & $42.8$               & $53.8$              & $34.9$              & $18.9$              & $+10.4$ & \degraded{$67.7$}   & $545$ & \degraded{$35.5$} \\
        \method{ATWU}$_\mathrm{SimNPO}$ & $\secondbest{0.109}$ & $\secondbest{0.567}$ & $\secondbest{0.458}$ & $\secondbest{32.6}$ & $50.1$               & $\secondbest{61.5}$ & $23.8$              & $\secondbest{37.8}$ & $+15.7$ & $69.9$              & $531$ & \degraded{$43.5$} \\
        \midrule
        \method{ATWU}                    & $\best{0.050}$       & $0.562$              & $\best{0.512}$       & $\best{15.8}$       & $\secondbest{50.9}$  & $\best{81.4}$       & $\secondbest{22.7}$ & $\best{58.7}$       & $+6.8$ & $70.3$              & $519$ & $46.0$ \\
        \bottomrule
    \end{tabular}%
    }
    \vspace{0.5em}
    \caption{
    \dataset{RWKU} canonical ten-subject batch with \method{ATWU} instantiated using different forget losses, \model{Phi-3-Mini-4k-Instruct}. $\Delta$\UQ reports the gain over the corresponding unweighted objective in \cref{tab:rwku-full}. In each performance column best is in bold and second-best underlined.
    }
    \label{tab:rwku-full-sb}
\end{table}

\paragraph{\dataset{TOFU} full-panel results.}
\Cref{tab:tofu-full-1b,tab:tofu-full-8b} report the full
\dataset{TOFU} metric panels across all forget splits and both base
models. Across settings, \method{ATWU} is consistently among the best
or second-best methods once utility-collapsing runs are excluded. This
distinction is important: aggressive methods such as \method{JensUn}
can obtain very strong forget-side scores, but often do so by severely
damaging generation quality, as reflected by large drops in Rep.\ and
\WR. We therefore interpret the tables through the joint lens of
forgetting, retain preservation, and general utility.

The main limitation appears on the smaller forget splits, especially
\texttt{forget01}. In these settings, the scorer receives relatively
few forget examples and fewer effective scorer updates, making it
harder to learn a stable token-level relevance signal. As the forget
set becomes larger, the learned scorer has more opportunity to
separate forget-specific tokens from structural context, and the
advantage of \method{ATWU} becomes clearer. This trend is most visible
on \texttt{forget10}, where \method{ATWU} achieves the strongest or
near-strongest \UQ while preserving MMLU, Rep., and \WR close
to the original checkpoint. These results suggest that the learned
token-weighting mechanism is most beneficial when enough forget-side
signal is available for the scorer to train reliably.

\paragraph{\dataset{RWKU} full-panel results.}
\Cref{tab:rwku-full} reports the corresponding full metric panel on
the canonical ten-subject \dataset{RWKU} batch. This setting is
challenging for a different reason than \dataset{TOFU}. In
\dataset{RWKU}, the forget set consists of paragraphs about the target
entities, whereas evaluation is performed through separate queries
about those entities. Thus, the unlearning objective is not optimized
on the same question--answer format used at evaluation time; the model
must forget information expressed in paragraph form and generalize that
forgetting to query-based probes.
Despite this mismatch, \method{ATWU} achieves the highest
\UQ, improving over the strongest non-\method{ATWU} baseline
while preserving MMLU and maintaining a non-collapsed win rate against
the original model. The ROUGE-L surrogate metrics show a similar
trend, with \method{ATWU} attaining the best \NDelta, but the
judge-based metrics provide the more important comparison: methods with
strong forget-side scores can still suffer substantial retain-side or
utility degradation. Overall, the \dataset{RWKU} results indicate that
learned token-level reweighting remains effective even when the forget
data and evaluation queries differ in format.

\paragraph{\method{ATWU} with different forget losses.}
Finally, \cref{tab:rwku-full-sb} evaluates whether \method{ATWU} is tied to a particular forget loss. In the main \method{ATWU} variant, we use a GradDiff-style forget loss. Here, we compare with alternative forget loss choices derived from \method{DPO}, \method{NPO}, and \method{SimNPO}, integrating the learned scorer into the forget-side term of each objective as described in \cref{app:subsec:additional:augment}. This keeps the central mechanism of \method{ATWU} fixed -- learning token-level forget-relevance scores and using them to weight the forget update -- while only changing the underlying forget loss. Across all forget losses, \method{ATWU} improves \UQ relative to the corresponding unweighted objective. The strongest results are obtained with the GradDiff-style instantiation, but the gains for \method{DPO}, \method{NPO}, and \method{SimNPO} show that the learned scorer is not specific to one loss form. Rather, \method{ATWU} provides a general way to turn sequence-level forget objectives into selective token-weighted objectives.

\clearpage
\newpage
\section{Additional Experiments}\label{app:sec:additional-experiments}

This section provides additional experimental details and ablations that support the main findings. We first ask whether a single linear function of the model's hidden states can even, in principle, separate forget-relevant tokens from structural ones, with and without supervision. We then show how \method{ATWU} can be combined with alternative forget losses by inserting the learned scorer only into the forget-side token-level terms. Finally, we ablate the main scorer design choices, including its regularization, training state, and update frequency. Together, these experiments test whether the gains of \method{ATWU} come from a particular objective choice, or from the more general mechanism of learning where the forget update should be applied.

\subsection{Linear Separability of Informative Tokens}
\label{app:subsec:additional:linear-separability}

Before validating our parametric scorer in detail, we briefly ask whether a linear function of the model's hidden states can even, in principle, separate forget-relevant tokens from structural ones, and whether such a scorer can be learned without supervision. The two experiments below answer affirmatively in both regimes.

\paragraph{Supervised baseline.}
We train a single-layer linear scorer on the final-layer hidden states of \model{Llama-3.1-8B-Instruct} fine-tuned on \dataset{TOFU} using binary cross-entropy against the ground-truth token-level forget labels of \citet{zhou2026not}, for one epoch on \texttt{forget10}. \Cref{fig:lin-sep:supervised} reports the per-sample AUROC distribution as training progresses, together with the final scorer's output on a representative forget sample. After only $\sim\!20\%$ of one epoch the AUROC distribution is already concentrated at the high end, and the final scorer cleanly recovers the GT forget tokens on the qualitative example. A single linear direction on the model's hidden states therefore suffices to separate forget-relevant tokens, at least \emph{when the labels are available}.

\begin{figure}[!ht]
    \centering
    \captionsetup{justification=centerlast}
    \small
    \setlength{\fboxsep}{1.5pt}
    \sbox{\linsepridgebox}{\includegraphics[width=0.55\textwidth]{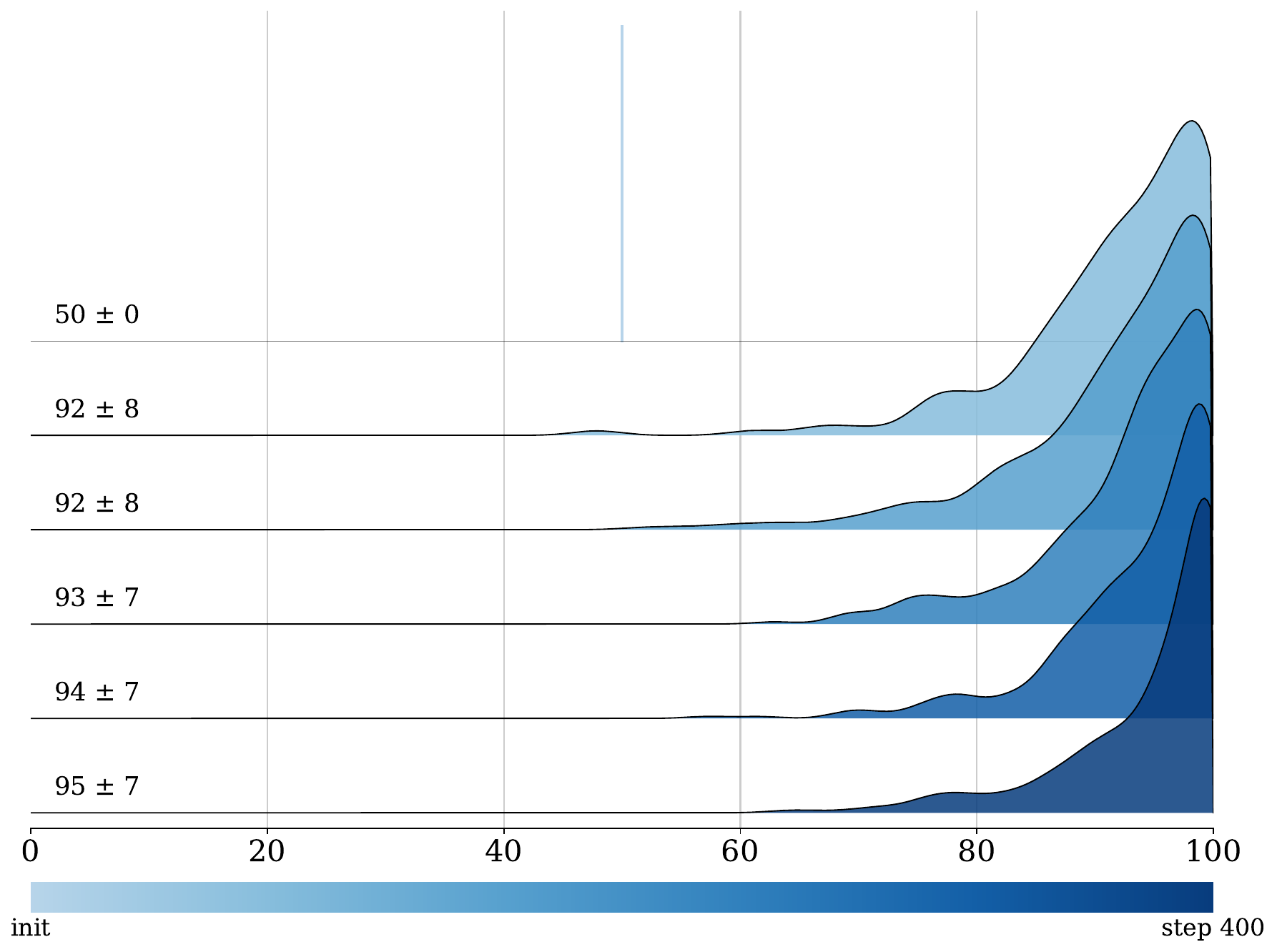}}
    \begin{subfigure}[b]{0.55\textwidth}
        \centering
        \usebox{\linsepridgebox}
        \subcaption{Per-sample AUROC distribution of the BCE-trained scorer over training.}
    \end{subfigure}\hfill
    \begin{subfigure}[b]{0.42\textwidth}
        \centering
        \parbox[c][\dimexpr\ht\linsepridgebox+\dp\linsepridgebox\relax][c]{\linewidth}{%
            \centering
            \chatbubble{human-icon.pdf}{bubble}{What does Hsiao Yun-Hwa identify as in terms of gender?}

            \vspace{0.3em}
            \chatbubble{ai-icon-bce.pdf}{bubble}{
                \tokhl{linsepbcehl}{0}{Hsiao} \tokhl{linsepbcehl}{0}{Yun-Hwa} \tokhl{linsepbcehl}{0}{is} \tokhl{linsepbcehl}{0}{part} \tokhl{linsepbcehl}{0}{of} \tokhl{linsepbcehl}{0}{the} \tokhl{linsepbcehl}{100}{\gttok{LGBTQ}}\tokhl{linsepbcehl}{100}{\gttok{+}} \tokhl{linsepbcehl}{100}{community}\tokhl{linsepbcehl}{0}{.}%
            }
        }
        \subcaption{Final scorer output on representative prompt.}
    \end{subfigure}
    \caption{\textbf{Supervised baseline.} A linear scorer trained with binary cross-entropy against the GT labels of \citet{zhou2026not} cleanly recovers the forget span after a fraction of an epoch. Bold tokens mark the ground-truth forget-relevant span.}
    \label{fig:lin-sep:supervised}
\end{figure}

\paragraph{Unsupervised scorer with the unlearning objective.}
Ground-truth token labels of the kind released by \citet{zhou2026not} require manual or LLM-based annotation of the forget data, which is unrealistic for many unlearning requests, especially when the forget set contains private, copyrighted, or otherwise sensitive content. We therefore ask whether the retain-conflict objective used inside \method{ATWU} (\cref{eq:atwu-parametric}) can recover a comparable linear scorer without labels. \Cref{fig:lin-sep:unsupervised} reports the analogous trajectories and qualitative output in two configurations: training the scorer against the original \dataset{TOFU}-fine-tuned target model that has \emph{memorized} the forget content (top row), and against a \emph{retain} model with the same architecture fine-tuned only on the retain split, with no exposure to the forget set (bottom row). On the target model, the unsupervised objective fails to localize the forget span; its near-perfect predictions on memorized forget content provide too little gradient signal. On the retain model, by contrast, the same objective recovers the answer-bearing tokens (\textit{community}, \textit{Hsiao}, \textit{Yun-Hwa}, partially \textit{LGBTQ}), comparable in shape to the supervised scorer. This observation underpins a central design choice in \method{ATWU}: jointly training the scorer with the unlearned model gradually exposes the scorer to representations closer to the retain regime, supplying the gradient signal needed to localize forget-specific tokens without any labels.

\begin{figure}[!ht]
    \centering
    \captionsetup{justification=centerlast}
    \small
    \setlength{\fboxsep}{1.5pt}
    \sbox{\linsepridgebox}{\includegraphics[width=0.55\textwidth]{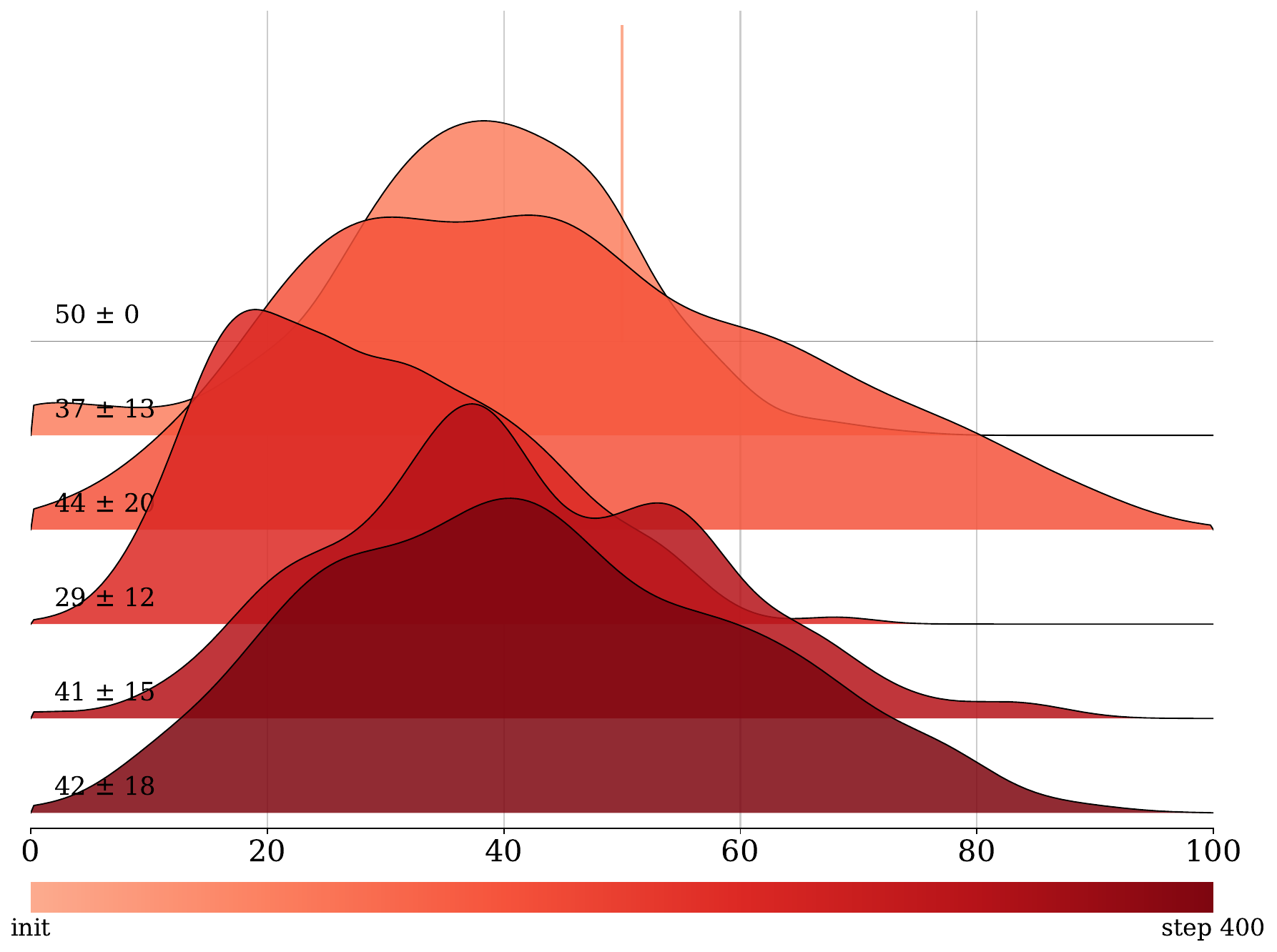}}
    \begin{subfigure}[b]{0.55\textwidth}
        \centering
        \usebox{\linsepridgebox}
        \subcaption{\method{ATWU} objective on the \emph{target} (memorized) model.}
    \end{subfigure}\hfill
    \begin{subfigure}[b]{0.42\textwidth}
        \centering
        \parbox[c][\dimexpr\ht\linsepridgebox+\dp\linsepridgebox\relax][c]{\linewidth}{%
            \centering
            \chatbubble{human-icon.pdf}{bubble}{What does Hsiao Yun-Hwa identify as in terms of gender?}

            \vspace{0.3em}
            \chatbubble{ai-icon-ours.pdf}{bubble}{
                \tokhl{ourshl}{29}{Hsiao} \tokhl{ourshl}{12}{Yun-Hwa} \tokhl{ourshl}{0}{is} \tokhl{ourshl}{2}{part} \tokhl{ourshl}{2}{of} \tokhl{ourshl}{0}{the} \tokhl{ourshl}{0}{\gttok{LGBTQ}}\tokhl{ourshl}{0}{\gttok{+}} \tokhl{ourshl}{1}{community}\tokhl{ourshl}{94}{.}%
            }
        }
        \subcaption{Final output (target model).}
    \end{subfigure}

    \vspace{0.8em}
    \sbox{\linsepridgebox}{\includegraphics[width=0.55\textwidth]{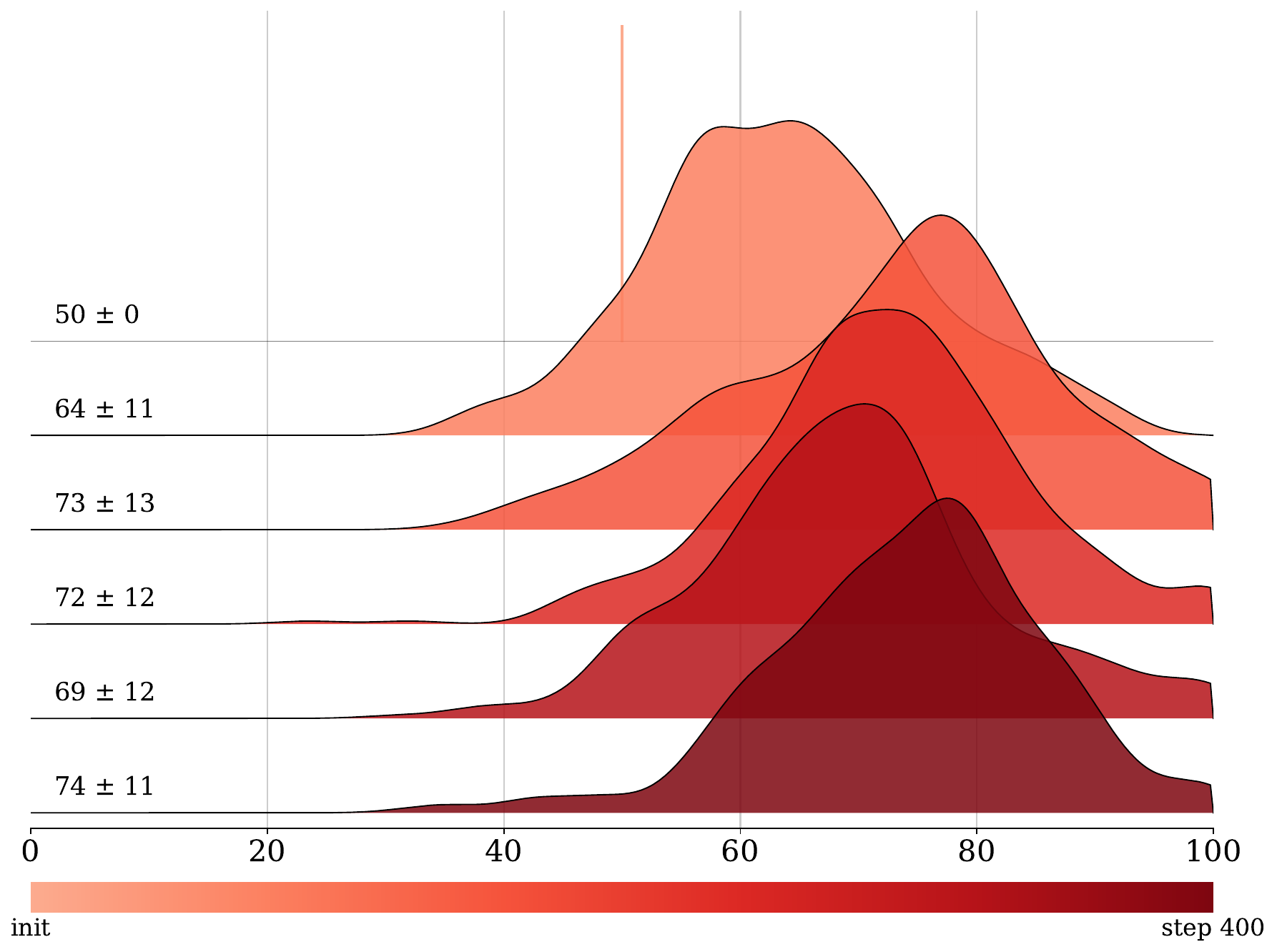}}
    \begin{subfigure}[b]{0.55\textwidth}
        \centering
        \usebox{\linsepridgebox}
        \subcaption{\method{ATWU} objective on the \emph{retain} model.}
    \end{subfigure}\hfill
    \begin{subfigure}[b]{0.42\textwidth}
        \centering
        \parbox[c][\dimexpr\ht\linsepridgebox+\dp\linsepridgebox\relax][c]{\linewidth}{%
            \centering
            \chatbubble{human-icon.pdf}{bubble}{What does Hsiao Yun-Hwa identify as in terms of gender?}

            \vspace{0.3em}
            \chatbubble{ai-icon-ours.pdf}{bubble}{
                \tokhl{ourshl}{57}{Hsiao} \tokhl{ourshl}{38}{Yun-Hwa} \tokhl{ourshl}{0}{is} \tokhl{ourshl}{1}{part} \tokhl{ourshl}{0}{of} \tokhl{ourshl}{0}{the} \tokhl{ourshl}{26}{\gttok{LGBTQ}}\tokhl{ourshl}{6}{\gttok{+}} \tokhl{ourshl}{98}{community}\tokhl{ourshl}{3}{.}%
            }
        }
        \subcaption{Final output (retain model).}
    \end{subfigure}
    \caption{\textbf{Unsupervised scorer.} Trained against the original target model (top), the \method{ATWU} objective fails to localize the GT forget span; trained against a retain model (bottom), the same objective recovers the answer-bearing tokens. Bold tokens mark the ground-truth forget-relevant span.}
    \label{fig:lin-sep:unsupervised}
\end{figure}

\subsection{\method{ATWU} with different forget losses}
\label{app:subsec:additional:augment}

\begin{wraptable}[9]{r}{0.55\textwidth}
    \centering
    \vspace{-0.8em}
    \captionsetup{justification=centerlast}
    \small
    \renewcommand{\arraystretch}{1.25}
    \begin{tabular}{@{}lll@{}}
        \toprule
        \textbf{Method} & \textbf{Replacement} & \textbf{Scope} \\
        \midrule
        \method{DPO}
            & $r(\mathrm{s}^{-}) \mapsto r^{g}(\mathrm{s}^{-})$
            & dispreferred sequence only \\
        \method{NPO}
            & $r(\mathrm{s}) \mapsto r^{g}(\mathrm{s})$
            & forget sequence \\
        \method{SimNPO}
            & $\hat{\ell}(\mathrm{s}) \mapsto \hat{\ell}^{g}(\mathrm{s})$
            & length-normalized forget NLL \\
        \bottomrule
    \end{tabular}
    \caption{Summary of how the learned \method{ATWU} scorer is
    integrated into each forget loss.}
    \label{tab:augment-glance}
    \vspace{-0.6em}
\end{wraptable}
\method{ATWU} is not tied to a single forget loss. Whenever an
unlearning objective contains a forget-side term that decomposes
autoregressively over tokens, we can replace that term by its
scorer-weighted analogue while leaving the rest of the objective
unchanged. This preserves the baseline loss structure but redirects the
forget update toward tokens that the scorer identifies as
forget-relevant.

We use this recipe to integrate the \method{DPO}, \method{NPO}, and
\method{SimNPO} losses into \method{ATWU}, evaluated in
\cref{fig:qualitative:rwku-uplift} and
\cref{tab:exp:rwku-atwu-augmented}. Retain-side terms, preferred
responses, and frozen reference-model terms are never reweighted; only
the original forget sequence receives scorer-weighted token mass.
\Cref{tab:augment-glance} gives the corresponding replacements, and the
following paragraphs define them precisely.

\paragraph{Notation and reweighted building blocks.}
For a sequence
$\seq{x} \in \vocab^\star$ and model parameters
$\tth$, define the unnormalized autoregressive negative
log-likelihood and its scorer-reweighted counterpart as
\begin{equation}
    \ell(\tth \,;\, \seq{x})
    \;=\; -\sum_{t=1}^{\size{\seq{x}}}
        \GAloss{\tth}{x}{t},
    \qquad
    \ell^{g}(\tth \,;\, \seq{x})
    \;=\; -\sum_{t=1}^{\size{\seq{x}}}
        g_t\,
        \GAloss{\tth}{x}{t},
    \label{eq:augment-ell}
\end{equation}
where
$g_t = g_{\sw}(\mathbf{h}_{\tth}(\tok{x}{t}))$
is the scorer output at position $t$. Dataset-level
losses are obtained by summing over sequences:
\[
    \mathcal{L}(\tth \,;\, \mathcal{D})
    = \sum_{\seq{x} \in \mathcal{D}}
        \ell(\tth \,;\, \seq{x}),
    \qquad
    \mathcal{L}^{g}(\tth \,;\, \mathcal{D})
    = \sum_{\seq{x} \in \mathcal{D}}
        \ell^{g}(\tth \,;\, \seq{x}).
\]

For methods that compare a trainable unlearned model
$\tth_{\mathrm{un}}$ against a frozen reference model
$\tth_{\mathrm{ref}}$, we define the sequence-level log-ratio
and its scorer-reweighted analogue as
\begin{equation}
    r_{\tth_{\mathrm{un}}}(\seq{x})
    \;=\; -\ell(\tth_{\mathrm{un}} \,;\, \seq{x})
        + \ell(\tth_{\mathrm{ref}} \,;\, \seq{x}),
    \qquad
    r_{\bm{\theta}_{\mathrm{un}}}^{g}(\mathrm{x})
    \;=\; -w_{\mathrm{s}}\,
        \ell^{g}(\bm{\theta}_{\mathrm{un}} \,;\, \mathrm{x})
        + \ell(\bm{\theta}_{\mathrm{ref}} \,;\, \mathrm{x}),
    \label{eq:augment-ratio}
\end{equation}
where
\[
    w_{\mathrm{x}}
    = \frac{|\mathrm{x}|}{\sum_{t=1}^{|\mathrm{x}|} g_t}.
\]
This normalization keeps the total token mass of the trainable-model
term equal to that of the original sequence loss, since
$\sum_t w_{\mathrm{s}} g_t = |\mathrm{x}|$. The scorer therefore
redistributes the contribution across token positions without changing
the average scale of the trainable-model term. The reference term is
left unchanged because the reference model is frozen and does not
depend on the scorer.

Each augmented objective below replaces only the forget-side building
block of the original loss by its scorer-reweighted counterpart, such
as $\ell^{g}$, $r^{g}$, or a scorer-modulated saturation term.

\paragraph{DPO~\citep{rafailov2023direct}.}
\method{DPO} for unlearning operates on a paired forget set
$\mathcal{A} = \{(\seq{x}_i^{+}, \seq{x}_i^{-})\}_{i=1}^{N}$,
where $\seq{x}_i^{-}$ is the original forget sequence and
$\seq{x}_i^{+}$ is an alternative sequence that the model should
prefer. The original loss compares the two sequences through a sigmoid
of their log-ratio difference:
\begin{equation}
    \mathcal{L}_{\mathrm{DPO}}
    (\bm{\theta}_{\mathrm{un}} \,;\, \mathcal{A})
    \;=\; -\frac{2}{\beta}
        \sum_{(\seq{x}^{+}, \seq{x}^{-}) \in \mathcal{A}}
        \log \sigm\!\Bigl(
            \beta \bigl[
                r_{\bm{\theta}_{\mathrm{un}}}(\seq{x}^{+})
                - r_{\bm{\theta}_{\mathrm{un}}}(\seq{x}^{-})
            \bigr]
        \Bigr).
\end{equation}
The two sides of each pair play asymmetric roles. The dispreferred
sequence $\seq{x}^{-}$ contains the information to be forgotten,
whereas the preferred sequence $\seq{x}^{+}$ is intended to remain
ordinary, acceptable language. We therefore apply the scorer only to
the dispreferred side:
\begin{equation}
    \mathcal{L}_{\mathrm{DPO}}^{g}
    (\bm{\theta}_{\mathrm{un}} \,;\, \mathcal{A})
    \;=\; -\frac{2}{\beta}
        \sum_{(\seq{x}^{+}, \seq{x}^{-}) \in \mathcal{A}}
        \log \sigm\!\Bigl(
            \beta \bigl[
                r_{\bm{\theta}_{\mathrm{un}}}(\seq{x}^{+})
                - r_{\bm{\theta}_{\mathrm{un}}}^{g}(\seq{x}^{-})
            \bigr]
        \Bigr).
\end{equation}
The normalization in \cref{eq:augment-ratio} preserves the average
scale of the dispreferred log-ratio, while shifting its token-level
mass toward forget-relevant positions.

\paragraph{NPO~\citep{zhang2024negative}.}
\method{NPO} removes the preferred sequence from \method{DPO} and
retains only the dispreferred log-ratio:
\begin{equation}
    \mathcal{L}_{\mathrm{NPO}}
    (\bm{\theta}_{\mathrm{un}} \,;\, \DF)
    \;=\; -\frac{2}{\beta}
        \sum_{\seq{x} \in \DF}
        \log \sigm\!\bigl(
            -\beta\, r_{\bm{\theta}_{\mathrm{un}}}(\seq{x})
        \bigr).
\end{equation}
The \method{ATWU}-augmented version replaces this log-ratio by its
scorer-reweighted analogue:
\begin{equation}
    \mathcal{L}_{\mathrm{NPO}}^{g}
    (\bm{\theta}_{\mathrm{un}} \,;\, \DF)
    \;=\; -\frac{2}{\beta}
        \sum_{\seq{x} \in \DF}
        \log \sigm\!\bigl(
            -\beta\, r_{\bm{\theta}_{\mathrm{un}}}^{g}(\seq{x})
        \bigr).
\end{equation}
As in \method{DPO}, the purpose is to preserve the sequence-level
scale of the original objective while reallocating the forgetting
signal across tokens.

\paragraph{SimNPO~\citep{fan2026simplicity}.}
\method{SimNPO} removes the reference model from \method{NPO} and
uses a length-normalized, margin-shifted score. Let
$\hat{\ell}(\bm{\theta} \,;\, \seq{x})
= \ell(\bm{\theta} \,;\, \seq{x}) / |\seq{x}|$. The original
loss is
\begin{equation}
    \mathcal{L}_{\mathrm{SimNPO}}
    (\bm{\theta}_{\mathrm{un}} \,;\, \DF)
    \;=\; -\frac{2}{\beta}
        \sum_{\seq{x} \in \DF}
        \log \sigm\!\Bigl(
            \beta \bigl[
                \hat{\ell}(\bm{\theta}_{\mathrm{un}} \,;\, \seq{x})
                - \delta
            \bigr]
        \Bigr).
\end{equation}
We define the scored length-normalized loss as
\[
    \hat{\ell}^{g}(\bm{\theta} \,;\, \seq{x})
    =
    \frac{\ell^{g}(\bm{\theta} \,;\, \seq{x})}{|\seq{x}|}.
\]
The augmented objective is then
\begin{equation}
    \mathcal{L}_{\mathrm{SimNPO}}^{g}
    (\bm{\theta}_{\mathrm{un}} \,;\, \DF)
    \;=\; -\frac{2}{\beta}
        \sum_{\seq{x} \in \DF}
        \log \sigm\!\Bigl(
            \beta \bigl[
                \hat{\ell}^{g}(\bm{\theta}_{\mathrm{un}} \,;\, \seq{x})
                - \delta
            \bigr]
        \Bigr).
\end{equation}
This keeps the same per-token averaging structure as \method{SimNPO},
but makes the average selective: high-scored tokens contribute more to
the forget loss, and low-scored structural tokens contribute less.

\subsection{Ablations}\label{app:subsec:additional:ablations}


\paragraph{Note on the ablation runs.} The ablations reported in this subsection (\cref{tab:ablation-regs-tofu-1b-f10,tab:scorer-procedure}) were produced under an earlier version of the codebase whose headline \method{ATWU} run differs from the configuration used for the final results in \cref{tab:tofu-full-1b}. The absolute numbers therefore deviate by roughly $10$ percentage points on \FQ and \UQ from the latest \method{ATWU} entry in that table, but the relative trends and the qualitative conclusions of each ablation are unchanged.

\begin{table}[!b]
    \centering
    \small
    \captionsetup{justification=centering}
    \renewcommand{\arraystretch}{1.1}
    \setlength{\tabcolsep}{4.5pt}
    \setlength{\dashlinedash}{0.8pt}\setlength{\dashlinegap}{1.4pt}
    \arrayrulecolor{black!60}
    \begin{tabular}{@{}c:c:c ccc ccc@{}}
        \arrayrulecolor{black}\toprule
        \multicolumn{3}{c}{\textbf{Regularizers}}
        & \multicolumn{3}{c}{\textbf{Relative}}
        & \multicolumn{3}{c}{\textbf{Utility}} \\
        \cmidrule(lr){1-3} \cmidrule(lr){4-6} \cmidrule(lr){7-9}
        $\lambda_H$ & $\lambda_\rho$ & $\lambda_{\ell_2}$
        & $\FQ\,\uparrow$ & $\RD\,\downarrow$ & $\UQ\,\uparrow$
        & MMLU & Rep. & \WR \\
        \midrule
        \multicolumn{3}{c}{\method{Original}}
            & $0.0$ & $0.0$ & $0.0$ & $45.1$ & $559$ & $50.0$ \\
        \midrule
                    &             &             & $49.8$ & $\best{0.0}$ & $49.8$ & $45.1$ & $565$ & $49.0$ \\ \arrayrulecolor{black!60}\cdashline{1-3}
                    &             & \checkmark  & $51.4$ & $1.4$ & $50.0$ & $45.0$ & $568$ & $53.0$ \\ \cdashline{1-3}
                    & \checkmark  &             & $67.3$ & $3.3$ & $63.9$ & $45.0$ & $570$ & $48.0$ \\ \cdashline{1-3}
                    & \checkmark  & \checkmark  & $66.7$ & $1.3$ & $\secondbest{65.4}$ & $45.2$ & $571$ & $50.0$ \\ \cdashline{1-3}
        \checkmark  &             &             & $49.3$ & $\secondbest{0.5}$ & $48.7$ & $45.1$ & $568$ & $53.5$ \\ \cdashline{1-3}
        \checkmark  &             & \checkmark  & $50.3$ & $5.0$ & $45.3$ & $45.2$ & $567$ & $51.0$ \\ \cdashline{1-3}
        \checkmark  & \checkmark  &             & $65.1$ & $3.9$ & $61.2$ & $45.1$ & $574$ & $50.0$ \\ \cdashline{1-3}
        \checkmark  & \checkmark  & \checkmark  & $\best{70.5}$ & $3.9$ & $\best{66.6}$ & $45.0$ & $575$ & $51.0$ \\
        \arrayrulecolor{black}\bottomrule
    \end{tabular}
    \vspace{0.5em}
    \caption{Scorer regularizer ablation on
    \dataset{TOFU}~\texttt{forget10},
    \model{Llama-3.2-1B-Instruct}. \checkmark\ indicates that the
    corresponding regularizer is on with its headline value
    ($\lambda_H = 1$, $\lambda_\rho = 10$, $\lambda_{\ell_2} = 1$);
    a blank cell means the regularizer is set to $0$. Rows are
    ordered by binary expansion of $(\lambda_H, \lambda_\rho,
    \lambda_{\ell_2})$ from all-off to all-on. The bottom row is
    the headline configuration of \method{ATWU}.}
    \label{tab:ablation-regs-tofu-1b-f10}
\end{table}

\paragraph{Scorer regularizers.}
The scorer-side objective in \cref{eq:atwu-parametric} contains two explicit
regularization terms: an entropy penalty, scaled by $\lambda_H$, which
encourages scores to become close to binary, and a population penalty,
scaled by $\lambda_\rho$, which anchors the mean score on $\DF$ to the
target frequency $\rho$. In practice, we also apply $\ell_2$ weight decay
to the scorer parameters, scaled by $\lambda_{\ell_2}$. To isolate the
contribution of each term, we toggle each regularizer between off
($\lambda=0$) and on, where the on values are the headline settings
$\lambda_H=1$, $\lambda_\rho=10$, and $\lambda_{\ell_2}=1$. This gives
the $2^3=8$-cell grid in \cref{tab:ablation-regs-tofu-1b-f10}, with all other
components being constant.

\textbf{Result.}
The population penalty is the dominant factor. Every configuration with
$\lambda_\rho$ enabled achieves $\FQ \geq 65$, whereas every configuration
without it remains in the narrow range $\FQ \in [49,52]$, regardless of the
entropy or $\ell_2$ setting. Thus, the mean-anchor is what allows the
scorer to move away from the near-uniform initialization and commit to a sparse forget-relevant subset.

The entropy and $\ell_2$ terms have smaller, but still measurable,
effects. With the population term enabled, removing the entropy penalty
reduces \FQ from $70.5$ to $66.7$, but improves \RD from $3.9$ to $1.3$,
for a modest \UQ drop from $66.6$ to $65.4$. By contrast, removing
$\ell_2$ weight decay is more costly: \FQ falls from $70.5$ to $65.1$ and
\UQ from $66.6$ to $61.2$, with essentially unchanged \RD. The full
configuration
$(\lambda_H,\lambda_\rho,\lambda_{\ell_2})=(1,10,1)$ achieves the best
\FQ and \UQ overall, while remaining competitive on the other metrics. The
only exception is \RD, where the best values occur in under-committed
$\lambda_\rho=0$ configurations that also fail to forget effectively. We
therefore use the full regularizer set in the headline \method{ATWU}
configuration.

\begin{table}[t]
    \centering
    \small
    \renewcommand{\arraystretch}{1.1}
    \begin{subtable}[t]{0.49\textwidth}
        \centering
        \setlength{\tabcolsep}{3.5pt}
        \resizebox{\linewidth}{!}{%
        \begin{tabular}{@{}c ccc ccc@{}}
            \toprule
            \multirow{2}{*}[-0.6ex]{\textbf{Method}}
            & \multicolumn{3}{c}{\textbf{Relative}}
            & \multicolumn{3}{c}{\textbf{Utility}} \\
            \cmidrule(lr){2-4} \cmidrule(lr){5-7}
            & $\FQ\,\uparrow$ & $\RD\,\downarrow$ & $\UQ\,\uparrow$
            & MMLU & Rep. & \WR \\
            \midrule
            \method{Original}          & $0.0$               & $0.0$              & $0.0$               & $45.1$             & $559$              & $50.0$  \\
            \midrule
            \method{PF}                & $35.0$              & $9.0$              & $26.1$              & $45.2$             & $566$              & $53.0$ \\
            \method{PU}                & $33.4$              & $8.7$              & $24.7$              & $45.3$             & $566$              & $51.0$ \\
            \method{TF}                & $\best{73.7}$       & $\best{2.1}$       & $\best{71.6}$       & $45.1$             & $567$              & $54.0$ \\
            \method{ATWU}              & $\secondbest{70.5}$ & $\secondbest{3.9}$ & $\secondbest{66.6}$ & $45.0$             & $575$              & $51.0$ \\
            \bottomrule
        \end{tabular}%
        }
        \subcaption{Scorer-state ablations.
        \method{TF} = \method{Trained-Frozen}, \method{PF} = \method{Pretrain-Frozen}, \method{PU} = \method{Pretrain-Unfrozen}.}
        \label{tab:scorer-procedure-state}
    \end{subtable}\hfill
    \begin{subtable}[t]{0.49\textwidth}
        \centering
        \setlength{\tabcolsep}{3.5pt}
        \resizebox{\linewidth}{!}{%
        \begin{tabular}{@{}c ccc ccc@{}}
            \toprule
            \multirow{2}{*}[-0.6ex]{$n_{\mathrm{s}}$}
            & \multicolumn{3}{c}{\textbf{Relative}}
            & \multicolumn{3}{c}{\textbf{Utility}} \\
            \cmidrule(lr){2-4} \cmidrule(lr){5-7}
            & $\FQ\,\uparrow$ & $\RD\,\downarrow$ & $\UQ\,\uparrow$
            & MMLU & Rep. & \WR \\
            \midrule
            \method{Original}  & $0.0$               & $0.0$              & $0.0$               & $45.1$           & $559$           & $50.0$  \\
            \midrule
            $1$                & $31.3$              & $\best{0.2}$       & $31.1$              & $45.0$           & $565$           & $53.0$ \\
            $5$                & $\best{70.5}$       & $3.9$              & $\best{66.6}$       & $45.0$           & $575$           & $51.0$ \\
            $10$               & $\secondbest{67.0}$ & $\secondbest{3.6}$ & $\secondbest{63.4}$ & $45.2$           & $578$           & $49.0$ \\
            \texttt{joint}     & $39.9$              & $4.5$              & $35.4$              & $45.1$           & $571$           & $51.0$ \\
            \bottomrule
        \end{tabular}%
        }
        \subcaption{Update-frequency sweep.}
        \label{tab:scorer-procedure-uev}
    \end{subtable}
    \vspace{0.5em}
    \caption{Scorer training-procedure ablations on
    \dataset{TOFU}~\texttt{forget10},
    \model{Llama-3.2-1B-Instruct}.
    (\subref{tab:scorer-procedure-state})~Three alternatives to
    joint online training of the scorer.
    (\subref{tab:scorer-procedure-uev})~Update-frequency sweep:
    the scorer is refreshed once per $n_{\mathrm{s}}$ model steps;
    $n_{\mathrm{s}}=5$ matches the headline configuration of
    \method{ATWU}. 
    }
    \label{tab:scorer-procedure}
\end{table}

\paragraph{Scorer training procedure.}
The preceding ablations fix the scorer training routine and vary either
how the scorer is used or which regularizers shape it. We next ablate the
training procedure itself, asking two questions: (i) whether the scorer
must co-adapt with the language model during unlearning, and (ii) how
often the scorer should be refreshed relative to model updates.
\Cref{tab:scorer-procedure} reports side-by-side ablations on
\dataset{TOFU}~\texttt{forget10} with
\model{Llama-3.2-1B-Instruct}. All non-scorer settings match the headline
\method{ATWU} run.

\begin{wrapfigure}{r}{0.5\textwidth}
    \centering
    \vspace{-0.8em}
    \includegraphics[width=\linewidth]{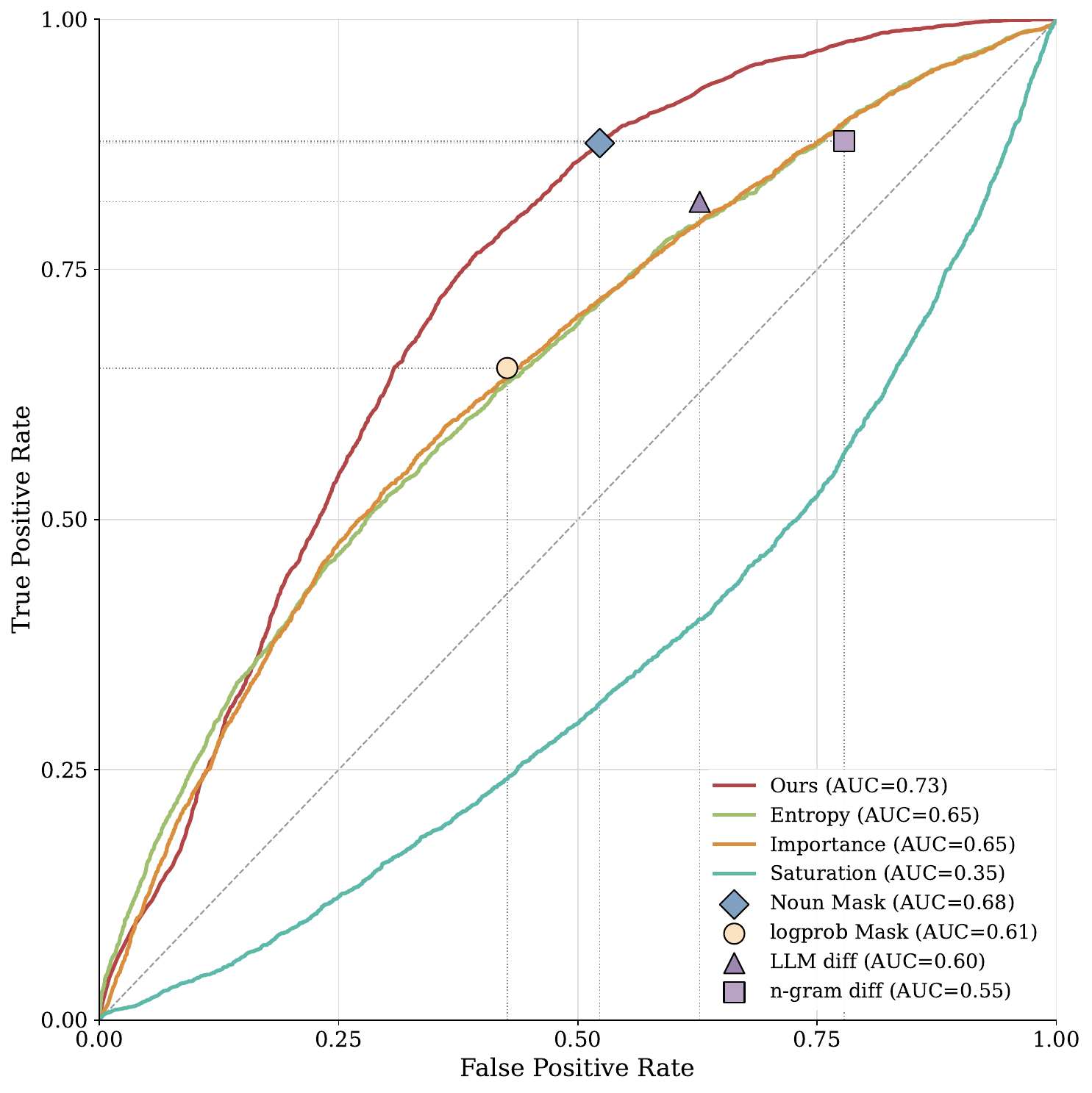}
    \vspace{-1.2em}
    \caption{
    ROC curves for token-level forget-relevance detection on
    \dataset{TOFU}~\texttt{forget10}. \method{ATWU} obtains the highest
    AUROC among the compared scoring methods; the dashed diagonal denotes
    random scoring.
    }
    \label{fig:method-roc-tofu-forget10}
    \vspace{-0.8em}
\end{wrapfigure}
\Cref{tab:scorer-procedure-state} compares three alternatives to the
headline online procedure. \method{Trained-Frozen} (\method{TF}) reuses
the converged scorer from a previous \method{ATWU} run, freezes it, and
then retrains the language model from scratch under this fixed weighting.
\method{Pretrain-Frozen} (\method{PF}) trains a fresh scorer for one
epoch on the original, non-unlearned checkpoint and freezes it during
unlearning. \method{Pretrain-Unfrozen} (\method{PU}) uses the same
pretrained scorer as an initialization, but continues to update it during
unlearning.

The results show that online co-adaptation is not strictly necessary once
a useful scorer has already been learned. \method{TF} slightly exceeds the
headline online run on all relative metrics, improving \FQ from $70.5$ to
$73.7$, \UQ from $66.6$ to $71.6$, and \RD from $3.9$ to $2.1$. Thus, a
converged scorer can be frozen and reused successfully. However, this
does not hold for scorers trained only on the original model:
\method{PF} reaches only \FQ $35.0$, and continuing to fine-tune that
scorer in \method{PU} performs similarly poorly, with \FQ $33.4$. The key
requirement is therefore not continuous co-training per se, but exposure
to an unlearning trajectory. A scorer pretrained only on the original
checkpoint is not a sufficient substitute.

\Cref{tab:scorer-procedure-uev} ablates the scorer update frequency. The
scorer is refreshed once every $n_{\mathrm{s}}\in\{1,5,10\}$ model steps,
with all other settings fixed. Although updating every step is the most
frequent option, it performs poorly: at $n_{\mathrm{s}}=1$, \FQ drops to
$31.3$ and \RD to $0.2$, indicating that the method fails to move the model
decisively in either the forget or retain direction. The headline setting
$n_{\mathrm{s}}=5$ gives the best trade-off, achieving \FQ $70.5$ and \UQ $66.6$. Slowing the refresh rate to $n_{\mathrm{s}}=10$ remains
competitive, with \FQ $67.0$ and \UQ $63.4$ at similar \RD. This suggests a
broad plateau between five and ten model steps: the exact refresh period
is not especially sensitive, but updating the scorer every step is
clearly detrimental.

Finally, the \texttt{joint} variant, which co-updates the scorer and
language model in lockstep rather than using a scheduled refresh, also
underperforms, reaching only \FQ $39.9$ and \UQ $35.4$. This confirms that
the scheduled refresh mechanism, rather than mere simultaneous training,
is what allows the scorer to commit to a useful forget-relevant subset.



\end{document}